%% file: main.tex
\documentclass[10pt]{article} % For LaTeX2e
\usepackage[accepted]{tmlr}
\input{math_commands.tex}

\usepackage{hyperref}
\usepackage{url}

\usepackage{xcolor}
\definecolor{umn_maroon}{RGB}{122, 0, 25}
\definecolor{myred}{HTML}{D62728}
\definecolor{myblue}{HTML}{1F77B4}
\definecolor{mygreen}{HTML}{00FF00}
\usepackage{wrapfig}
\usepackage{lipsum}
\usepackage{caption}
\usepackage{float}

\usepackage{enumitem}
\input{ju_math}
\usepackage{algorithm}%
\usepackage{algpseudocode}

\title{Early Stopping for Deep Image Prior}

\author{\name Hengkang Wang \email wang9881@umn.edu \\
      \addr University of Minnesota, Twin Cities
      \AND
      \name Taihui Li \email lixx5027@umn.edu \\
      \addr University of Minnesota, Twin Cities
      \AND
      \name Zhong Zhuang \email zhuan143@umn.edu\\
      \addr University of Minnesota, Twin Cities
      \AND
      \name Tiancong Chen \email chen6271@umn.edu \\
      \addr University of Minnesota, Twin Cities
      \AND
      \name Hengyue Liang \email liang656@umn.edu \\
      \addr University of Minnesota, Twin Cities
      \AND
      \name Ju Sun \email jusun@umn.edu \\
      \addr University of Minnesota, Twin Cities
      }

\def\month{12}  % Insert correct month for camera-ready version
\def\year{2023} % Insert correct year for camera-ready version
\def\openreview{\url{https://openreview.net/forum?id=231ZzrLC8X}} % Insert correct link to OpenReview for camera-ready version

\begin{document}

\maketitle
\begin{abstract}
\input{sections/Sec0_Abstract}
\end{abstract}

\input{sections/Sec1_Introduction}
\input{sections/Sec3_Method}
\input{sections/Sec4_Exp}

\input{sections/Sec5_Discussion}

\section*{Acknowledgements}
Zhong Zhuang, Hengkang Wang, Tiancong Chen and Ju Sun are partly supported by NSF CMMI 2038403. We thank the anonymous reviewers and the associate editor for their insightful comments that have substantially helped us to improve the presentation of this paper. The authors acknowledge the Minnesota Supercomputing Institute (MSI) at the University of Minnesota for providing resources that contributed to the research results reported within this paper.

\bibliography{main}
\bibliographystyle{tmlr}

\appendix
\input{sections/Sec7_Appendix}

\end{document}

%% file: math_commands.tex
\usepackage{amsmath,amsfonts,bm}

\def\eqref#1{equation~\ref{#1}}
\def\1{\bm{1}}

\def\eps{{\epsilon}}

\DeclareMathAlphabet{\mathsfit}{\encodingdefault}{\sfdefault}{m}{sl}
\SetMathAlphabet{\mathsfit}{bold}{\encodingdefault}{\sfdefault}{bx}{n}

\newcommand{\R}{\mathbb{R}}

%% file: ju_math.tex
\usepackage{graphicx,amsmath,amsfonts,amsthm,amscd,amssymb,bm,url,color,latexsym}
\usepackage{physics}
\allowdisplaybreaks
\usepackage[capitalize,nameinlink]{cleveref}
\crefname{section}{Sec.}{Secs.}
\Crefname{section}{Section}{Sections}
\Crefname{table}{Table}{Tables}
\crefname{table}{Tab.}{Tabs.}

\newtheorem{theorem}{Theorem}[section]

\renewcommand{\mathbf}{\boldsymbol}

\newcommand{\mb}{\mathbf}
\newcommand{\mc}{\mathcal}

\newcommand{\bb}{\mathbb}

\newcommand{\set}[1]{\left\{ #1 \right\}}

\newcommand{\epss}{\varepsilon}
\newcommand{\RR}{\bb R}

\newcommand{\paren}{\pqty}
\newcommand{\brac}{\bqty}

\newcommand{\wh}{\widehat}

\newcommand{\T}{\intercal}

\newcommand{\innerprod}[2]{\left\langle #1,  #2 \right\rangle}

%% file: sections/Sec0_Abstract.tex
Deep image prior (DIP) and its variants have shown remarkable potential to solve inverse problems in computational imaging, \emph{needing no separate training data}. Practical DIP models are often substantially overparameterized. During the learning process, these models first learn the desired visual content and then pick up potential modeling and observational noise, i.e., performing early learning then overfitting (ELTO). Thus, the practicality of DIP hinges on early stopping (ES) that can capture the transition period. In this regard, most previous DIP works for computational imaging tasks only demonstrate the potential of the models, reporting peak performance against ground truth, but providing no clue about how to operationally obtain near-peak performance \emph{without access to ground truth}. In this paper, we set out to break this practicality barrier of DIP and propose an effective ES strategy that consistently detects near-peak performance in various computational imaging tasks and DIP variants. Simply based on the running variance of DIP intermediate reconstructions, our ES method not only outpaces the existing ones---which only work in very narrow regimes, but also remains effective when combined with methods that try to mitigate overfitting. The code to reproduce our experimental results is available at \url{https://github.com/sun-umn/Early_Stopping_for_DIP}.

%% file: sections/Sec1_Introduction.tex
\section{Introduction}\label{sec:introduction}

Inverse problems (IPs) are prevalent in computational imaging, ranging from basic image denoising, super-resolution, and deblurring, to advanced 3D reconstruction and major tasks in scientific and medical imaging~\citep{Szeliski2021Computer}. Despite the disparate settings, all these problems take the form of recovering a visual object $\mb x$ from $\mb y = f(\mb x)$, where $f$ models the forward physical process to obtain the observation $\mb y$. Typically, these visual IPs are not determined in a unique way: $\mb x$ cannot be determined uniquely from $\mb y$. This is exacerbated by potential modeling (e.g., linear $f$ to approximate a nonlinear process) and observational (e.g., Gaussian or shot) noise, i.e., $\mb y \approx f\paren{\mb x}$. To overcome nonuniqueness and improve noise stability, researchers often encode a variety of problem-specific priors on $\mb x$ when formulating IPs. 

Traditionally, IPs are phrased as regularized data fitting problems: 
\begin{align}  \label{eq:reg_df}
    \min_{\mb x} \; \ell\paren{\mb y, f\paren{\mb x}} + \lambda R\paren{\mb x}  \quad \quad \ell\paren{\mb y, f\paren{\mb x}}: \text{data-fitting loss}, \;  R\paren{\mb x}: \text{regularizer} 
\end{align}
where $\lambda$ is the regularization parameter. Here, the loss $\ell$ is often chosen according to the noise model, and the regularizer $R$ encodes priors on $\mb x$. The advent of deep learning has revolutionized the way IPs are solved. On the radical side, deep neural networks (DNNs) are trained to directly map any given $\mb y$ to an $\mb x$; on the mild side, pre-trained or trainable {deep learning} models are taken to replace certain nonlinear mappings in iterative numerical algorithms for solving \cref{eq:reg_df} (e.g. plug-and-play and algorithm unrolling); see recent surveys~\cite{OngieEtAl2020Deep,JanaiEtAl2020Computer} on these developments. All of these {deep-learning}-based methods rely on large training sets to adequately represent the underlying priors and/or noise distributions. \textbf{This paper concerns another family of striking ideas that do not require separate training data}. 

\paragraph{\textcolor{umn_maroon}{Deep image prior (DIP)}}
\cite{ulyanov2018deep} proposes parameterizing $\mb x$ as $\mb x = G_{\mb \theta} \paren{\mb z}$, where $G_{\mb \theta}$ is a trainable DNN parameterized by $\mb \theta$ and $\mb z$ is a frozen or trainable random seed. \textbf{No separate training data other than $\mb y$ are used!} Plugging the reparametrization into \cref{eq:reg_df}, we obtain
\begin{align}\label{eq:dip}
        \min_{\mb \theta}\;  \ell \paren{\mb y, f \circ G_{\mb \theta}\paren{\mb z}} + \lambda R \circ G_{\mb \theta}\paren{\mb z}. 
\end{align}
$G_{\mb \theta}$ is often ``overparameterized''---containing substantially more parameters than the size of $\mb x$, and ``structured''---e.g., consisting of convolution networks to encode structural priors in natural visual objects. The resulting optimization problem is solved using standard first-order methods (e.g., (adaptive) gradient descent). When $\mb x$ has multiple components with different physical meanings, one can naturally parametrize $\mb x$ using multiple DNNs. This simple idea has led to surprisingly competitive results in numerous visual IPs, from low-level image denoising,  super-resolution, inpainting~\citep{ulyanov2018deep,heckel2018deep,liu2018image} and blind deconvolution~\citep{Ren_2020_CVPR,wang2019image,asim2019blind,TranEtAl2021Explore,zhuang_blind_2022}, to mid-level image decomposition and fusion~\citep{Gandelsman_2019_CVPR,MaEtAl2021Unsupervised}, and to advanced {computational imaging} problems~\citep{darestani2021accelerated,hand2018phase,williams2019deep,yoo2021timedependent,2020,cascarano2021combining,hashimoto2021direct,gong2021direct,vanveen2020compressed,TayalEtAl2021Phase,zhuang2022practical}; see the survey~\cite{QayyumEtAl2021Untrained}.

\begin{wrapfigure}{r}{0.5\textwidth}
    \vspace{-1em}
    \centering
    \includegraphics[width=0.9\linewidth]{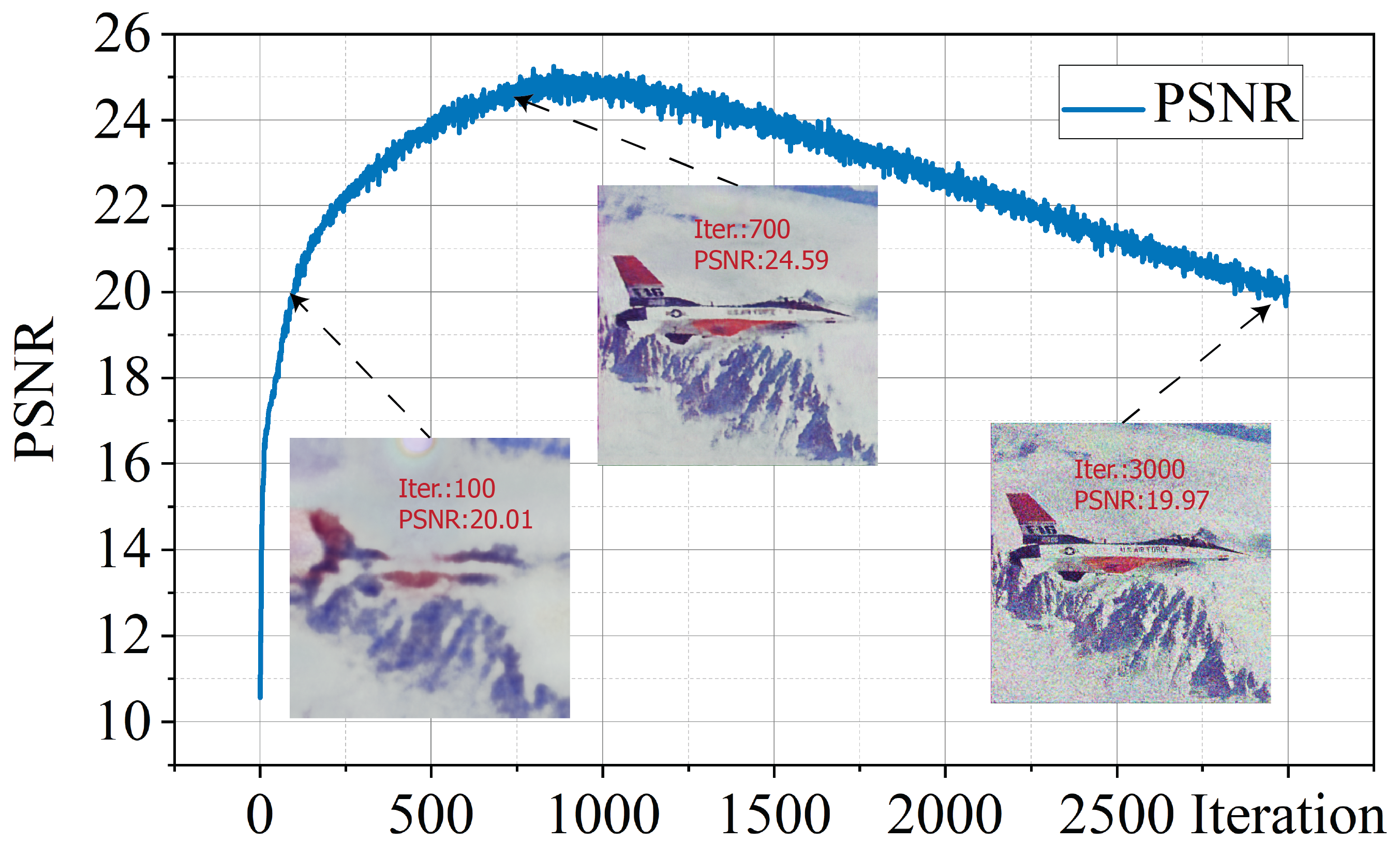}
    \caption{The ``early-learning-then-overfitting'' (ELTO) phenomenon in DIP for image denoising. The quality of the estimated image climbs first to a peak and then drops once the noise is picked up by the model $G_{\mb \theta}(\mb z)$ also.
    }
    % \vspace{-0.5em}
    \label{fig:overfit}
\end{wrapfigure}
\paragraph{\textcolor{umn_maroon}{Overfitting issue in DIP}}
A critical detail that we have glossed over is \textbf{overfitting}. Since $G_{\mb \theta}$ is often substantially overparameterized, $G_{\mb \theta}(\mb z)$ can represent arbitrary elements in the $\mb x$ domain. Global optimization of~\eqref{eq:dip} would normally lead to $\mb y = f \circ G_{\mb \theta}\paren{\mb z}$, but $G_{\mb \theta}(\mb z)$  may not reproduce $\mb x$, e.g., when $f$ is non-injective, or $\mb y \approx f\paren{\mb x}$ so that $G_{\mb \theta}(\mb z)$ also accounts for the modeling and observational noise. Fortunately, DIP models and first-order optimization methods together offer a blessing: in practice, $G_{\mb \theta}(\mb z)$ has a bias towards the desired visual content and learns it much faster than learning noise. Therefore, the quality of reconstruction climbs to a peak before the potential degradation due to noise; see \cref{fig:overfit}. This ``early-learning-then-overfitting'' (ELTO) phenomenon has been repeatedly reported in previous work and is also supported by theories on simple $G_{\mb \theta}$ and linear $f$~\citep{heckel2020denoising,HeckelSoltanolkotabi2020Compressive}. \textbf{The successes of the DIP models claimed above are on the premise that appropriate early stopping (ES) around performance peaks can be made}. 

\paragraph{\textcolor{umn_maroon}{Is ES for DIP trivial?}} 
Natural ideas trying to perform ES can fail quickly. \textbf{(1) Visual inspection}: This subjective approach is fine for small-scale tasks involving few problem instances, but quickly becomes infeasible for many scenarios, such as (a) large-scale batch processing, (b) recovery of visual contents tricky to visualize and/or examine by eyes (e.g. 3D or 4D visual objects), and (c) scientific imaging of unfamiliar objects (e.g., MRI imaging of rare tumors and microscopic imaging of new virus species); \textbf{(2) Tracking full-reference/no-reference image quality metrics (FR/NR-IQMs) or fitting loss}: Without the ground truth $\mb x$, computing any FR-IQM and thereby tracking its trajectory (e.g., the PNSR curve in \cref{fig:overfit}) is out of the question. We consider the tracking of NR-IQMs as a family of baseline methods in \cref{sec:baseline}; the performance is much worse than ours. We also explore the possibility of using the loss curve for ES here, but are unable to find correlations between the trend of the loss and that of the PSNR curve, shown in \cref{fig:loss_curve}; \textbf{(3) Tuning the iteration number}: This ad-hoc solution is taken in most previous work. But since the peak iterations of DIP vary considerably across images and tasks (see, e.g., \cref{fig:helper_example,fig:mri_curve,sec:denoising_eg,sec:competing_ap}), this could entail numerous trial-and-error steps and lead to suboptimal stopping points; \textbf{(4) Validation-based ES}: ES easily reminds us of validation-based ES in supervised learning. The DIP approach to IPs, as summarized in \cref{eq:dip} \textbf{does not belong to} supervised learning, as it only deals with a single instance $\mb y$, without separate $(\mb x, \mb y)$ pairs as training data. There are recent ideas~\citep{yaman2021zeroshot,ding_validation_2022} that hold part of the observation $\mb y$ out as a validation set to emulate validation-based ES in supervised learning, but they quickly become problematic for nonlinear IPs due to the significant violation of the underlying i.i.d. assumption; see \cref{sec:bid_exp}.  

\paragraph{\textcolor{umn_maroon}{Prior work addressing the overfitting}}
\label{sec:prior_work}
There are three main approaches to counteracting the overfitting of DIP models. \textbf{(1) Regularization}: \cite{heckel2018deep} mitigates overfitting by restricting the size of $G_{\mb \theta}$ to the underparametrization regime. \cite{MetzlerEtAl2018Unsupervised,shi2021measuring,jo2021rethinking,ChengEtAl2019Bayesian} control the network capacity by regularizing the layer-wise weights or the network Jacobian. \cite{liu2018image,mataev2019deepred,sun2021solving,cascarano2021combining} use additional regularizer(s) $R\paren{G_{\mb \theta}(\mb z)}$, such as the total-variation norm or trained denoisers. These methods require the right regularization level---which depends on the noise type and level---to avoid overfitting; with an improper regularization level, they can still lead to overfitting (see \cref{fig:helper_example} and \cref{sec:exp_ID}). Moreover, when they do succeed, the performance peak is postponed to the last iterations, often increasing the computational cost by several folds. 
\textbf{(2) Noise modeling}: \cite{you2020robust} models sparse additive noise as an explicit term in their optimization objective. \cite{jo2021rethinking} designs regularizers and ES criteria specific to Gaussian and shot noise. \cite{DingEtAl2021Rank} explores subgradient methods with diminishing step size schedules for impulse noise with the $\ell_1$ loss, with preliminary success. These methods do not work beyond the types and levels of noise they target, whereas our knowledge of the noise in a given visual IP is typically limited. \textbf{(3) Early stopping (ES)}: \cite{shi2021measuring} tracks progress based on a ratio of no-reference blurriness and sharpness, but the criterion only works for their modified DIP models, as acknowledged by the authors. \cite{jo2021rethinking} provides noise-specific regularizer and ES criterion, but it is not clear how to extend the method to unknown types and levels of noise. \cite{Li2021} proposes monitoring DIP reconstruction by training a coupled autoencoder. Although its performance is similar to ours, the extra autoencoder training slows down the whole process dramatically; see \cref{sec:expriments}. \cite{yaman2021zeroshot,ding_validation_2022} emulate validation-based ES in supervised learning by splitting elements of $\mb y$ into ``training'' and ``validation'' sets so that validation-based ES can be performed. But in IPs, especially nonlinear ones (e.g., in blind image deblurring (BID), $\mb y \approx \mb k \ast \mb x$ where $\ast$ is the linear convolution), elements of $\mb y$ can be far from being i.i.d., and so validation may not work well. Moreover, holding out part of the observation in $\mb y$ can substantially reduce the peak performance; see \cref{sec:bid_exp}.

\begin{table}[!htpb]
\centering 
% \vspace{-1em}
\caption{Summary of performance of our DIP$+$ES-WMV and competing methods on image denoising and blind image deblurring (BID). \textbf{$\checkmark$}: working reasonably well (PSNR $\ge$ $2dB$ less of the original DIP peak); -: not working well (PSNR $\le$ $2dB$ less of the original DIP peak): N/A: not applicable (i.e., we do not perform comparison due to certain reasons). Note that DF-STE, DOP, and SB are based on modified DIP models.}

\label{tab:overall}
\setlength{\tabcolsep}{2.5mm}{
% \begin{table}[]
\begin{tabular}{ccccccccccc}
\hline
        & \multicolumn{8}{c}{Image denoising}                                                                                                                                                                                                                                                                                                                                                                                                                                                                                                   & \multicolumn{2}{c}{BID}                                                                                       \\ \hline
        & \multicolumn{2}{c}{Gaussian}                                                                                                       & \multicolumn{2}{c}{Impulse}                                                                                                        & \multicolumn{2}{c}{Speckle}                                                                                                        & \multicolumn{2}{c}{Shot}                                                                                      & \multicolumn{2}{c}{Real world}                                                                                \\ \hline
        & \multicolumn{1}{c}{Low}                                         & \multicolumn{1}{c}{High}                                        & \multicolumn{1}{c}{Low}                                         & \multicolumn{1}{c}{High}                                        & \multicolumn{1}{c}{Low}                                         & \multicolumn{1}{c}{High}                                        & \multicolumn{1}{c}{Low}                                         & High                                        & \multicolumn{1}{c}{Low}                                         & High                                        \\ \hline
DIP$+$ES-WMV (\textcolor{red}{Ours})  & \multicolumn{1}{c}{\textbf{$\checkmark$}} & \multicolumn{1}{c}{\textbf{$\checkmark$}} & \multicolumn{1}{c}{\textbf{$\checkmark$}} & \multicolumn{1}{c}{\textbf{$\checkmark$}}                                & \multicolumn{1}{c}{\textbf{$\checkmark$}} & \multicolumn{1}{c}{\textbf{$\checkmark$}} & \multicolumn{1}{c}{\textbf{$\checkmark$}} & \textbf{$\checkmark$} & \multicolumn{1}{c}{\textbf{$\checkmark$}} & \textbf{$\checkmark$} \\ %\hline
DIP+NR-IQMs & \multicolumn{1}{c}{-}                                           & \multicolumn{1}{c}{-}                                           & \multicolumn{1}{c}{-}                                           & \multicolumn{1}{c}{-}                                           & \multicolumn{1}{c}{-}                                           & \multicolumn{1}{c}{-}                                           & \multicolumn{1}{c}{-}                                           & -                                           & \multicolumn{1}{c}{N/A}                                           & N/A                                           \\ %\hline
DIP+SV-ES   & \multicolumn{1}{c}{\textbf{$\checkmark$}} & \multicolumn{1}{c}{\textbf{$\checkmark$}} & \multicolumn{1}{c}{\textbf{$\checkmark$}} & \multicolumn{1}{c}{\textbf{$\checkmark$}}                                & \multicolumn{1}{c}{\textbf{$\checkmark$}} & \multicolumn{1}{c}{\textbf{$\checkmark$}} & \multicolumn{1}{c}{\textbf{$\checkmark$}} & \textbf{$\checkmark$} & \multicolumn{1}{c}{N/A}                                           & N/A                                           \\ %\hline
DIP+VAL     & \multicolumn{1}{c}{\textbf{$\checkmark$}} & \multicolumn{1}{c}{\textbf{$\checkmark$}} & \multicolumn{1}{c}{\textbf{$\checkmark$}} & \multicolumn{1}{c}{\textbf{$\checkmark$}}                                & \multicolumn{1}{c}{\textbf{$\checkmark$}} & \multicolumn{1}{c}{\textbf{$\checkmark$}} & \multicolumn{1}{c}{\textbf{$\checkmark$}} & \textbf{$\checkmark$} & \multicolumn{1}{c}{-}                                           & - \\ 
DF-STE  & \multicolumn{1}{c}{\textbf{$\checkmark$}} & \multicolumn{1}{c}{\textbf{$\checkmark$}}                                & \multicolumn{1}{c}{N/A}                                           & \multicolumn{1}{c}{N/A}                                           & \multicolumn{1}{c}{N/A}                                           & \multicolumn{1}{c}{N/A}                                           & \multicolumn{1}{c}{\textbf{$\checkmark$}} & \textbf{$\checkmark$}                                & \multicolumn{1}{c}{N/A}                                           & N/A                                           \\ %\hline

DOP     & \multicolumn{1}{c}{N/A}                                           & \multicolumn{1}{c}{N/A}                                           & \multicolumn{1}{c}{\textbf{$\checkmark$}} & \multicolumn{1}{c}{\textbf{$\checkmark$}} & \multicolumn{1}{c}{N/A}                                           & \multicolumn{1}{c}{N/A}                                           & \multicolumn{1}{c}{N/A}                                           & N/A                                           & \multicolumn{1}{c}{N/A}                                           & N/A                                           \\ %\hline
SB      & \multicolumn{1}{c}{{\textbf{$\checkmark$}}}                                           & \multicolumn{1}{c}{{\textbf{$\checkmark$}}}                                           & \multicolumn{1}{c}{N/A}                                           & \multicolumn{1}{c}{N/A}                                           & \multicolumn{1}{c}{N/A}                                           & \multicolumn{1}{c}{N/A}                                           & \multicolumn{1}{c}{N/A}                                           & N/A                                           & \multicolumn{1}{c}{N/A}                                           & N/A                                            %\hline
                                           \\ \hline
% \end{tabular}
% \end{table}

\end{tabular}
}
% \vspace{-1em}
\end{table}

\paragraph{\textcolor{umn_maroon}{Our contribution}}
We advocate the ES approach---\textbf{the iteration process stops once a good ES point is detected}, as (1) the regularization and noise modeling approaches, even if effective, often do not improve peak performance but push it until the last iterations; there could be $\ge 10\times$ more iterations spent than climbing to the peak in the original DIP models; (2) both need deep knowledge about the noise type/level, which is practically unknown for most applications. If their key models and hyperparameters are not set appropriately, overfitting probably remains, and ES is still needed. \textbf{In this paper, we build \textbf{a novel ES criterion} for various DIP models simply by monitoring the trend of the running variance of the reconstruction sequence}. Our ES method is \textbf{(1) Effective}: {The gap between our detected and the peak performance, i.e., the detection gap}, is typically very small, as measured by standard visual quality metrics (PSNR and SSIM). Our method works well for DIP and its variants, including sinusoidal representation networks~\citep[SIREN]{sitzmann2020implicit} and deep decoder~\citep{heckel2018deep}, on different noisy types/levels and in $5$ visual IPs, including both linear and non-linear ones. Furthermore, our method can help several regularization-based methods, e.g., Gaussian process-DIP~\citep[GP-DIP]{ChengEtAl2019Bayesian}, DIP with total variation regularization~\citep[DIP-TV]{liu2018image,cascarano2021combining} to perform reasonable ES when they fail to prevent overfitting; \textbf{(2) Efficient}: The per-iteration overhead is a fraction---for the standard version in \cref{alg:framework}, or negligible---for the variant in \cref{alg:framework_emavg}, relative to the per-iteration cost of \cref{eq:dip}; \textbf{(3) Robust}: Our method is relatively insensitive to the two hyperparameters, i.e. window size and patience number. We keep the same hyperparameters for all experiments \cref{sec:method,sec:expriments} except for the ablation study. In contrast, the hyperparameters of most of the methods reviewed above are sensitive to the noise type/level. We summarize the performance of our DIP+ES method against competing methods for image denoising and BID in \cref{tab:overall}; we present the detailed results in \cref{sec:expriments}.

\begin{wrapfigure}{r}{0.45\textwidth}
    \centering 
    \vspace{-1em}
    \includegraphics[width=0.9\linewidth]{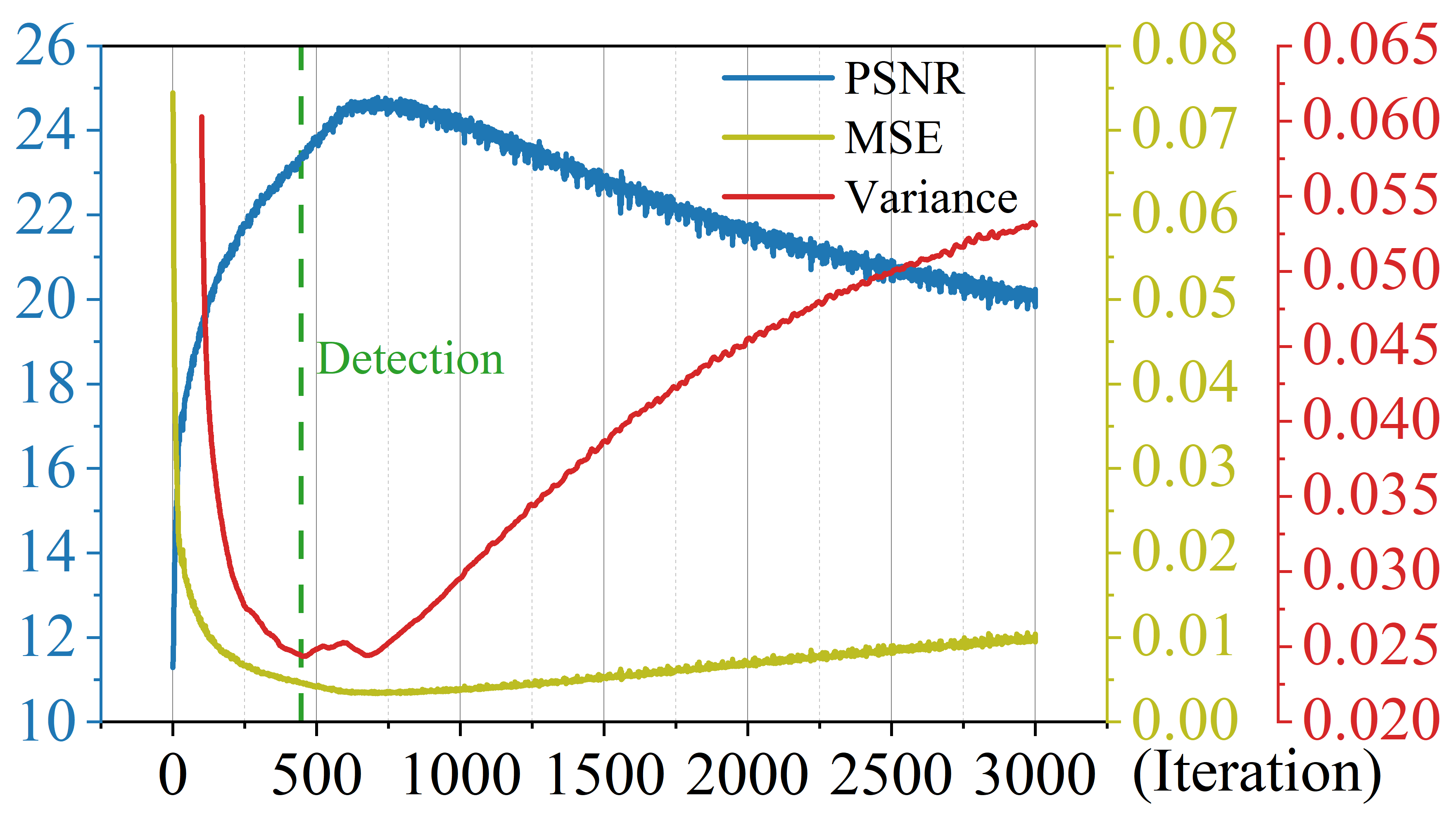}
    % % \vspace{-1em}
    \caption{Relationship between the PSNR, MSE, and VAR curves. Our method relies on the VAR curve, whose valley is often well aligned with the MSE valley---that corresponds to the PSNR peak.
    }
    \label{fig:mse_var_trend}
    \vspace{-4em}
\end{wrapfigure} 

\paragraph{\textcolor{umn_maroon}{Remarks on diffusion models for IPs}}

Recently, diffusion-based models (DBMs) have shown great promise in solving linear IPs~\cite{wang_zero-shot_2022,zhu_denoising_2023}. However, we note three things about these ideas: (1) their performance seems to be sensitive to the match of noise type and level between the training of the diffusion models and those in the actual IPs. Mismatch can lead to miserable results, as we demonstrate in \cref{tab:sr_2,tab:denoising_dm}; (2) DBMs in solving IPs can suffer from overfitting issues similar to that in DIP also (see \cref{sec:diffusion}); (3) there has been limited success in tackling nonlinear IPs by DBMs so far, see, e.g., the very recent attempt~\cite{chung2023diffusion}. It remains to be seen how effective these ideas can be on general nonlinear IPs. 

%% file: sections/Sec3_Method.tex
\section{Our Early-Stopping (ES) Method}\label{sec:method}
\paragraph{\textcolor{umn_maroon}{Intuition for our method}} \label{sec:method_ituit}
We assume that $\mb x$ is the unknown groundtruth visual object of size $N$, $\set{\mb \theta^t}_{t \ge 1}$ is the iterate sequence and $\set{\mb x^t}_{t \ge 1}$ the reconstruction sequence where $\mb x^t \doteq G_{\mb \theta^t}(\mb z)$. Since we do not know $\mb x$, we cannot compute the PNSR or any FR-IQM curve. But we observe from \cref{fig:mse_var_trend} that the MSE (resp. PSNR; recall $\mathrm{PSNR}(\mb x^t) = 10 \log_{10} \norm{\mb x}_{\infty}^2/\mathrm{MSE}(\mb x^t)$) curve follows a U (resp. bell) shape: $\norm{\mb x^t - \mb x}_F^2$ initially drops rapidly to a low level and then climbs back due to the noise effect, i.e., the ELTO phenomenon in \cref{sec:introduction}; we hope to detect the valley of this U-shaped MSE curve. 

Then how to gauge the MSE curve \textbf{without knowing $\mb x$}? We consider the running variance (VAR): 
\begin{align}  \label{eq:lp_argument_key2}
    \mathrm{VAR}\paren{t} \doteq \frac{1}{W}\sum_{w=0}^{W-1} \|\mb x^{t+w} - 1/W \cdot \sum_{i=0}^{W-1} \mb x^{t+i}\|_F^2. 
\end{align}
Initially, the models quickly learn the desired visual content, resulting in a monotonic and rapidly decreasing MSE curve (see \cref{fig:mse_var_trend}). So we expect the running variance of $\set{\mb x^t}_{t \ge 1}$ to also drop quickly, as shown in \cref{fig:mse_var_trend}. When the iteration is near the MSE valley, all $\mb x^t$' s are near, but scattered around $\mb x$. So $\frac{1}{W}\sum_{i=0}^{W-1} \mb x^{t+i} \approx \mb x$ and $\mathrm{VAR}\paren{t} \approx \frac{1}{W}\sum_{w=0}^{W-1} \norm{\mb x^{t+w} - \mb x}_F^2$. Afterward, the noise effect kicks in and the MSE curve bounces back, leading to a similar bounce back in the VAR curve as the $\mb x^t$ sequence gradually moves away from $\mb x$. 

\begin{wraptable}{r}{0.45\linewidth}
    % \resizebox{\linewidth}{!}{
    \centering 
    \vspace{-2em}
    \caption{ES-WMV (our method) on real-world image denoising for \textbf{1024 images}: mean and \scriptsize{(std)} on the images. \footnotesize{(\textbf{D}: detected)}}
    \label{tab:reall}
    \setlength{\tabcolsep}{1mm}{
    \begin{tabular}{c c c c c}
    %\toprule
    \hline
    %\multirow{2}{*}{\scriptsize{Learning Rate}}
    \scriptsize{$\ell$ (loss)}
    &
    \multicolumn{1}{c}{\scriptsize{PSNR (\textbf{D})}} &
    \multicolumn{1}{c}{\scriptsize{PSNR Gap}}
    &
    \multicolumn{1}{c}{\scriptsize{SSIM (\textbf{D})}} &
    \multicolumn{1}{c}{\scriptsize{SSIM Gap}}
    \\
    \hline
    \scriptsize{MSE}
    & \scriptsize{34.04} \tiny({3.68})
    & \textcolor{red}{\scriptsize{0.92}} \tiny({0.83})
    & \scriptsize{0.92} \tiny({0.07})
    & \textcolor{red}{\scriptsize{0.02}} \tiny({0.04})
    \\
    
    % \hline
    \scriptsize{$\ell_1$}
    & \scriptsize{33.92} \tiny({4.34})
    & \textcolor{red}{\scriptsize{0.92}} \tiny({0.59})
    & \scriptsize{0.93} \tiny({0.05})
    & \textcolor{red}{\scriptsize{0.02}} \tiny({0.02})
    \\
    
    % \hline
    \scriptsize{Huber}
    & \scriptsize{33.72} \tiny({3.86})
    & \textcolor{red}{\scriptsize{0.95}} \tiny({0.73})
    & \scriptsize{0.92} \tiny({0.06})
    & \textcolor{red}{\scriptsize{0.02}} \tiny({0.03})
    \\
    \hline
    \end{tabular}
    }
    % }
    % % \vspace{-1em}
\end{wraptable}
This argument suggests a U-shaped VAR curve and the curve should follow the trend of the MSE curve, with approximately aligned valleys, which in turn are aligned with the PSNR peak. To quickly verify this, we randomly sample $1024$ images from the RGB track of the NTIRE 2020 Real Image Denoising Challenge~\citep{abdelhamed2020ntire}, and perform DIP-based image denoising (i.e. $\min\; \ell(\mb y, G_{\mb \theta}(\mb z))$ where $\mb y$ denotes the noisy image). \cref{tab:reall} reports the average detected PSNR/SSIM and the average detection gaps based on our ES method (see \cref{alg:framework}) that tries to detect the valley of the VAR curve. On average, the detection gaps are $\le 0.95$ in PSNR and $\le 0.02$ in SSIM, and the difference in visual qualities is typically barely noticeable by eyes! Furthermore, we provide histograms of the PSNR and SSIM gaps in \cref{fig:hist}. For more than $95\%$ of the images, our ES method obtains a PSNR gap less than $2dB$.

% More details are given in \cref{fig:denoising_example}, and \cref{sec:expriments,sec:denoising_eg}.

\begin{wrapfigure}{r}{0.5\textwidth}
\vspace{-2em}
\begin{minipage}{0.5\textwidth}
    {\small
\begin{algorithm}[H]
% \small
\caption{DIP with ES--WMV}
\label{alg:framework} 
\begin{algorithmic}[1]
\Require random seed $\mb z$, randomly-initialized $\mb \theta^0$, window size $W$, patience $P$, empty queue $\mc Q$, iteration counter $k = 0$, $\mathrm{VAR}_{\min} = \infty$
\Ensure reconstruction $\mb x^{*}$
\While{not stopped}
\State update $\mb \theta$ via \cref{eq:dip} to obtain $\mb \theta^{k+1}$ and $\mb x^{k+1}$
\State push $\mb x^{k+1}$ to $\mc Q$, pop queue if $\abs{\mc Q} > W$
\If{$\abs{\mc Q} = W$}
\State compute $\mathrm{VAR}$ of elements in $\mc Q$ via \cref{eq:lp_argument_key2}
\If{$\mathrm{VAR} < \mathrm{VAR}_{\min}$}
\State $\mathrm{VAR}_{\min} \leftarrow \mathrm{VAR}$, $\mb x^{*} \leftarrow \mb x^{k+1}$
\EndIf
\If{$\mathrm{VAR}_{\min}$ stagnates for $P$ iterations}
\State stop and return $\mb x^\ast$ 
\EndIf
\EndIf
\State $k = k+1$
\EndWhile
\end{algorithmic}
\end{algorithm}
    }
\end{minipage}
% % \vspace{-1em}
\end{wrapfigure}

\paragraph{\textcolor{umn_maroon}{Detecting transition by running variance}}
Our lightweight method only involves computing the VAR curve and numerically detecting its valley---\textbf{the iteration stops once the valley is detected}. To obtain the curve, we set a window size parameter $W$ and compute the windowed moving variance (WMV). To robustly detect the valley, we introduce a patience number $P$ to tolerate up to $P$ consecutive steps of variance stagnation. Obviously, the cost is dominated by the calculation of variance per step, which is $O(WN)$ ($N$ is the size of the visual object). In comparison, a typical gradient update step for solving \cref{eq:dip} costs at least $\Omega(\abs{\mb \theta} N)$, where $\abs{\mb \theta}$ is the number of parameters in the DNN $G_{\mb \theta}$. Since $\abs{\mb \theta}$ is typically much larger than $W$ (default: $100$), our running VAR and detection incur very little computational overhead. Our entire algorithmic pipeline is summarized in \cref{alg:framework}.
\begin{figure*}[!htbp]
    \centering 
    % \vspace{-1em}
    \includegraphics[width=\linewidth]{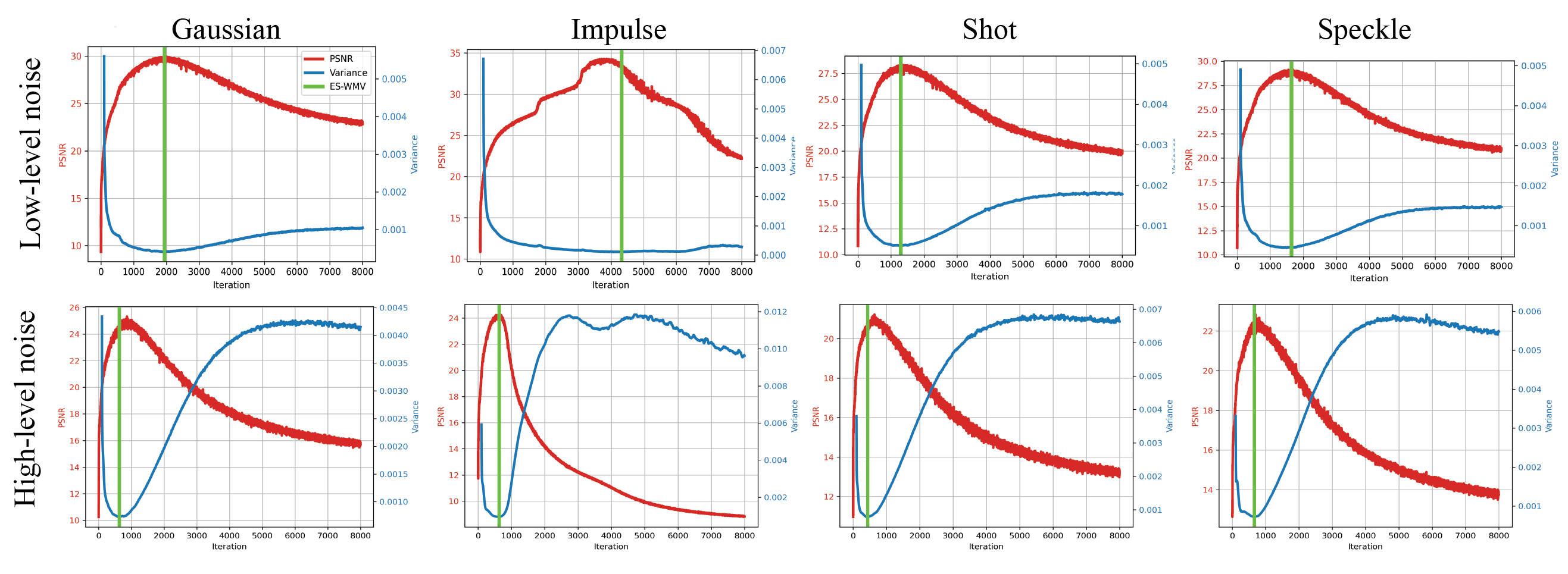}
    \vspace{-1em}
    \caption{Our ES-WMV method on DIP for denoising ``F16" with different noise types and levels (top: low-level noise; bottom: high-level noise). \textcolor{myred}{Red curves} are PSNR curves, and \textcolor{myblue}{blue curves} are VAR curves. The \textcolor{mygreen}{green bars} indicate the detected ES point. }
    \label{fig:denoising_example}
    % % \vspace{-1em}
\end{figure*} 
% \begin{figure*}[!htbp]
%     \centering 
%     % \vspace{-1em}
%     \includegraphics[width=\linewidth]{figures/helper_example-01.png}
%     % \vspace{-3em}
%     \caption{ES-WMV on {deep decoder}, GP-DIP, DIP-TV, and SIREN for denoising "F16" with different levels of Gaussian noise (top: low-level noise; bottom: high-level noise). \textcolor{myred}{Red curves} are PSNR curves, and \textcolor{myblue}{blue curves} are VAR curves. The \textcolor{mygreen}{green bars} indicate the detected ES point. (We sketch the details of the DIP variants above in \cref{sec:details_variants})}
%     \label{fig:helper_example}
%     % % \vspace{-1em}
% \end{figure*}
To confirm the effectiveness, we provide qualitative samples in \cref{fig:denoising_example,fig:helper_example}, with more quantitative results included in the experiment part (\cref{sec:expriments}; see also \cref{tab:reall}). \cref{fig:denoising_example} shows that for image denoising with different noise types/levels, our ES method can detect ES points that achieve near-peak performance. Similarly, our method remains effective in several popular DIP variants, as shown in \cref{fig:helper_example}. Note that although our detection for DIP-TV in \cref{fig:helper_example} is a bit far from the peak in terms of iteration count (as the VAR curve is almost flat after the peak), the detection gap is still small ($\sim 1.29\mathrm{dB}$).

\paragraph{\textcolor{umn_maroon}{Seemingly similar ideas}} 
Our running variance and its U-shaped curve are reminiscent of the classical U-shaped bias-variance tradeoff curve and therefore validation-based ES~\citep{geman_neural_1992,yang_rethinking_2020}. But there are crucial differences: (1) our learning setting is not supervised; (2) the variance in supervised learning is with respect to the sample distribution, while our variance here pertains to the $\set{\mb x^t}_{t \ge 1}$ sequence. As discussed in \cref{sec:introduction}, we cannot directly apply validation-based ES, although it is possible to heuristically emulate it by splitting the elements in $\mb y$~\citep{yaman2021zeroshot,ding_validation_2022}---which might be problematic for nonlinear IPs. Another line of related ideas is the detection of variance-based online change points in time series analysis~\citep{AminikhanghahiCook2016survey}, where the running variance is often used to detect shifts in means under the assumption that the means are piecewise constant. Here, the piecewise constancy assumption does not hold for our $\set{\mb x^t}_{t \ge 1}$. 

\begin{figure*}[!htbp]
    \centering 
    % \vspace{-1em}
    \includegraphics[width=\linewidth]{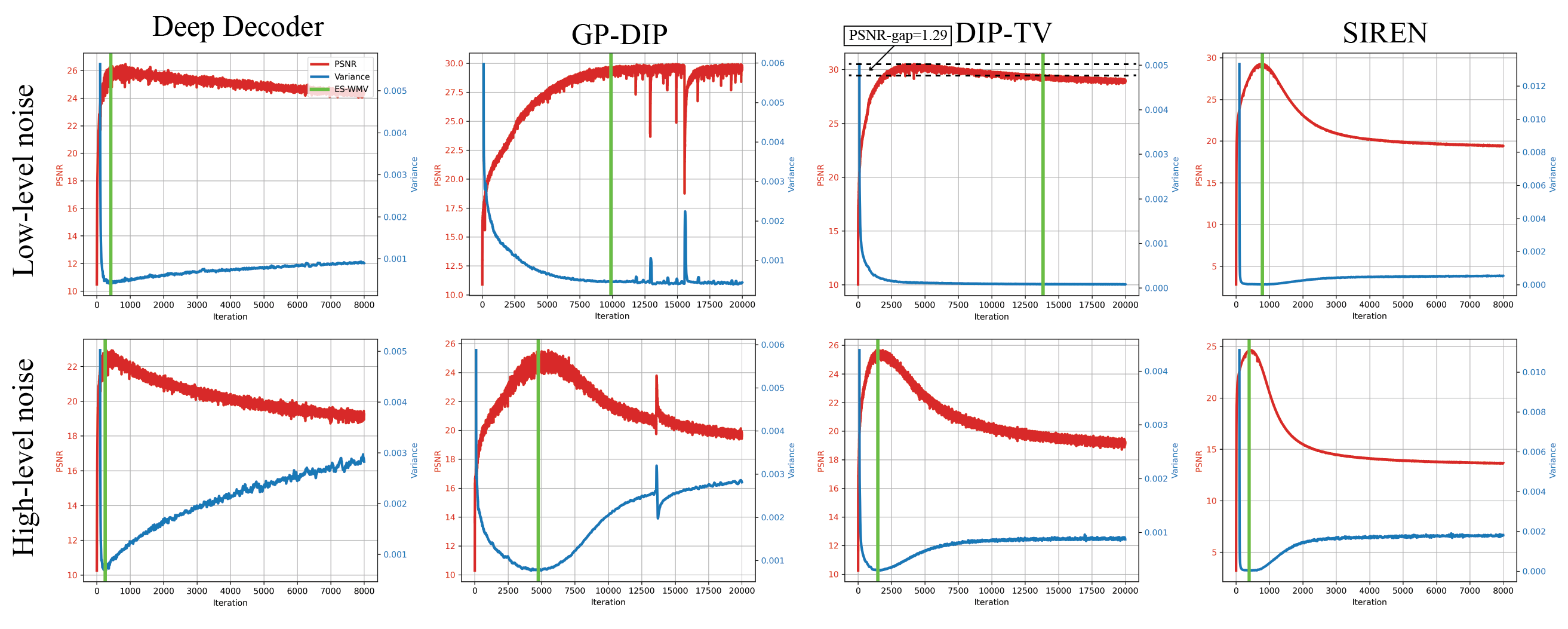}
    \vspace{-1em}
    \caption{ES-WMV on {deep decoder}, GP-DIP, DIP-TV, and SIREN for denoising ``F16'' with different levels of Gaussian noise (top: low-level noise; bottom: high-level noise). \textcolor{myred}{Red curves} are PSNR curves, and \textcolor{myblue}{blue curves} are VAR curves. The \textcolor{mygreen}{green bars} indicate the detected ES point. (We sketch the details of the above DIP variants in \cref{sec:details_variants})}
    \label{fig:helper_example}
    % % \vspace{-1em}
\end{figure*}

\paragraph{\textcolor{umn_maroon}{Theoretical justification}}
We can make our heuristic argument in \cref{sec:method_ituit} more rigorous by restricting ourselves to additive denoising, that is, $\mb y = \mb x + \mb n$, {where the noise $\mb n \sim \mc N\left(\mathbf{0}, \xi^{2}/n \cdot \mathbf{I}\right)$,} and appealing to the popular linearization strategy (i.e. neural tangent kernel~\cite{JacotEtAl2018Neural,heckel2020denoising}) in understanding DNN. The idea is based on the assumption that during DNN training $\mb \theta$ does not move much away from initialization $\mb \theta^0$, so that the learning dynamic can be approximated by that of a linearized model, i.e. suppose that we take the MSE loss, 
\begin{equation}  \label{eq:linearized_obj}
    \norm{\mb y - G_{\mb \theta} \paren{\mb z}}_2^2 \approx  
    \\
    \norm{\mb y - G_{\mb \theta^0} \paren{\mb z} - \mb J_G\paren{\mb \theta^0}\paren{\mb \theta - \mb \theta^0}}_2^2 \doteq \wh{f}\paren{\mb \theta}, 
\end{equation}
where $\mb J_G\paren{\mb \theta^0}$ is the Jacobian of $G$ with respect to $\mb \theta$ at $\mb \theta^0$, and 
$G_{\mb \theta^0} \paren{\mb z} + \mb J_G\paren{\mb \theta^0}\paren{\mb \theta - \mb \theta^0}$ is the first-order Taylor approximation to $G_{\mb \theta}(\mb z)$ around $\mb \theta^0$. $\wh{f}\paren{\mb \theta}$ is simply a linear least-squares objective. We can directly calculate the running variance based on the linear model, as shown below.   
\begin{theorem}\label{prop:first_stage}
     Let $\sigma_i$'s and $\mb w_i$'s be the singular values and left singular vectors of $\mb J_G(\mb \theta^0)$, and suppose that we run a gradient descent with step size $\eta$ on the linearized objective $\wh{f}\paren{\mb \theta}$ to obtain $\set{\mb \theta^t}$ and $\set{\mb x^t}$ with $\mb x^t \doteq G_{\mb \theta^0} \paren{\mb z} + \mb J_G(\mb \theta^0)(\mb \theta^t - \mb \theta^0)$. Then, provided that $\eta \le 1/\max_i(\sigma_i^2)$, 
    \begin{align}
        \mathrm{VAR}\paren{t} = \sum_i C_{W, \eta, \sigma_i} \innerprod{\mb w_i}{\wh{\mb y}}^2 \paren{1- \eta \sigma_i^2}^{2t}, 
    \end{align}
    where $\wh{\mb y} = \mb y - G_{\mb \theta^0} (\mb z)$, and $C_{W, \eta, \sigma_i} \ge 0$ depend only on $W$, $\eta$, and $\sigma_i$ for all $i$. 
\end{theorem}

\begin{wrapfigure}{r}{0.35\textwidth}
    \centering
    \vspace{-1em}
    \includegraphics[width=1\linewidth]{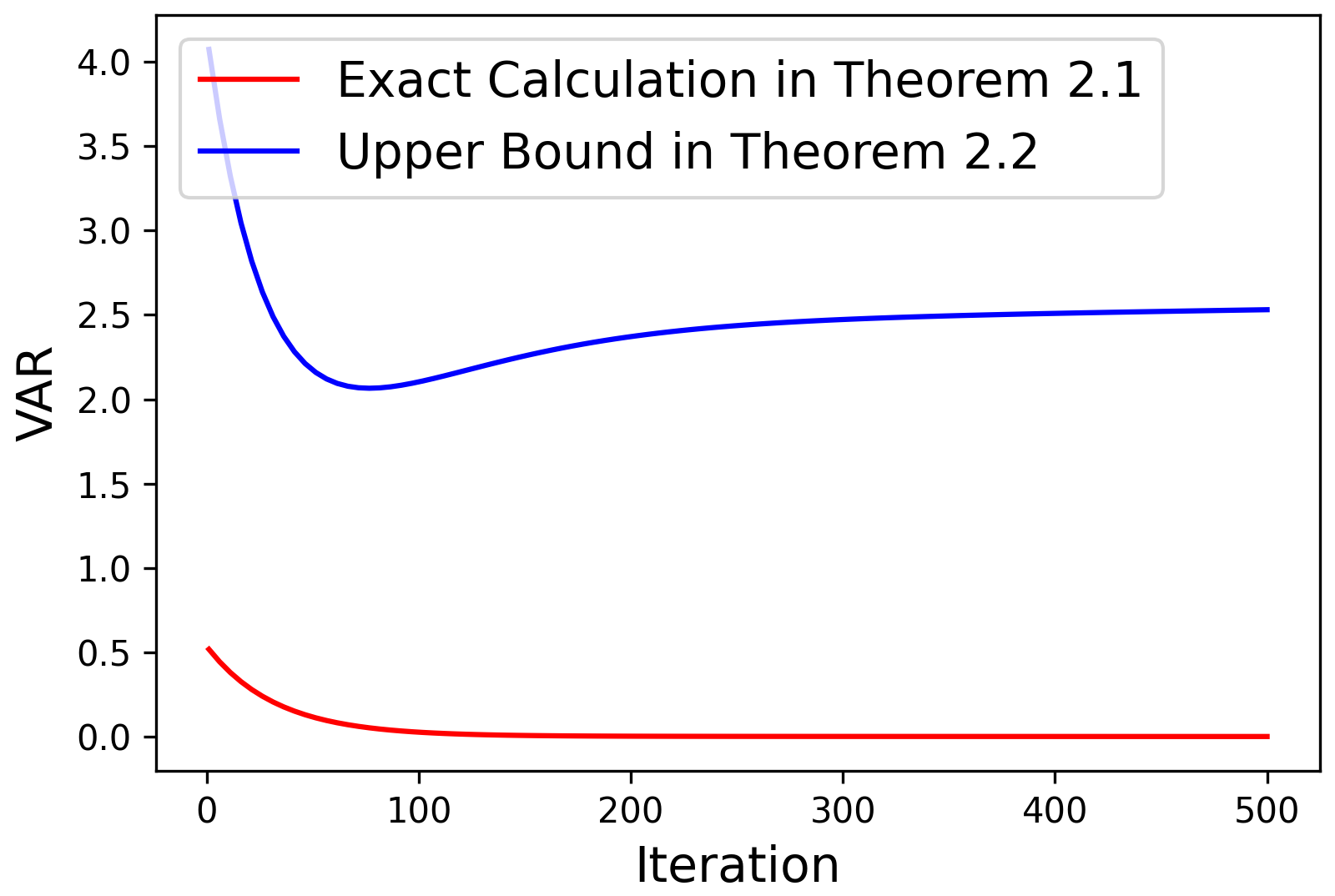}
    \vspace{-2em}
    \caption{The exact and upper bounds predicted by \cref{prop:first_stage,thm:upper_bound}. }
    % \vspace{-2em}
    \label{fig:theo_ex}
\end{wrapfigure}

The proof can be found in \cref{sec:proof_exact_trend}. 
\cref{prop:first_stage} shows that if the learning rate (LR) $\eta$ is sufficiently small, the WMV of $\set{\mb x^t}$ decreases monotonically. 
% The result correctly predicts the initial trend, but conflicts with the climbing trend after the performance peak. This is a common limitation of the current linearization-based deep learning theory~\citep{ChizatEtAl2018Lazy,LiuEtAl2020linearity}, as the late-stage learning dynamic may deviate substantially from that of the linearized version. Another factor is that practical LR is not very small. 
We can develop a complementary upper bound for the WMV that has a U shape. 
To this end, we make use of Theorem 1 of \cite{heckel2020denoising}, which can be summarized (some technical details omitted; precise statement is reproduced in \cref{sec:proof_upper}) as follows: consider the two-layer model $G_{\mb C} \paren{\mb B} = \operatorname{ReLU}(\mathbf{U B C}) \mathbf{v}$, where $\mb C\in\R^{n\times k}$ models $1\times 1$ trainable convolutions, $\mb v\in\RR^{k\times 1}$ contains fixed weights, $\mb U$ is an upsampling operation, and $\mb B$ is the fixed random seed. Let $\mb J$ be a reference Jacobian matrix solely determined by the upsampling operation $\mb U$, and $\sigma_i$'s and $\mb w_i$'s the singular values and left singular vectors of $\mb J$. Assume that $\mb x \in \mathrm{span}\set{\mb w_1, \dots, \mb w_p}$. Then, when $\eta$ is sufficiently small, with high probability, 
\begin{equation}
\left\|G_{\mb C^t} \paren{\mb B} - \mb x \right\|_{2} \leq\left(1-\eta \sigma_{p}^{2}\right)^t\| \mb x \|_{2}+E(\mb n)+\epss\| \mb y \|_{2}, 
\end{equation}
where $\epss >0$ is a small scalar related to the structure of the network and $E(\mb n)$ is the error introduced by noise: 
$E^2(\mb n)\doteq \sum_{j=1}^{n}((1-\eta \sigma_{j}^{2})^{t}-1)^{2}\langle\mathbf{w}_{j}, \mb n\rangle^{2}$.
So, if the gap $\sigma_p/\sigma_{p+1} > 1$, $\left\|G_{\mb C^t} \paren{\mb B} - \mb x \right\|_{2}$ is dominated by $\left(1-\eta \sigma_{p}^{2}\right)^t\| \mb x \|_{2}$ when $t$ is small and then by $E(\mb n)$ when $t$ is large. However, since the former decreases and the latter increases as $t$ grows, the upper bound has a U shape with respect to $t$. On the basis of this result, we have the following. 
\begin{theorem} \label{thm:upper_bound}
Assume the same setting as Theorem 2 of~\cite{heckel2020denoising}. With high probability, our WMV is upper bounded by 
\begin{equation}
    \frac{12}{W} \norm{\mb x}_2^2 \frac{ \left(1-\eta \sigma_{p}^{2}\right)^{2t}}{1-(1-\eta \sigma_p^2)^2} + \\
    12 \sum_{i=1}^n \paren{\paren{1 - \eta \sigma_i^2}^{t + W-1} - 1}^2 \paren{\mb w_i^\T \mb n}^2 + 12\epss^2 \norm{\mb y}_2^2. 
\end{equation}
\end{theorem}

The exact statement and proof can be found in \cref{sec:proof_upper}. Using a reasoning process similar to that above, we can conclude that the upper bound in \cref{thm:upper_bound} also has a U shape. To interpret the results, \cref{fig:theo_ex} shows the curves (as functions of $t$) predicted by \cref{prop:first_stage,thm:upper_bound}. 
% We reiterate that the exact calculation is conditioned on the validity of linearization, which is justified for the initial stage but not necessarily for the whole process. 
The actual VAR curve should be between the two curves. These results are primitive and limited, similar to the situations for many {deep learning} theories that provide loose upper and lower bounds; we leave a complete theoretical justification for future work.  

\paragraph{\textcolor{umn_maroon}{A memory-efficient variant}}
While \cref{alg:framework} is already lightweight and effective in practice, we can modify it slightly to avoid maintaining $\mc Q$ and therefore saving memory. The trick is to use exponential moving variance (EMV), together with exponential moving average (EMA), shown in \cref{sec:alg2}. The hard window size parameter $W$ is now replaced by the soft forgetting factor $\alpha$: the larger the $\alpha$, the smaller the impact of the history, and hence a smaller effective window. We systematically compare
ES-WMV with ES-EMV in \cref{sec:es_wmv_emv} for image denoising tasks. The latter has slightly better detection due to the strong smoothing effect ($\alpha = 0.1$). For this paper, we prefer to remain simple and leave systematic evaluations of ES-EMV on other IPs for future work.

% \begin{algorithm}[!htbp]
%     % \small 
% \caption{DIP with ES--EMV}
% \label{alg:framework_emavg} 
% \begin{algorithmic}[1]
% \Require random seed $\mb z$, randomly-initialized $G_{\mb \theta}$, forgetting factor $\alpha \in (0, 1)$, patience number $P$, iteration counter $k = 0$, $\mathrm{EMA}^0 = 0$, $\mathrm{EMV}^0 = 0$, $\mathrm{VAR}_{\min} = \infty$
% \Ensure reconstruction $\mb x^{*}$
% \While{not stopped}
% \State update $\mb \theta$ via \cref{eq:dip} to obtain $\mb \theta^{k+1}$ and $\mb x^{k+1}$
% \State $\mathrm{EMA}^{k+1} = \paren{1-\alpha} \mathrm{EMA}^{k} + \alpha \mb x^{k+1}$
% \State $\mathrm{EMV}^{k+1} = \paren{1-\alpha} \mathrm{EMV}^{k} + \alpha(1-\alpha)\|\mb x^{k+1} -\mathrm{EMA}^{k}\|_2^2$\hspace{-0.5em}
% \If{$\mathrm{VAR} < \mathrm{VAR}_{\min}$}
% \State $\mathrm{VAR}_{\min} \leftarrow \mathrm{VAR}$, $\mb x^{*} \leftarrow \mb x^{k+1}$
% \EndIf
% \If{$\mathrm{VAR}_{\min}$ stagnates for $P$ iterations}
% \State stop and return $\mb x^\ast$ 
% \EndIf
% \State $k = k+1$
% \EndWhile
% \end{algorithmic}
% \end{algorithm}

%% file: sections/Sec4_Exp.tex
% \begin{wrapfigure}{r}{0.55\textwidth}
%     \centering
%     % \vspace{-4em}
%     \includegraphics[width=1\linewidth]{figures/baseline_l_psnr_2rows-01.png}
%     % \vspace{-2em}
%     \caption{Baseline ES vs our ES-WMV on denoising with \textbf{low-level noise}. For NIMA, we report both technical quality assessment (NIMA-q) and aesthetic assessment (NIMA-a). Smaller PSNR gaps are better.}
%     \label{fig:baseline_l_psnr}
%     % \vspace{-1em}
% \end{wrapfigure}
\section{Experiments}\label{sec:expriments}
{We test ES-WMV for DIP in \textbf{image denoising, inpainting, demosaicing, super-resolution, MRI reconstruction, and blind image deblurring}}, spanning both linear and nonlinear IPs. For image denoising, we also systematically evaluate ES-WMV for main DIP variants, including {deep decoder}~\citep{heckel2018deep}, DIP-TV~\citep{cascarano2021combining}, GP-DIP~\citep{ChengEtAl2019Bayesian}, and demonstrate ES-WMV as a reliable helper to detect good ES points. Details of the DIP variants are discussed in \cref{sec:details_variants}. We also compare ES-WMV with the main competing methods, including DF-STE~\citep{jo2021rethinking}, SV-ES~\citep{Li2021}, DOP~\citep{you2020robust}, SB~\citep{shi2021measuring}, and VAL~\citep{yaman2021zeroshot,ding_validation_2022}. Details of the main ES-based methods can be found in \cref{sec:details_ES_methods}. We use both PSNR and SSIM to assess reconstruction quality and report PSNR and SSIM gaps (the difference between our detected
and peak numbers) as indicators of our detection performance. {\textbf{Common acronyms, pointers to external codes, detailed experiment settings, real-world denoising, image inpainting, and image demosaicing are in \cref{sec:acronyms,sec:external_codes,sec:exp_setting,sec:real_ap,sec:inpainting,sec:raw_img_demosaicing}, respectively.}}

\begin{wrapfigure}{r}{0.55\textwidth}
    \centering
    \vspace{-1em}
    \includegraphics[width=1\linewidth]{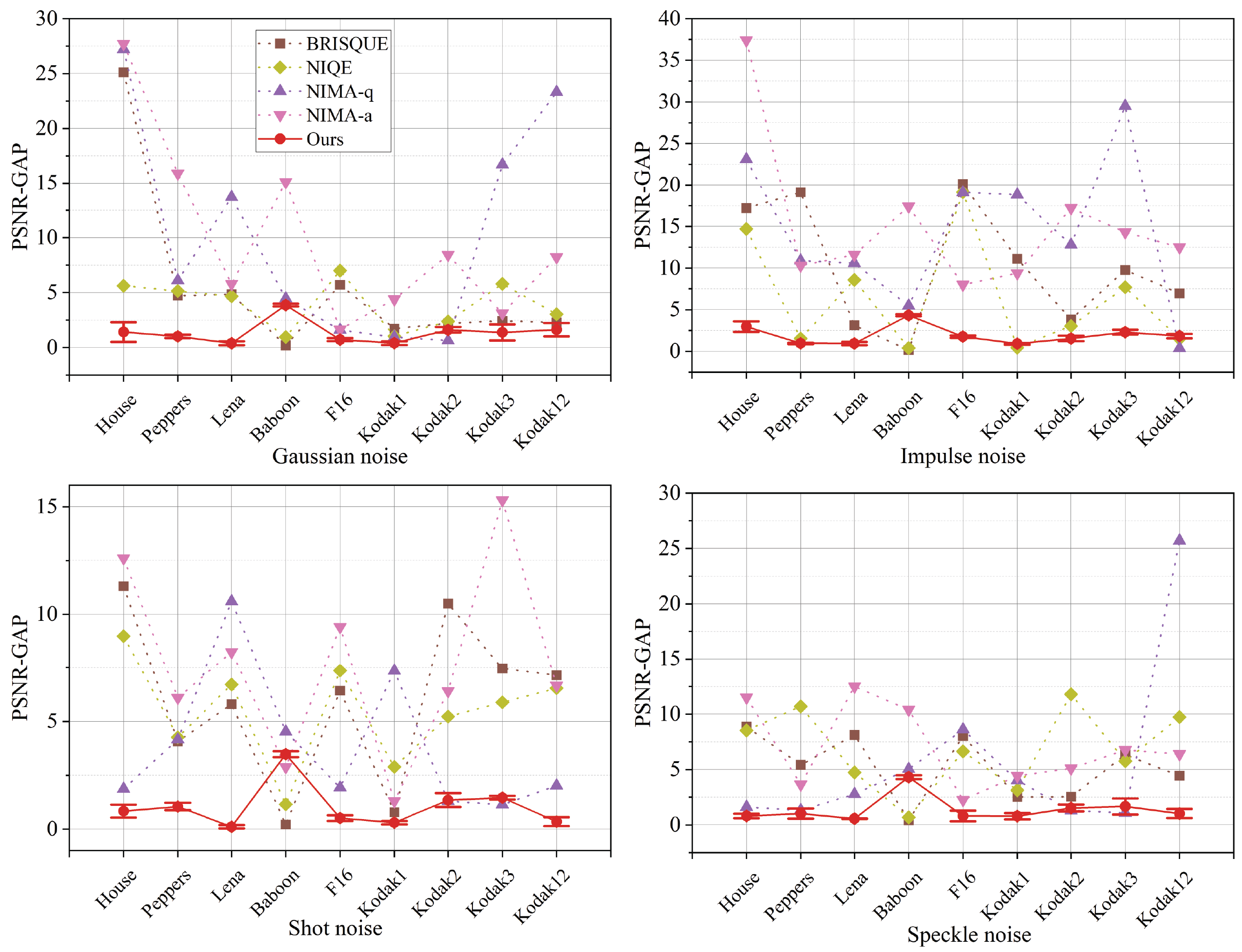}
    \vspace{-1em}
    \caption{Baseline ES vs our ES-WMV on denoising with \textbf{low-level noise}. For NIMA, we report both technical quality assessment (NIMA-q) and aesthetic assessment (NIMA-a). Smaller PSNR gaps are better.}
    \label{fig:baseline_l_psnr}
    \vspace{-1em}
\end{wrapfigure}
% \begin{figure*}[!htbp]
%     \centering 
%     % \vspace{-1em}
%     \includegraphics[width=0.9\linewidth]{figures/noref_2.png}
%     % \vspace{-1em}
%     \caption{Visual comparisons of NR-IQMs and ES-WMV. From top to bottom: shot noise (low), shot noise (high), speckle noise (low), speckle noise (high).}
%     \label{fig:noref_2}
%     % \vspace{-1em}
% \end{figure*} 

\subsection{Image denoising}
\label{sec:exp_ID}
%%%%%%%%%%%%%%%%%%%%%%%%%%%%%%%%%%%%%%%%%%%%%%%%%%%%%%%%%%%%%%%%%%%%%%%%%%%%%%%%%%%%%%%%%%%%%%%%%%%%%%%%%%%%%%%%%%%%%%%%%%%%%%%%%%%%%%%%%%%%%% 

Prior work dealing with DIP overfitting mostly focuses on image denoising and typically only evaluates their methods on one or two kinds of noise with low noise levels, e.g., low-level Gaussian noise. To stretch our evaluation, we consider $4$ types of noise: Gaussian, shot, impulse, and speckle. We take the classical 9-image dataset~\citep{Dabov2007}, and for each noise type, generate two noise levels, low and high, i.e., level 2 and 4 of~\cite{hendrycks2019robustness}, respectively. In \cref{tab:reall} and \cref{sec:real_ap}, we also report the performance of our ES-WMV on real-world denoising evaluated on \textbf{large-scale datasets}. In addition, we also compare DIP-based denoising with a {state-of-the-art} diffusion-model-based denoising in \cref{tab:denoising_dm}.

\paragraph{\textcolor{umn_maroon}{Comparison with baseline ES methods}} \label{sec:baseline}
It is natural to expect that NR-IQMs, such as the classical BRISQUE~\citep{mittal2012no}, NIQE~\citep{mittal2012making}, and modern DNN-based NIMA~\citep{talebi2018nima}, can be used to monitor the quality of intermediate reconstructions and hence induce natural ES criteria.  Therefore, we set $3$ baseline methods using BRISQUE, NIQE, and NIMA, respectively, and seek the optimal $\mb x^t$ using these metrics. \cref{fig:baseline_l_psnr} presents the comparison (in terms of PSNR gaps) of these $3$ methods with our ES-WMV on denoising with low-level noise by using DIP; results on high-level noise and also as measured by SSIM are included in \cref{sec:baselines_ap}. Visual comparisons between our ES-WMV and the baseline methods are shown in \cref{fig:noref_2,fig:noref_1}. While \textbf{our method enjoys favorable detection gaps ($\le 2$)} for most tested noise types/levels (except for Baboon, Kodak1, Kodak2 for certain noise types/levels; DIP itself is suboptimal in terms of denoising such images with substantial high-frequency components), \textbf{the baseline methods can see huge detection gaps up to $10$.}

\begin{figure*}[!htbp]
    \centering 
    % \vspace{-1em}
    \includegraphics[width=0.95\linewidth]{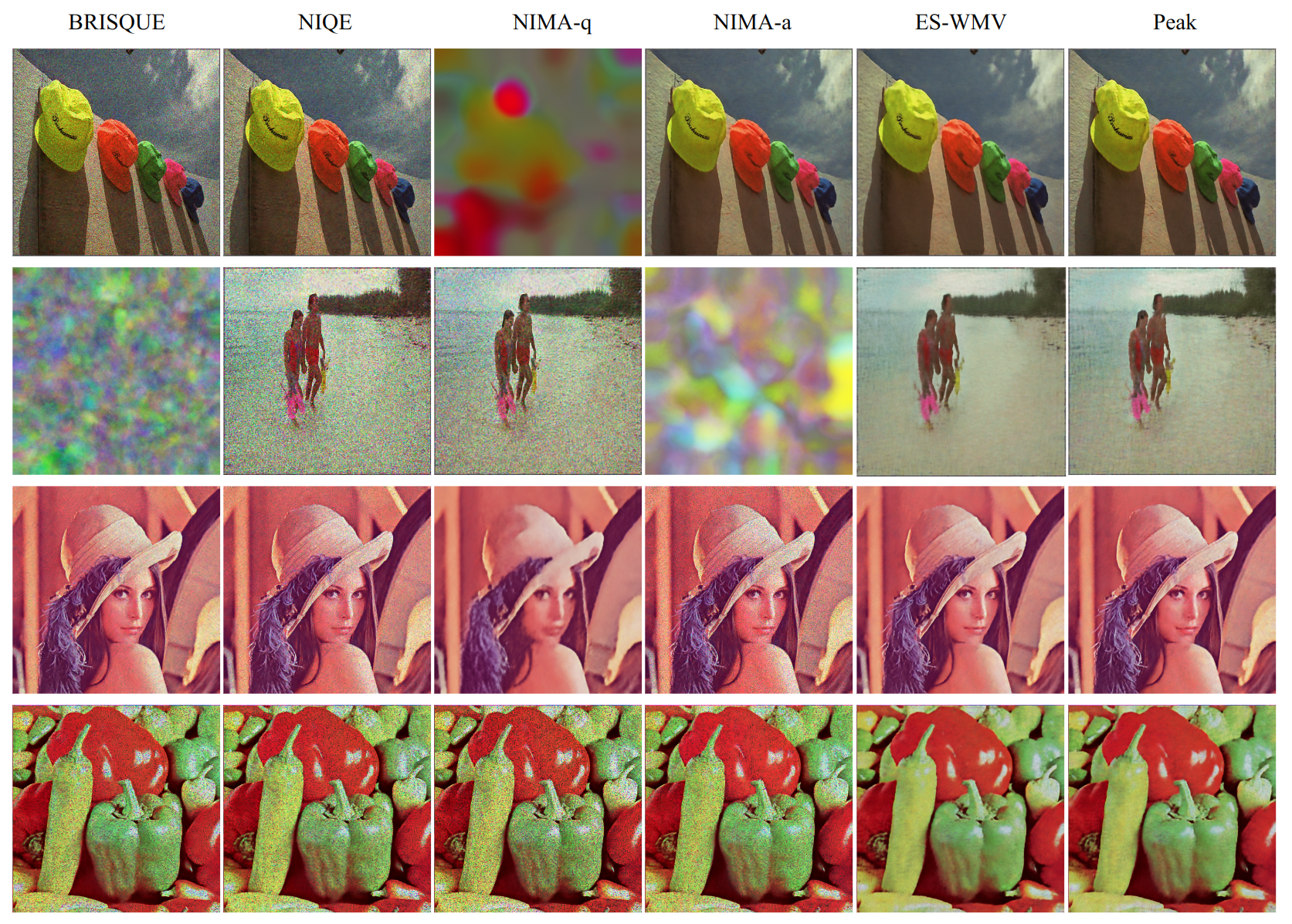}
    % \vspace{-1em}
    \caption{Visual comparisons of the detected images by NR-IQMs and ES-WMV. From top to bottom: shot noise (low), shot noise (high), speckle noise (low), speckle noise (high).}
    \label{fig:noref_2}
    \vspace{-1em}
\end{figure*} 
% We hence do not report the mean and standard deviation for these methods, as it is unlikely these bad detection results are due to randomness.

% \cref{fig:baseline_l_ssim} and \cref{fig:baseline_h_ssim}, respectively, which show a similar trend to the results of PSNR gaps. The detection gaps of our method are very marginal ($< 0.02$) for most noise types and levels (except for Baboon and Kodak1 for certain noise types/levels), while the baseline methods can well exceed $0.1$ for most cases.

\paragraph{\textcolor{umn_maroon}{Competing methods}}
\label{sec:exp_competing_methods}
\begin{wrapfigure}{r}{0.55\textwidth}
    \centering
    \vspace{-2em}
    \includegraphics[width=1\linewidth]{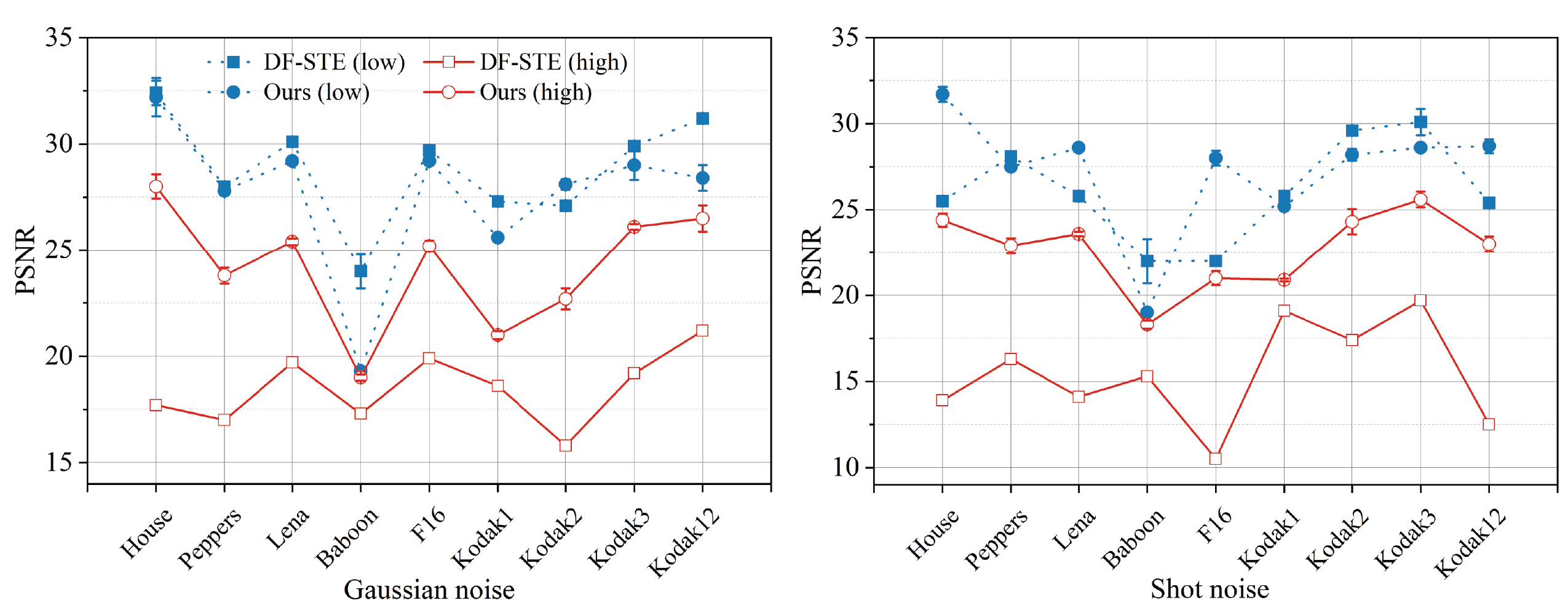}
    \vspace{-2em}
    \caption{Comparison of DF-STE and ES-WMV for Gaussian and shot noise in terms of PSNR.}
    % \vspace{-1em}
    \label{fig:df_ste_psnr}
\end{wrapfigure}
DF-STE~\citep{jo2021rethinking} is specific for Gaussian and Poisson denoising, and noise variance is needed for their tuning parameters. \cref{fig:df_ste_psnr} presents the comparison of our method with DF-STE in terms of PSNR; SSIM results are in \cref{sec:competing_ap}. Here, we directly report the final PSNRs obtained by both methods. For low-level noise, there is no clear winner. \textbf{For high-level noise, ES-WMV outperforms DF-STE by considerable margins.} Although the right variance level is provided to DF-STE in order to tune their regularization parameters, DF-STE stops after only very few epochs, leading to very low performance and almost zero standard deviations---since they return almost the noisy input. However, we do not perform any parameter tuning for ES-WMV. Furthermore, we compare the two methods on the CBSD68 dataset in \cref{sec:competing_ap} that leads to a similar conclusion.

We report the results of SV-ES in \cref{sec:competing_ap} since ES-WMV performs largely comparable to SV-ES. However, ES-WMV is much faster in wall-clock time, as reported in \cref{tab:wall-clock time}: for each epoch, the overhead of our ES-WMV is less than $3/4$ of the DIP update itself, while SV-ES is around $25\times$ of that. 
\begin{wraptable}{r}{0.5\linewidth}
    \centering
    % \vspace{-1em}
    \caption{Wall-clock time (secs) of DIP and three ES methods per epoch on \textit{NVIDIA Tesla K40 GPU}: mean and \scriptsize{(std)}\normalsize. The total wall clock time should contain both DIP and a certain ES method.}
    \label{tab:wall-clock time}
    \setlength{\tabcolsep}{1mm}{
    \begin{tabular}{c c c c c}
    %\toprule
    \hline
    %\multirow{2}{*}{\scriptsize{Learning Rate}}
    &
    \multicolumn{1}{c}{\footnotesize{DIP}} &
    \multicolumn{1}{c}{\footnotesize{SV-ES}}
    &
    \multicolumn{1}{c}{\footnotesize{ES-WMV}} &
    \multicolumn{1}{c}{\footnotesize{ES-EMV}}
    \\
    \hline
    \footnotesize{Time}
    &\scriptsize{0.448} \tiny({0.030})
    & \textbf{\scriptsize{13.027} \tiny({3.872})}
    & \scriptsize{0.301} \tiny({0.016})
    & \textcolor{red}{\scriptsize{0.003}} \tiny({0.003})
    \\
    \hline
    \end{tabular}
    }
    \vspace{-1em}
    \end{wraptable}
There is no surprise: while our method only needs to update the running variance of $\set{\mb x^t}_{t \ge 1}$ each time, \textbf{SV-ES needs to train a coupled autoencoder which is extremely expensive.}

DOP is \textbf{designed specifically just for impulse noise}, so we compare ES-WMV with DOP on impulse noise (see \cref{sec:competing_ap}). The loss is changed to $\ell_1$ to account for the sparse noise. In terms of the final PSNRs, DOP outperforms DIP with ES-WMV by a small gap, but even the peak PSNR of DIP with $\ell_1$ lags behind DOP by about $2$dB for high noise levels. 

\begin{minipage}{\textwidth}
\begin{minipage}{0.45\textwidth}
% \begin{table}
    \captionof{table}{Comparison between ES-WMV and SB for image denoising on the CBSD68 dataset with varying noise level $\sigma$. The higher PSNR detected and earlier detection are better, which are in \textcolor{red}{red}: mean and \scriptsize{(std)}.}
    \label{tab:sb_table}
    \setlength{\tabcolsep}{0.6mm}{
    \begin{tabular}{c c c c c c c}
    %\toprule
    \hline
    %\multirow{2}{*}{\scriptsize{Learning Rate}}
    &
    \multicolumn{2}{c}{\scriptsize{$\sigma=15$}} &
    \multicolumn{2}{c}{\scriptsize{$\sigma=25$}} &
    \multicolumn{2}{c}{\scriptsize{$\sigma=50$}}
    \\
    \hline
    &
    \multicolumn{1}{c}{\scriptsize{PSNR}} &
    \multicolumn{1}{c}{\scriptsize{Epoch}} &
    \multicolumn{1}{c}{\scriptsize{PSNR}}
    &
    \multicolumn{1}{c}{\scriptsize{Epoch}} &
    \multicolumn{1}{c}{\scriptsize{PSNR}} &
    \multicolumn{1}{c}{\scriptsize{Epoch}}
    \\
    \hline
    \scriptsize{WMV}
    & \scriptsize{28.7}\tiny({3.2})
    & \textcolor{red}{\scriptsize{3962}}\tiny({2506})
    & \textcolor{red}{\scriptsize{27.4}}\tiny({2.6})
    & \textcolor{red}{\scriptsize{3068}}\tiny({2150})
    & \textcolor{red}{\scriptsize{24.2}}\tiny({2.3})
    & \textcolor{red}{\scriptsize{1548}}\tiny({1939})
    \\
    
    % \hline
    % \scriptsize{ES-EMV}
    % & \scriptsize{26.98}
    % & \scriptsize{1377}
    % & \scriptsize{26.40}
    % & \scriptsize{1356}
    % & \scriptsize{24.20}
    % & \scriptsize{773}
    % \\
    % \hline
    \scriptsize{SB}
    & \textcolor{red}{\scriptsize{29.0}}\tiny({3.1})
    & \scriptsize{4908}\tiny({1757})
    & \scriptsize{27.3}\tiny({2.2})
    & \scriptsize{5099}\tiny({1776})
    & \scriptsize{23.0}\tiny({1.0})
    & \scriptsize{5765}\tiny({1346})
    \\
    \hline
    \end{tabular}
    \vspace{-1em}
    }
    % \end{table}
\end{minipage}
\hspace{0.05\textwidth}
\begin{minipage}{0.47\textwidth}
% \begin{figure}
    \includegraphics[width=1\linewidth]{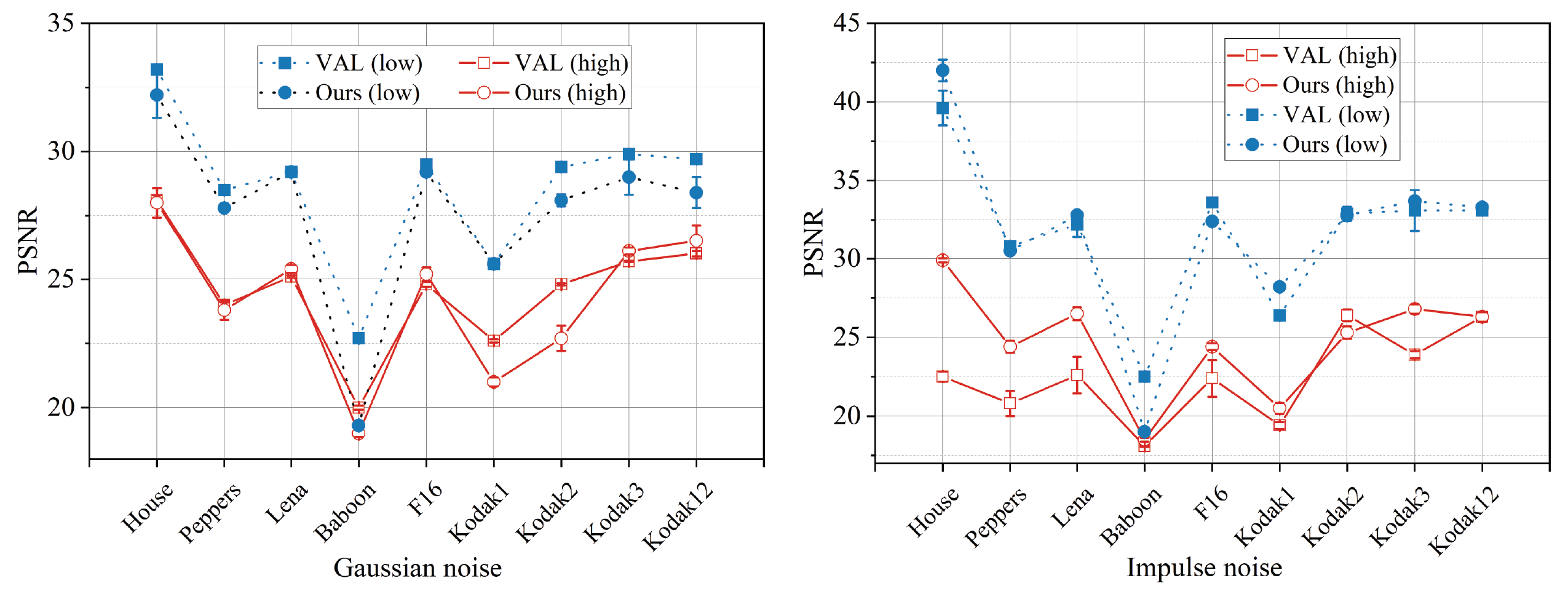}
    \vspace{-2em}
    \captionof{figure}{Comparison of VAL and ES-WMV for Gaussian and impulse noise in terms of PSNR.}
    \vspace{-1em}
    \label{fig:val_psnr}
% \end{figure}
\end{minipage}
\end{minipage}
\vspace{2em}
 
\begin{wrapfigure}{r}{0.55\textwidth}
    \vspace{-1em}
    \centering
    \includegraphics[width=1\linewidth]{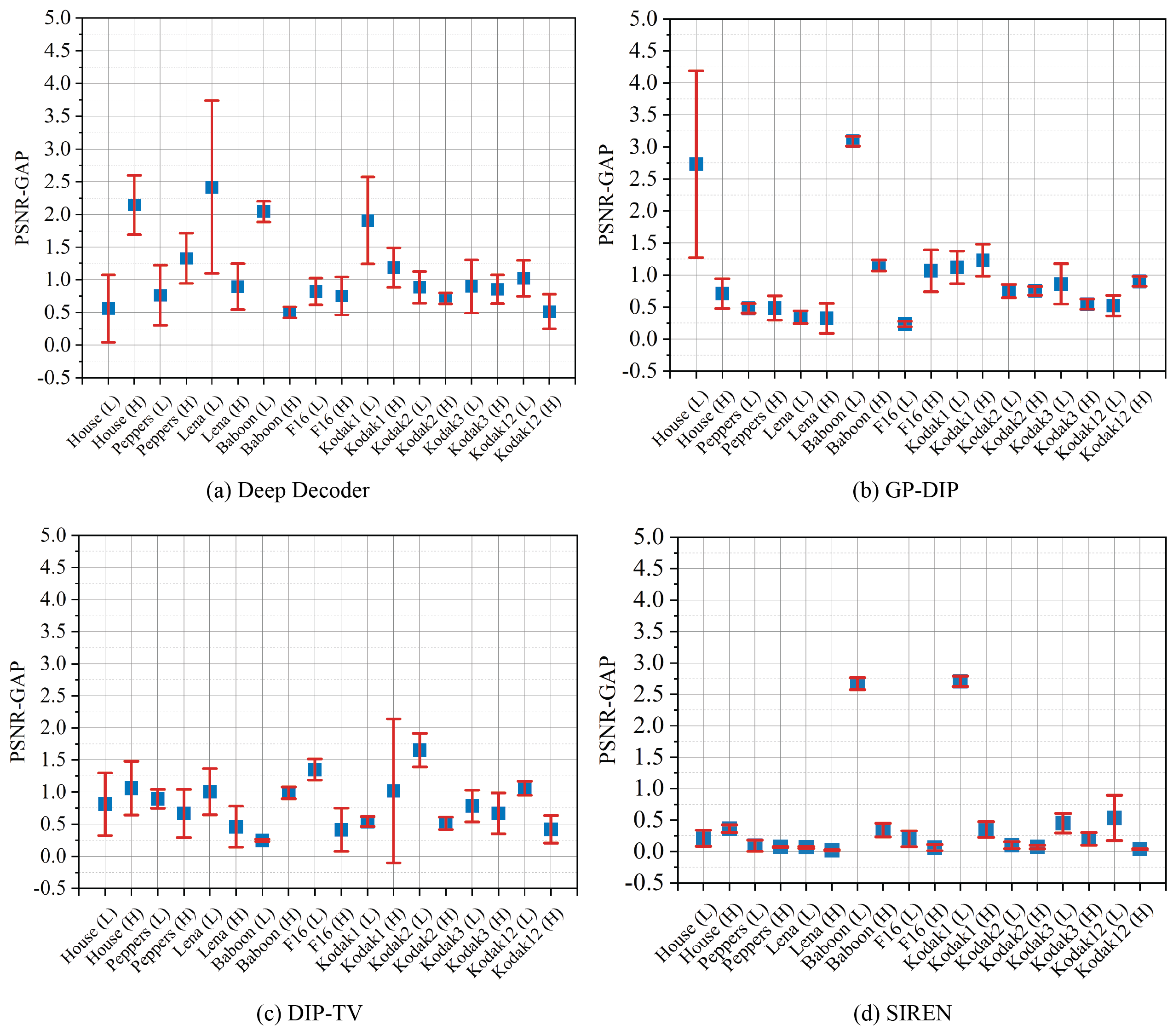}
    \vspace{-2em}
    \caption{Performance of ES-WMV on {deep decoder}, GP-DIP, DIP-TV, and SIREN for Gaussian denoising in terms of PSNR gaps. L: low noise level; H: high noise level.}
    \vspace{-1em}
    \label{fig:dd_gp_tv}
\end{wrapfigure}
\textbf{The ES method in SB is acknowledged by its authors to fail for vanilla DIP~\citep{shi2021measuring}}. Moreover, their modified model still suffers from the overfitting issue beyond very low noise levels, as shown in \cref{fig:SB_compare}. Their ES method fails to stop at appropriate places when the noise level is high. Hence, we test both ES-WMV and SB on their modified DIP model in~\citep{shi2021measuring}, based on the two datasets they test: the classic $9$-image dataset~\citep{Dabov2007} and the CBSD68 dataset~\citep{MartinEtAl2001database}. 
Qualitative results on $9$ images are shown in \cref{sec:competing_ap}; detected PSNR and stop epochs on the CBSD68 dataset are reported in \cref{tab:sb_table}. For SB, the detection threshold parameter is set to $0.01$. It is evident that both methods have similar detection performance for low noise levels, but ES-WMV outperforms SB when the noise level is high. Also, ES-WMV tends to stop much earlier than SB, saving computational cost. 

We compare VAL with our ES-WMV on the $9$-image dataset with low-/high-level Gaussian and impulse noise. Since \cite{ding_validation_2022} takes $90\%$ pixels to train DIP and this usually decreases the peak performance, we report the final PSNRs detected by both methods (see \cref{fig:val_psnr}). The two ES methods \textbf{perform very comparably in image denoising}, probably due to a mild violation of the i.i.d. assumption only, and also to a relatively low degree of information loss due to data splitting. \textbf{The more complex nonlinear BID in \cref{sec:bid_exp} reveals their gap.}

% % \vspace{-1em}

\paragraph{\textcolor{umn_maroon}{ES-WMV as a helper for DIP variants}}
\label{sec:helper_methods}
{Deep decoder}, DIP-TV, and GP-DIP represent different regularization strategies to control overfitting. However, a critical issue is setting the right hyperparameters for them so that overfitting is removed while peak-level performance is preserved. Therefore, practically, these methods are not free from overfitting, especially when the noise level is high. Thus, instead of treating them as competitors, we test whether ES-WMV can reliably detect good ES points for them. We focus on Gaussian denoising and report the results in \cref{fig:dd_gp_tv} (a)-(c) and \cref{sec:helper_ap}. \textbf{ES-WMV is able to attain $\le 1$ PNSR gap for most cases}, with few outliers; we provide a detailed analysis about some of the outliers in \cref{sec:sup_failure}. 

\paragraph{\textcolor{umn_maroon}{ES-WMV as a helper for implicit neural representations (INRs)}} 
INRs, such as \cite{tancik_fourier_2020} and \cite{sitzmann2020implicit}, use multilayer perceptrons to represent highly nonlinear functions in low-dimensional problem domains and have achieved superior results in complex 3D visual tasks. We further extend our ES-WMV to help the INR family and take SIREN~\citep{sitzmann2020implicit} as an example. SIREN parameterizes $\mb x$ as the discretization of a continuous function: this function takes in spatial coordinates and returns the corresponding function values. 
% This parameterization proves beneficial in representing visual objects with substantial high-frequency components, compared to CNNs models used in DIP.
Here, we test SIREN, which is reviewed in \cref{sec:details_variants}, as a replacement for DIP models for Gaussian denoising and summarize the results in \cref{fig:dd_gp_tv} and \cref{fig:dd_gp_tv_ssim}. \textbf{ES-WMV is again able to detect near-peak performance for most images.}

\subsection{Image Super-Resolution} 
\label{sec:super_resolution}

\begin{wraptable}{r}{0.6\textwidth}
% \begin{table*}[!htbp]
\centering 
\vspace{-1em}
\caption{Comparison of ES-WMV for DIP and DDNM+~\cite{wang_zero-shot_2022} for \textbf{\boldsymbol{$2 \times$} image super-resolution with low-level Gaussian and impulse noise}: mean and \scriptsize{(std)}. \normalsize{The highest PSNR and SSIM for each task are in \textcolor{red}{red}}. In particular, we set the best hyperparameter for {DDNM+} ($\sigma_y=0.12$), \textbf{which is unfair for the DIP + ES-WMV combination as we fix its hyperparameter setting}.}
\label{tab:sr_2}
\setlength{\tabcolsep}{1.2mm}{
\begin{tabular}{c c c c c}
    \hline
    & \multicolumn{2}{c}{\footnotesize{PSNR}}                            & \multicolumn{2}{c}{\footnotesize{SSIM}}                            \\ \hline
    & \multicolumn{1}{c}{\footnotesize{Gaussian}} &\footnotesize{Impulse} & \multicolumn{1}{c}{\footnotesize{Gaussian}} & \footnotesize{Impulse}
    \\ \hline
\footnotesize{DIP (peak)}           & \multicolumn{1}{c}{\scriptsize{22.88} \tiny{(1.58)}}            & \textcolor{red}{\scriptsize{28.28}} \tiny{(2.73)}           & \multicolumn{1}{c}{\scriptsize{0.61} \tiny{(0.09)}}             & \textcolor{red}{\scriptsize{0.88}} \tiny{(0.06)}            \\ 
\footnotesize{DIP + ES-WMV} & \multicolumn{1}{c}{\scriptsize{22.11} \tiny{(1.90)}}             & \scriptsize{26.77} \tiny{(3.76)}            & \multicolumn{1}{c}{\scriptsize{0.54} \tiny{(0.11)}}             & \scriptsize{0.86} \tiny{(0.06)}            \\ 
\footnotesize{{DDNM+} ($\sigma_y=.12$)}  & \multicolumn{1}{c}{\textcolor{red}{\scriptsize{25.37}} \tiny{(2.00)}}            & \scriptsize{18.50} \tiny{(0.68)}           & \multicolumn{1}{c}{\textcolor{red}{\scriptsize{0.74}} \tiny{(0.11)}}             & \scriptsize{0.50} \tiny{(0.08)}            \\ 
\footnotesize{{DDNM+} ($\sigma_y=.00$)}         & \multicolumn{1}{c}{\scriptsize{16.91} \tiny{(0.42)}}            & \scriptsize{16.59} \tiny{(0.34)}           & \multicolumn{1}{c}{\scriptsize{0.31} \tiny{(0.09)}}             & \scriptsize{0.49} \tiny{(0.06)}            \\ \hline
\end{tabular}
}
% \vspace{-1em}
\end{wraptable}
% \end{table*}

In this task, we try to recover a clean image $\mb x_0$ from a noisy downsampled version $\mb y = \mc D_t \paren{\mb x_0} + \mb \eps$, where $\mc D_t \paren{\cdot}: [0, 1]^{3 \times tH \times tW} \rightarrow [0, 1]^{3 \times H \times W}$ is a downsampling operator that resizes an image by the factor $t$ and $\mb \eps$ models extra additive noise. We consider the following DIP-reparametrized formulation 
% \begin{equation}
$
    \min_{\mb \theta} \; \ell(\mb \theta) \doteq \norm{ \mc D_t \paren{G_{\mb \theta}\paren{\mb z}} - \mb y}_F^2 
$,
% \end{equation}
where $G_{\mb \theta}$ is a trainable DNN parameterized by $\mb \theta$ and $\mb z$ is a frozen random seed. Then we conduct experiments for $2\times$ super-resolution with low-level Gaussian and impulse noise. We test our ES-WMV for DIP and a state-of-the-art zero-shot method based on pre-trained diffusion model---DDNM+~\cite{wang_zero-shot_2022} on the standard super-resolution dataset Set14~\cite{boissonnat_single_2012}, as shown in \cref{tab:sr_2,fig:sr_2,sec:appendix_sr}. We note that \textbf{DDNM+ relies on pre-trained models from large external training datasets, while DIP does not.} We observe that (1) \textbf{Our ES-WMV is again able to detect near-peak performance for most images}: the average PSNR gap is $\leq 1.50$ and the average SSIM gap is $\leq 0.07$; (2) DDNM+ is sensitive to the noise type and level: from \cref{tab:sr_2}, DDNM+ trained assuming Gaussian noise level $\sigma_y=0.12$ outperforms DIP and DIP+ES-WMV when there is Gaussian measurement noise at the level $\sigma_y=0.12$, \textbf{which is unrealistic in practice, as the noise level is often unknown beforehand}. When the noise level is not set correctly, e.g., as $\sigma_y=0$ in the {DDNM+} ($\sigma_y=.00$) row of \cref{tab:sr_2}, the performance of DDNM+ is much worse than that of DIP and DIP+ES-WMV. Also, for super-resolution with impulse noise, DIP is also a clear winner that leads {DDNM+} by a large margin; 
and (3) in \cref{sec:diffusion}, we show that DDNM+ may also suffer from the overfitting issue and our ES-WMV can help DDNM+ to stop around the performance peak as well.

\begin{figure*}[!htbp]
    \centering 
    % % \vspace{-2em}
    \includegraphics[width=0.95\linewidth]{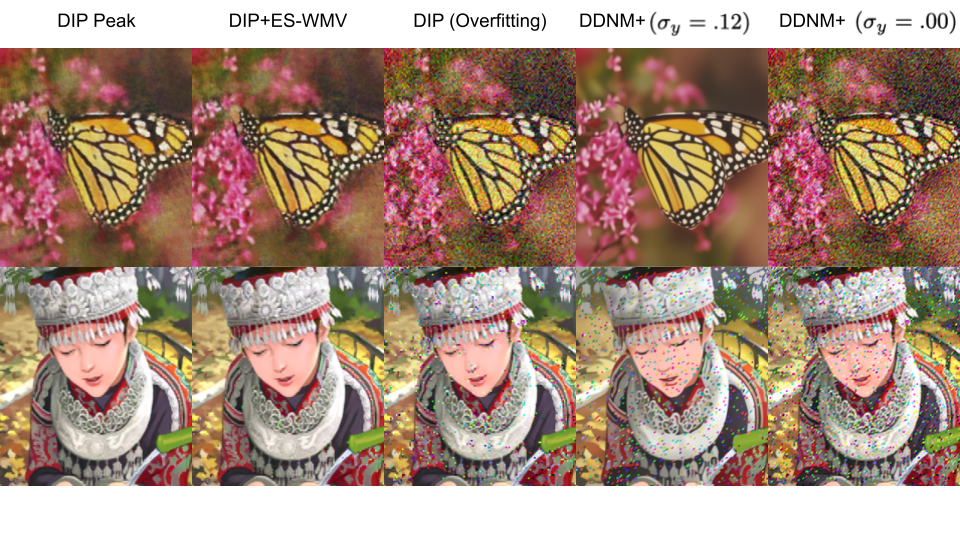}
    \vspace{-2em}
    \caption{Visual comparisons for $2 \times$ image super-resolution task with additional noise. Top row: with additional low-level Gaussian noise; Bottom row: with additional low-level impulse noise.}
    \label{fig:sr_2}
    % \vspace{-1em}
\end{figure*} 

\subsection{MRI reconstruction}
\begin{wraptable}{r}{0.45\linewidth}
    \centering
    \vspace{-1em}
    \caption{ConvDecoder on MRI reconstruction for \textbf{30 cases}: mean and \scriptsize{(std)}. \footnotesize{(\textbf{D}: Detected)}}
    \label{tab:mri_30}
    \setlength{\tabcolsep}{1mm}{
    \begin{tabular}{c c c c}
    %\toprule
    \hline
    
    \multicolumn{1}{c}{\footnotesize{PSNR(\textbf{D})}} &
    \multicolumn{1}{c}{\footnotesize{PSNR Gap}}
    &
    \multicolumn{1}{c}{\footnotesize{SSIM(\textbf{D})}} &
    \multicolumn{1}{c}{\footnotesize{SSIM Gap}}
    \\
    \hline
    \scriptsize{32.63} \tiny({2.36})
    & \textcolor{red}{\scriptsize{0.23}} \tiny({0.32})
    & \scriptsize{0.81} \tiny({0.09})
    & \textcolor{red}{\scriptsize{0.01}} \tiny({0.01})
    \\
    \hline
    \end{tabular}
    }
\end{wraptable}
We further test ES-WMV on MRI reconstruction, a classical linear IP with a nontrivial forward mapping: $\mb y  \approx \mc F \paren{x}$, where $\mc F$ is the subsampled Fourier operator, and we use $\approx$ to indicate that the noise encountered in practical MRI imaging may be hybrid (e.g., additive, shot) and uncertain. Here, we take the $8$-fold undersampling and parameterize $\mb x$ using ``Conv-Decoder''~\citep{darestani2021accelerated}, a variant of {deep decoder}. Due to the heavy over-parameterization, overfitting occurs and ES is needed. \cite{darestani2021accelerated} directly sets the stopping point at the $2500$-th epoch, and we run our ES-WMV. We visualize the performance on two random cases (C1: $1001339$ and C2: $1000190$ sampled from \cite{darestani2021accelerated}, part of the fastMRI datatset~\citep{ZbontarEtAl2018fastMRI}) in \cref{fig:mri_curve} (quality measured in SSIM, consistent with~\cite{darestani2021accelerated}). It is clear that ES-WMV detects near-peak performance for both cases and is adaptive enough to yield comparable or better ES points than heuristically fixed ES points. Furthermore, we test our ES-WMV on ConvDecoder for \textbf{30 cases} from the fastMRI dataset (see \cref{tab:mri_30}), which \textbf{shows the precise and stable detection of ES-WMV}.

\subsection{Blind image deblurring (BID)}
\label{sec:bid_exp}

\begin{wrapfigure}{r}{0.5\textwidth}
    \centering
   \vspace{-2em}
    \includegraphics[width=\linewidth]{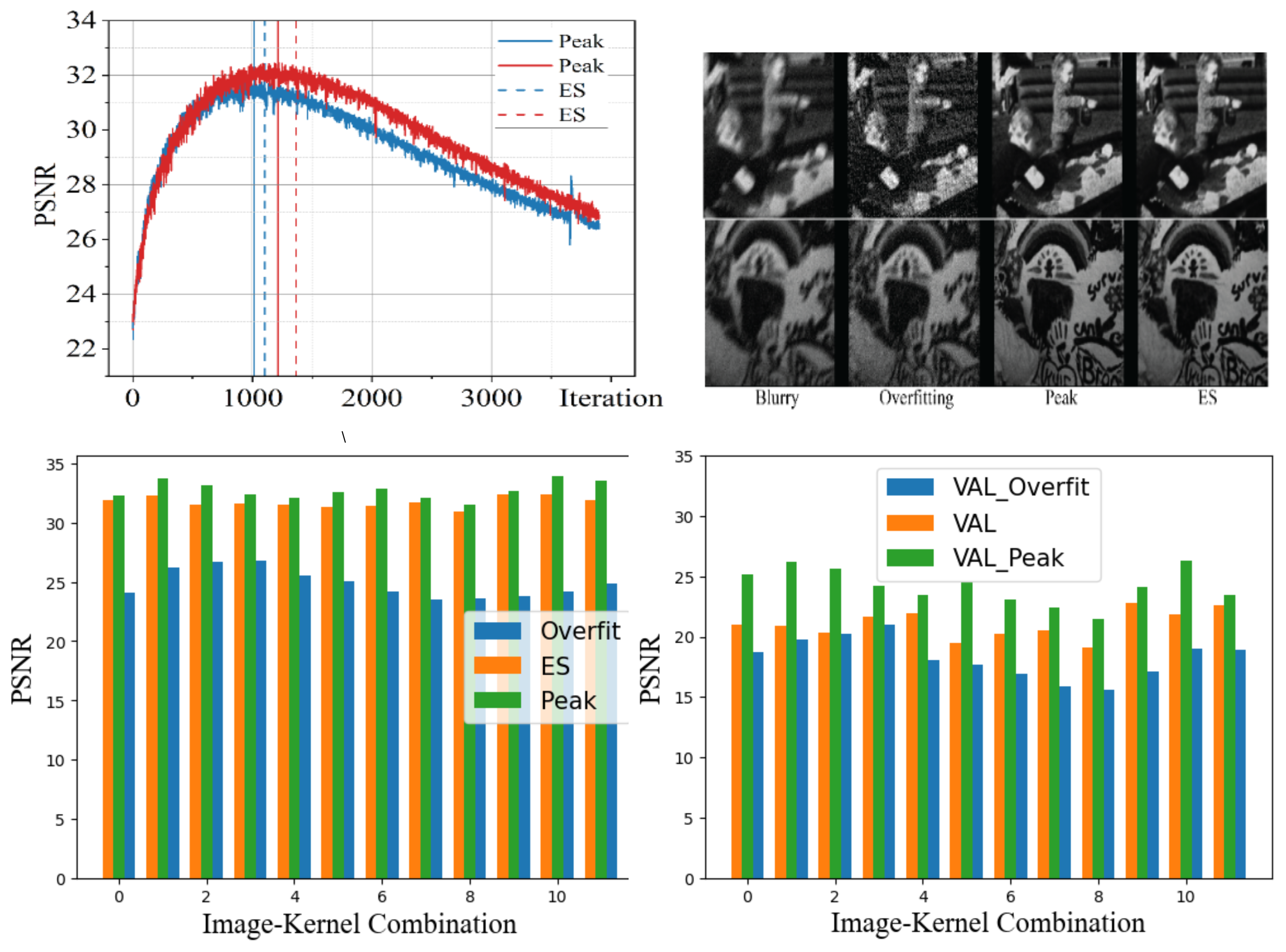}
    \vspace{-2em}
    \caption{Top left: ES-WMV in BID; top right: visual results of ES-WMV; bottom: quantitative results of ES-WMV and VAL, respectively.}
    \vspace{-1em}
    \label{fig:BD}
\end{wrapfigure}

In BID, a blurry and noisy image is given, and the goal is to recover a sharp and clean image. The blur is mostly caused by motion and/or optical non-ideality in the camera, and the forward process is often modeled as $\mb y = \mb k \ast \mb x + \mb n$, 
where $\mb k$ is the blur kernel, $\mb n$ models additive sensory noise, and $\ast$ is linear convolution to model the spatial uniformity of the blur effect~\citep{Szeliski2021Computer}. BID is a very challenging visual IP due to bilinearity: $\paren{\mb k, \mb x} \mapsto \mb k \ast \mb x$. Recently, \cite{Ren_2020_CVPR,wang2019image,asim2019blind,TranEtAl2021Explore} have tried to use DIP models to solve BID by modeling $\mb k$ and $\mb x$ as two separate DNNs, i.e., 
$
    \min_{\mb \theta_k, \mb \theta_x} \norm{\mb y - G_{\mb \theta_k}(\mb z_k) \ast G_{\mb \theta_x} (\mb z_x) }_2^2 + \lambda \norm{\nabla G_{\mb \theta_x} (\mb z_x)}_1/\norm{\nabla G_{\mb \theta_x} (\mb z_x)}_2
$, 
where the regularizer~\cite{li_robust_2023} is to promote sparsity in the gradient domain for the reconstruction of $\mb x$, as standard in BID. We follow~\cite{Ren_2020_CVPR} and choose multilayer perceptron (MLP) with softmax activation for $G_{\mb \theta_k}$, and the canonical DIP model (CNN-based encoder-decoder architecture) for $G_{\mb \theta_x} (\mb z_x)$. We change their regularizer from the original $\norm{\nabla G_{\mb \theta_x} (\mb z_x)}_1$ to the current, as their original formulation is tested only at a very low noise level $\sigma = 10^{-5}$ and no overfitting is observed. We set the test with a higher noise level $\sigma = 10^{-3}$, and find that its original formulation does not work. The benefit of the modified regularizer on BID is discussed in~\cite{krishnan2011blind}. 

\begin{wrapfigure}{r}{0.55\linewidth}
    \centering
    \vspace{-1em}
    \includegraphics[width=\linewidth]{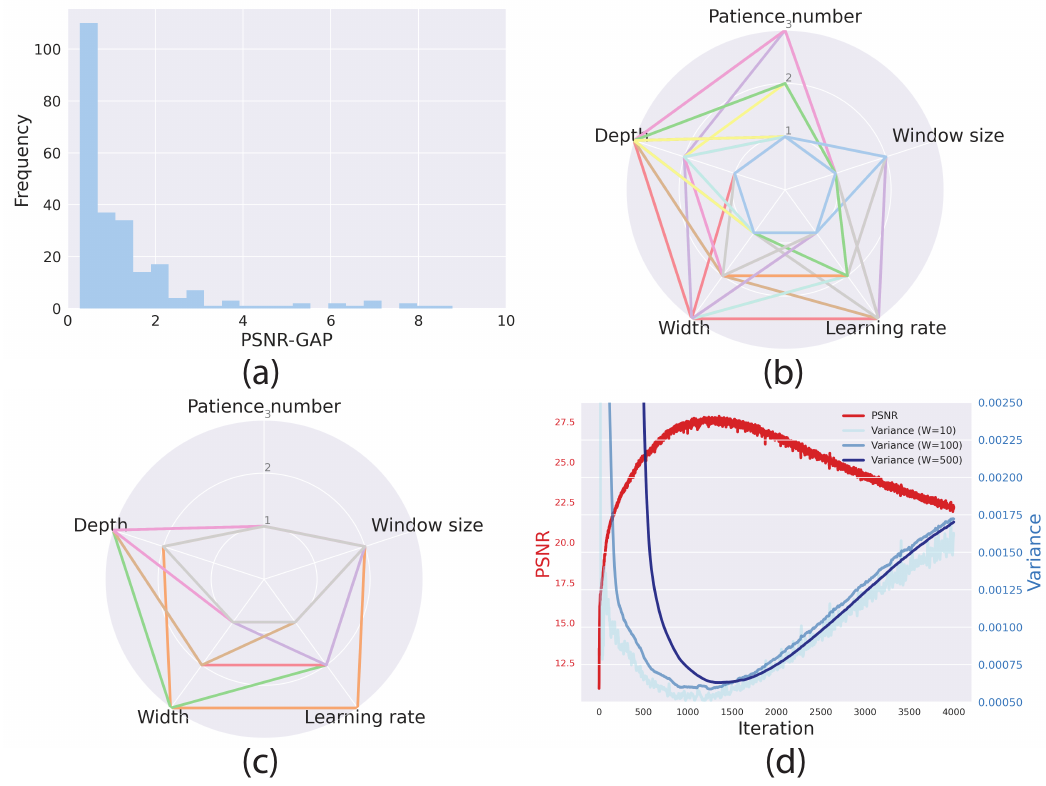}
    % \vspace{-2em}
    \caption{(a) Histogram of PSNR gaps for $243$ different hyperparamter combinations; (b) radar plot of the cases whose PSNR gaps are larger than $2$ dB; (c) radar plot of the cases whose PSNR gaps are larger than $2$ dB and window size = $100$; (d) an example case to visualize VAR curves with different window sizes. (For radar plots, there are three grid values for each hyperparameter. The farther away from the center, the higher the value.)}
    \vspace{-3em} 
    \label{fig:ab_1}
\end{wrapfigure}

First, we take $4$ images and $3$ kernels from the standard Levin dataset~\citep{levin2011understanding}, resulting in $12$ image-kernel combinations. The high noise level leads to substantial overfitting, as shown in \cref{fig:BD} (top left). However, ES-WMV can reliably detect good ES points and lead to impressive visual reconstructions (see \cref{fig:BD} (top right)). We systematically compare VAL and our ES-WMV on this difficult nonlinear IP, as we suspect that nonlinearity can break down VAL as discussed in \cref{sec:introduction}, and subsampling the observation $\mb y$ for training-validation splitting may be unwise. Our results (\cref{fig:BD} (bottom left/right)) confirm these predictions: \textbf{the peak performance detected by VAL is much worse after $10\%$ of the elements in $\mb y$ are removed for validation}. In contrast, our ES-WMV returns quantitatively near-peak performance, much better than leaving the process to overfit. In \cref{tab:level_bid}, we further test both low- and high-level noise on the entire Levin dataset for completeness.

\subsection{{Ablation study}}
The window size $W$ (default: $100$) and the patience number $P$ (default: $1000$) are the only hyperparameters for our ES-WMV. Moreover, in this abalation study, we also include the key DIP hyperparameters, which obviously can also affect our ES performance---in our experiments above, we have used the default published DIP hyperparameters for each IP, as our ES method works under the condition that DIP performs reasonably well for the IP under consideration. To this end, we select the learning rate, which typically determines the learning pace and peak performance of DIP, and the depth/width of the network, which rules the network capacity. 

\begin{figure*}[!b]
    \centering 
    % \vspace{-2em}
    \includegraphics[width=0.9\linewidth]{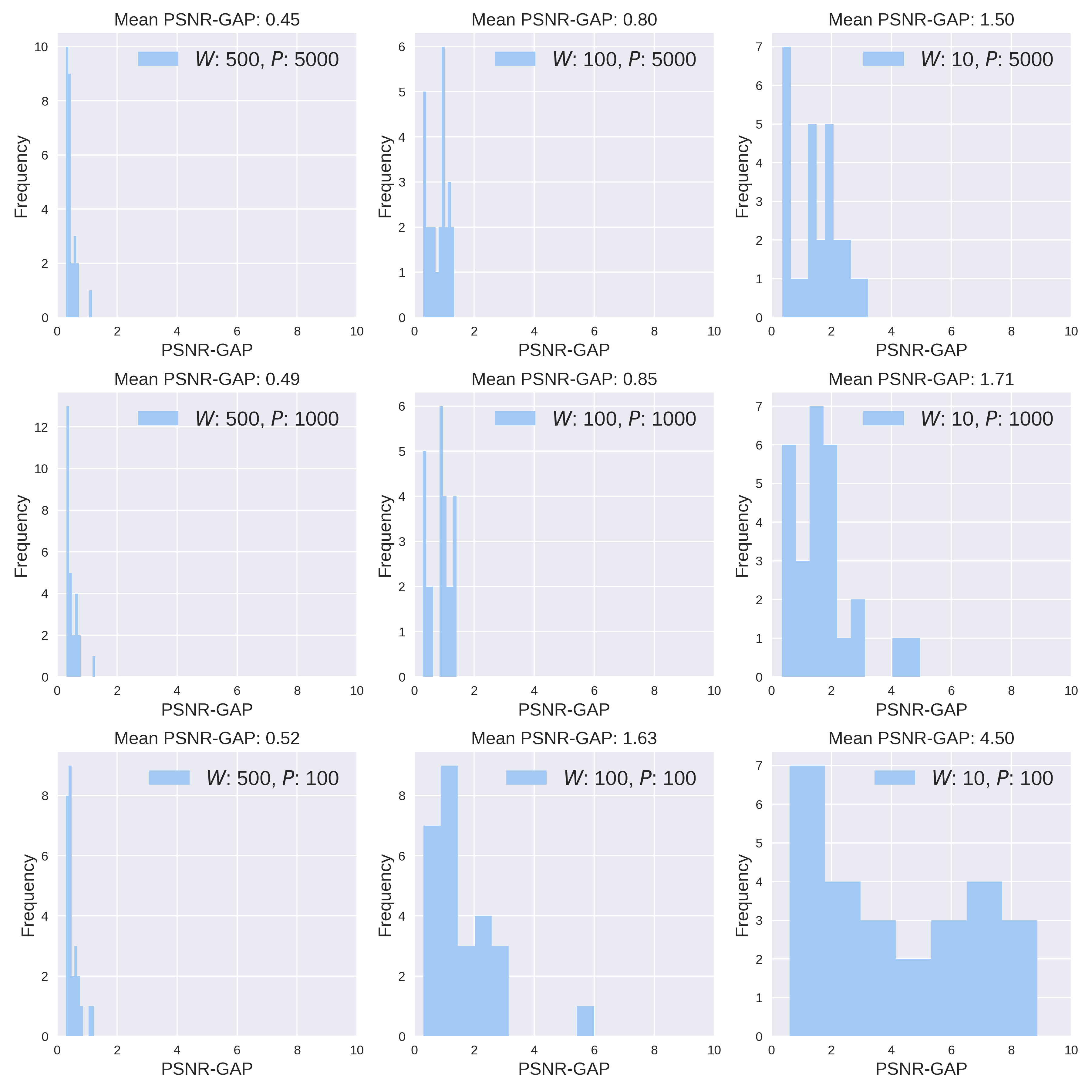}
    % \vspace{-1em}
    \caption{$9$ histograms for $9$ different hyperparameter combinations of our ES method. In each histogram, we show the frequencies of PSNR-GAPs for $27$ various DIP hyperparameter possibilities. $W$: window size; $P$: patience number.}
    \label{fig:ab_2}
    % \vspace{-1em}
\end{figure*} 

Our base task is Gaussian denoising on the classic $9$-image dataset~\citep{Dabov2007} with medium-level noise. We take the same default U-Net backbone model as the experiments in \cref{fig:denoising_example}, and perform experiments on the following hyperparameter grid: window size $\set{500, 100, 10}$, patience number $\set{5000, 1000, 100}$, DIP learning rate $\set{0.01, 0.001, 0.0001}$, DIP model width $\set{256,128,64}$, and DIP model depth $\set{7,5,3}$, resulting in a total of $243$ hyperparameter combinations. For each combination, we calculate the mean PSNR gap, on which our subsequent analysis is based. First of all, we see from \cref{fig:ab_1}(a) that for most ($\ge 83\%$) hyperparameter combinations, the mean PSNR gap falls below $2$ dB. For cases larger than $2$dB, we use the radar plot \cref{fig:ab_1}(b) to explore the deciding factors and find that most of these cases tend to have small ($10$) or medium ($100$) window sizes. This is not surprising, as a small window size can lead to a very fluctuating VAR curve, as shown in \cref{fig:ab_1}(d). To further explore other deciding factors, we focus on the subset with mean PSNR gap $\ge 2$dB and window size $=100$ and plot their settings in \cref{fig:ab_1}(c). We find that these cases invariably have a small patience number ($100$), which can trap our valley detection algorithm into a local fluctuation. So, overall, it seems that our window size and patience number are deciding factors for failures, relative to DIP hyperparameters such as learning rate and network capacity. 

Hence, we next look closely at the combined effect of window size (W) and patience number (P) on ES performance. For this, we plot $9$ histograms for the $9$ different (W, P) combinations in \cref{fig:ab_2} (i.e., each histogram is over the $27$ DIP hyperparameters). The trend is clear: the larger the patience number, the smaller the PSNR gaps; the larger the window size, the smaller the detected PSNR gaps. The average PSNR gap of our default hyperparameter combinations is below $1$ dB (the center one). If we further increase the patience number and the window size, the average PSNR gap can even be lower than $0.5$ dB (top left corner).
Overall, this confirms again our above observation that window size and patience number are the deciding factors for the detection performance of our ES method. Also, this suggests that our ES method can operate well with small PSNR gaps over a wide range of combinations ($W$, $P$), unless both are very small. 

% \begin{wrapfigure}{r}{0.5\textwidth}
%     \centering
%     % \vspace{-2em}
%     \includegraphics[width=1\linewidth]{figures/DD_SGLD_NEW.png}
%     % \vspace{-1em}
%     \caption{Performance of ES-WMV on {deep decoder} and GP-DIP with smaller learning rates for Gaussian denoising in terms of PSNR gaps. L: low noise level; H: high noise level.}
%     % \vspace{-4em}
%     \label{fig:dd_sgld_new}
% \end{wrapfigure}

% We also notice that smaller learning rates can smooth out the VAR curves and mitigate the multi-valley phenomenon in \cref{fig:dd_curve}. Therefore, we apply our ES-WMV to {deep decoder} and GP-DIP with smaller learning rates (both $0.001$), as shown in \cref{fig:dd_sgld_new}. Compared to the results of {deep decoder} and GP-DIP with the default learning rates in \cref{fig:dd_gp_tv}, most of the PSNR gaps decrease.

%% file: sections/Sec5_Discussion.tex
\section{Discussion}\label{sec:discussion}
We have proposed a simple yet effective ES detection method (ES-WMV, and the ES-EMV variant) that works robustly on multiple visual IPs and DIP variants. In comparison, most competing ES methods are noise- or DIP-model-specific and only work for limited scenarios; \cite{Li2021} has comparable performance but slows down the running speed too much; validation-based ES~\citep{ding_validation_2022} works well for the simple denoising task while significantly lags behind our ES method on nonlinear IPs, e.g. BID. 

% \begin{figure*}[!htbp]
%     \centering 
%     % \vspace{-1em}
%     \includegraphics[width=0.8\linewidth]{figures/ablation1.png}
%     % \vspace{-1em}
%     \caption{$9$ histograms for $9$ different hyperparameter combinations of our ES method. In each histogram, we show the frequencies of PSNR-GAPs for $27$ various DIP hyperparameter possibilities. $W$: window size; $P$: patience number.}
%     \label{fig:ab_2}
%     % \vspace{-0.5em}
% \end{figure*} 

\textbf{As for limitations}, our theoretical justification is only partial, sharing the same difficulty of analyzing DNNs in general; our ES method struggles with images with substantial high-frequency components; our detection is sometimes off the peak in terms of iteration numbers when helping certain DIP variants, e.g. DIP-TV with low-level Gaussian noise (\cref{fig:helper_example}), but the detected PSNR gap is still small. DIP variants typically do not improve peak performance and also do not necessarily avoid overfitting, especially for high-level noise. We recommend the original DIP with our ES method for visual IPs discussed in this paper for the best performance and overall speed. Besides ES, there are other major technical barriers to making DIP models practical and competitive for visual IPs. A major one is efficiency: one needs to train a DNN using iterative methods for every instance; our recent works~\cite{li_deep_2023,li_random_2023} have made progress on this issue. 

%% file: sections/Sec7_Appendix.tex
\section{Appendix}\label{sec:appendix}

\subsection{Acronyms}
\label{sec:acronyms}
\begin{table}[H]
  \centering
  \begin{tabular}{c  c}
    \multicolumn{2}{c}{\textbf{List of Common Acronyms} (in alphabetic order)} \\
    \hline
       % CI  & computational imaging \\ 
      CNN  & convolutional neural network \\
      % DD & deep decoder \\ 
      DIP  & deep image prior \\
      DIP-TV & DIP with total variation regularization \\ 
      % DL & deep learning \\
      DNN  & deep neural network \\
      ELTO & early-learning-then-overfitting \\
      ES & early stopping \\
      EMA & exponential moving average \\
      EMV  & exponential moving variance \\
      FR-IQM  &  full-reference image quality metric \\ 
      GP-DIP & Gaussian process DIP \\
      INR & implicit neural representations \\
      IP  & inverse problem \\ 
      MSE  & mean squared error    \\
      NR-IQM  &  no-reference image quality metric \\
      PSNR  &  peak signal-to-noise ratio \\
      SIREN & sinusoidal representation networks \\ 
      % SOTA  & state-of-the-art  \\
      VAR  &  variance \\ 
      WMV  & windowed moving variance \\
    \hline
  \end{tabular}
  % \caption{List of Common Acronyms (in alphabetic order)}
\end{table}

\subsection{Proof of \ref{prop:first_stage}}
\label{sec:proof_exact_trend} 
\begin{proof}
To simplify the notation, we write $\wh{\mb y} \doteq \mb y - G_{\mb \theta^0} \paren{\mb z}$, $\mb J \doteq \mb J_G\paren{\mb \theta^0}$, and $\mb c \doteq \mb \theta - \mb \theta^0$. So, the least-squares objective in \cref{eq:linearized_obj} is equivalent to 
\begin{align} 
\norm{\wh{\mb y} - \mb J \mb c}_{2}^2 
\end{align} 
and the gradient update reads 
\begin{align}
    \mb c^t = \mb c^{t-1} - \eta \mb J^\T \paren{\mb J \mb c^{k-1} - \wh{\mb y}},  
\end{align}
where $\mb c^0 = \mb 0$ and $\mb x^t = \mb J \mb c^t + G_{\mb \theta^0} \paren{\mb z}$. The residual at time $t$ can be computed as
\begin{align}
\mathbf{r}^{t} 
& \doteq\wh{\mb y}-\mathbf{J} \mb c^t \\
&=\wh{\mb y}-\mathbf{J} \left(\mb c^{t-1}-\eta \mathbf{J}^\T \left(\mathbf{J} \mb \theta^{t-1}-\wh{\mb y}\right)\right) \\
&=\left(\mathbf{I}-\eta \mathbf{J} \mathbf{J}^\T \right)\left(\wh{\mb y}-\mathbf{J} \mb c^{t-1}\right) \\
&= \left(\mathbf{I}-\eta \mathbf{J} \mathbf{J}^\T \right)^2 \left(\wh{\mb y}-\mathbf{J} \mb c^{t-2}\right) = \dots \\ 
&=\left(\mathbf{I}-\eta \mathbf{J} \mathbf{J}^\T \right)^{t}\left(\wh{\mb y}-\mathbf{J} \mb c^{0}\right) \quad (\text{using $\mb c^0  = \mb 0$}) \\
&=\left(\mathbf{I}-\eta \mathbf{J} \mathbf{J}^\T \right)^{t} \wh{\mb y}. 
\end{align}
Assume that the SVD of $\mb J$ is as $\mb J = \mb W \mb \Sigma \mb V^\T$. Then 
\begin{align} 
    \mb r^t 
    = \left(\mathbf{I}-\eta \mathbf{W} \mb \Sigma^{2} \mathbf{W}^\T \right)^{t} \wh{\mb y} 
    = \sum_{i}\left(1-\eta \sigma_{i}^{2}\right)^{t} \mathbf{w}_{i}^\T \wh{\mb y}  \mathbf{w}_{i}
\end{align} 
and so 
\begin{align}
    \mathbf{J} \mb c^t 
    = \wh{\mb y} - \mb r^t 
    = \sum_{i}\left(1-\left(1-\eta \sigma_{i}^{2}\right)^{t}\right) \mathbf{w}_{i}^\T \wh{\mb y}  \mathbf{w}_{i}. 
\end{align}
Consider a set of $W$ vectors $\mc V = \set{\mb v_1, \dots, \mb v_W}$. We have the empirical variance. 

\begin{align}
    \mathrm{VAR}\paren{\mc V} 
    = \frac{1}{W} \sum_{w=1}^W \norm{\mb v_w - \frac{1}{W} \sum_{j=1}^W \mb v_j}_2^2
    = \frac{1}{W} \sum_{w=1}^W \norm{\mb v_w}_2^2 - \norm{\frac{1}{W} \sum_{w=1}^W \mb v_w}_2^2.
\end{align}

Therefore, the variance of the set $\set{\mb x^t, \mb x^{t+1}, \dots, \mb x^{t+W -1}}$, same as the variance of the set $\set{\mb J\mb c^t, \mb J \mb c^{t+1}, \dots, \mb J \mb c^{t+W -1}}$, can be calculated as 
\begin{align}
    & \frac{1}{W}\sum_{w=0}^{W-1}\sum_{i}\paren{\mb w_i^\T \wh{\mb y}}^2 \left(1-\left(1-\eta \sigma_{i}^{2}\right)^{t+w}\right)^2 - \frac{1}{W^2}\sum_i \paren{\mb w_i^\T \wh{\mb y}}^2 \paren{\sum_{w=0}^{W-1}1-\left(1-\eta \sigma_{i}^{2}\right)^{t+w}}^2\\
    =\; & \frac{1}{W^2}\sum_{i}\paren{\mb w_i^\T \wh{\mb y}}^2 \left[W \sum_{w=0}^{W-1} \left(1-\left(1-\eta \sigma_{i}^{2}\right)^{t+w}\right)^2-\left(\sum_{w=0}^{W-1}1-\left(1-\eta \sigma_{i}^{2}\right)^{t+w}\right)^2\right]\\
    =\; & \frac{1}{W^2}\sum_{i}\paren{\mb w_i^\T \wh{\mb y}}^2 \left[\left(W^2+W\frac{(1-\eta\sigma_i^2)^{2t}(1-(1-\eta\sigma_i^2)^{2W})}{1-(1-\eta\sigma_i^2)^2}-2W\frac{(1-\eta\sigma_i^2)^t(1-(1-\eta\sigma_i^2)^W)}{\eta\sigma_i^2} \right)\right. \nonumber\\
    & \quad \quad \left.-\left(W^2- 2W \frac{(1-\eta\sigma_i^2)^t(1-(1-\eta\sigma_i^2)^W)}{\eta\sigma_i^2} + \frac{\paren{1-\eta \sigma_i^2}^{2t} \paren{1- \paren{1-\eta\sigma_i^2}^W}^2  }{\eta^2 \sigma_i^4} \right)\right]\\
    =\; & \frac{1}{W^2}\sum_{i}\left\langle\mathbf{w}_{i}, \wh{\mb y}\right\rangle^2\frac{(1-\eta\sigma_i^2)^{2t}}{\eta\sigma_i^2}\left[W\frac{1-(1-\eta\sigma_i^2)^{2W}}{2-\eta\sigma_i^2}-\frac{(1-(1-\eta\sigma_i^2)^W)^2}{\eta\sigma_i^2}\right]. 
\end{align}
So the constants $C_{W, \eta, \sigma_i}$'s are defined as 
\begin{align} 
    C_{W, \eta, \sigma_i} \doteq \frac{1}{W^2 \eta\sigma_i^2} \left[W\frac{1-(1-\eta\sigma_i^2)^{2W}}{2-\eta\sigma_i^2}-\frac{(1-(1-\eta\sigma_i^2)^W)^2}{\eta\sigma_i^2}\right]. 
\end{align} 
To see they are nonnegative, it is sufficient to show that 
\begin{multline} 
  W\frac{1-(1-\eta\sigma_i^2)^{2W}}{2-\eta\sigma_i^2}-\frac{(1-(1-\eta\sigma_i^2)^W)^2}{\eta\sigma_i^2} \ge 0 \\ \Longleftrightarrow \eta\sigma_i^2 W\paren{1-(1-\eta\sigma_i^2)^{2W}} - \paren{2-\eta\sigma_i^2} (1-(1-\eta\sigma_i^2)^W)^2 \ge 0. 
\end{multline} 
Now consider the function. 
\begin{align} 
    h\paren{\xi, W} = \xi W\paren{1-(1-\xi)^{2W}} - \paren{2-\xi} (1-(1-\xi)^W)^2    \quad  \xi \in [0, 1], W \ge 1. 
\end{align} 
First, one can easily check that $\partial_{W} h\paren{\xi, W} \ge 0$ for all $W \ge 1$ and all $\xi \in [0, 1]$, that is, $h(\xi, W)$ increases monotonically with respect to $W$. Thus, to prove $C_{W, \eta, \sigma_i} \ge 0$, it suffices to show that $h(\xi, 1) \ge 0$. Now 
\begin{align} 
    h(\xi, 1) = \xi \paren{1-(1-\xi)^2} - (2-\xi) \xi^2 = 0, 
\end{align} 
completing the proof. 
\end{proof}

\subsection{Proof of \ref{thm:upper_bound}}
\label{sec:proof_upper} 

We first re-state Theorem 2 in~\cite{heckel2020denoising}. 
\begin{theorem}[\cite{heckel2020denoising}]
    \label{thm:denoising_main_theorem}
Let $\mathbf{x} \in \mathbb{R}^{n}$ be a signal in the span of the first $p$ trigonometric basis functions, and consider a noisy observation $\mathbf{y}=\mathbf{x}+\mb n$, where the noise $\mb n \sim \mc N\left(\mathbf{0}, \xi^{2}/n \cdot \mathbf{I}\right)$. To denoise this signal, we fit a two-layer generator network $G_{\mb C} \paren{\mb B}=\operatorname{ReLU}(\mb U \mb B \mb  C) \mathbf{v}$, where $\mb v = [1, \dots, 1, -1, \dots, -1]/\sqrt{k}$, and $\mb B \sim_{iid} \mc N\paren{0, 1}$, and $\mb U$ is an upsampling operator that implements circular convolution with a given kernel $\mb u$. Denote $\mb \sigma \doteq \norm{\mb u}_2 \lvert \mb F g(\mb u \circledast \mb u/\norm{\mb u}_2^2) \rvert ^{1/2}$ where $g(t) = (1-\cos^{-1} (t)/\pi) t$ and $\circledast$ denote the circular convolution. Fix any $\epss \in (0, \sigma_p/\sigma_1]$, and suppose that $k \ge C_{\mb u} n/\epss^8$, where $C_{\mb u} > 0$ is a constant depending only on $\mb u$. Consider gradient descent with step size $\eta \leq \|\mathbf{F u}\|_{\infty}^{-2}$($\mathbf{F u}$ is the Fourier transform of $\mathbf{u}$ ) starting from $\mathbf{C}_{0}\sim_{iid}\mathcal{N}\left(0, \omega^{2}\right)$, entries $\omega \propto \frac{\|\mathbf{y}\|_{2}}{\sqrt{n}}$.
Then, for all iterations $t$ obeying $t \leq \frac{100}{\eta \sigma_{p}^{2}}$, the reconstruction error obeys
\begin{align*}
\left\|G_{\mb C^t} \paren{\mb B}- \mb x \right\|_{2} \leq\left(1-\eta \sigma_{p}^{2}\right)^t\| \mb x \|_{2}+ \sqrt{\sum_{i=1}^n \paren{(1-\eta \sigma_i^2)^t - 1}^2 \paren{\mb w_i^\T \mb n}^2}  +\epss \| \mb y \|_{2}
\end{align*}
with probability at least $1-\exp\paren{-k^2} - n^{-2}$. 
\end{theorem}
Note that since $\mb B \sim_{iid} \mc N\paren{0, 1}$ and hence is full-rank with probability one, the original Theorem 1 \& 2 of \cite{heckel2020denoising} rename $\mb B \mb C$ to $\mb C'$ and state the result directly on $\mb C'$, that is, assume that the model is $\operatorname{ReLU}(\mb U \mb C') \mathbf{v}$. It is easy to see that the original theorems imply the version stated here. 

With this, we can obtain our \cref{thm:upper_bound}, stated in full technical form here: 
\begin{theorem}
Let $\mathbf{x} \in \mathbb{R}^{n}$ be a signal in the span of the first $p$ trigonometric basis functions, and consider a noisy observation $\mathbf{y}=\mathbf{x}+\mb n$, where the noise $\mb n \sim \mc N\left(\mathbf{0}, \xi^{2}/n \cdot \mathbf{I}\right)$. To denoise this signal, we fit a two-layer generator network $G_{\mb C} \paren{\mb B}=\operatorname{ReLU}(\mb U \mb B \mb  C) \mathbf{v}$, where $\mb v = [1, \dots, 1, -1, \dots, -1]/\sqrt{k}$, and $\mb B \sim_{iid} \mc N\paren{0, 1}$, and $\mb U$ is an upsampling operator that implements circular convolution with a given kernel $\mb u$. Denote $\mb \sigma \doteq \norm{\mb u}_2 \lvert \mb F g(\mb u \circledast \mb u/\norm{\mb u}_2^2) \rvert ^{1/2}$ where $g(t) = (1-\cos^{-1} (t)/\pi) t$ and $\circledast$ denotes the circular convolution. Fix any $\epss \in (0, \sigma_p/\sigma_1]$, and suppose $k \ge C_{\mb u} n/\epss^8$, where $C_{\mb u} > 0$ is a constant only depending on $\mb u$. Consider gradient descent with step size $\eta \leq \|\mathbf{F u}\|_{\infty}^{-2}$($\mathbf{F u}$ is the Fourier transform of $\mathbf{u}$ ) starting from $\mathbf{C}_{0}\sim_{iid}\mathcal{N}\left(0, \omega^{2}\right)$, entries $\omega \propto \frac{\|\mathbf{y}\|_{2}}{\sqrt{n}}$.
Then, for all iterates $t$ obeying $t \leq \frac{100}{\eta \sigma_{p}^{2}}$, our WMV obeys
\begin{align}
\mathrm{WMV} \le \frac{12}{W}\norm{\mb x}_2^2 \frac{ \left(1-\eta \sigma_{p}^{2}\right)^{2t}}{1-(1-\eta \sigma_p^2)^2} + 12 \sum_{i=1}^n \paren{\paren{1 - \eta \sigma_i^2}^{t + W-1} - 1}^2 \paren{\mb w_i^\T \mb n}^2 + 12\epss^2 \norm{\mb y}_2^2
\end{align}
with probability at least $1-\exp\paren{-k^2} - n^{-2}$. 
\end{theorem}
\begin{proof} 
We make use of the basic inequality: $\norm{\mb a - \mb b}_2^2 \le 2\norm{\mb a}_2^2 + 2 \norm{\mb b}_2^2$ for any two vectors $\mb a, \mb b$ of compatible dimension. We have
\begin{align}
    & \frac{1}{W}\sum_{w=0}^{W-1} \|G_{\mb C^{t+w}} \paren{\mb B} - \frac{1}{W}\sum_{j=0}^{W-1} G_{\mb C^{t+j}} \paren{\mb B}\|_2^2 \\
    =\; &  \frac{1}{W}\sum_{w=0}^{W-1} \|G_{\mb C^{t+w}} \paren{\mb B} - \mb x + \mb x -  \frac{1}{W}\sum_{j=0}^{W-1} G_{\mb C^{t+j}} \paren{\mb B}\|_2^2\\
    \le\; &  \paren{\frac{2}{W}\sum_{w=0}^{W-1} \|G_{\mb C^{t+w}} \paren{\mb B} - \mb x\|_2^2}  + 2 \| \mb x -  \frac{1}{W}\sum_{j=0}^{W-1} G_{\mb C^{t+j}} \paren{\mb B} \|_2^2 \\
    \le \; & \frac{2}{W}\sum_{w=0}^{W-1} \|G_{\mb C^{t+w}} \paren{\mb B} - \mb x\|_2^2  + \frac{2}{W} \sum_{j=0}^{W-1} \| G_{\mb C^{t+j}} \paren{\mb B}-\mb x \|_2^2  \\
    & \quad (\text{$\mb z \mapsto \|\mb z - \mb x\|_2^2$ convex and Jensen's inequality}) \nonumber \\
    =\; & \frac{4}{W}\sum_{w=0}^{W-1} \|G_{\mb C^{t+w}} \paren{\mb B} - \mb x\|_2^2. 
\end{align}
In view of \cref{thm:denoising_main_theorem}, 
\begin{align} 
    \left\|G_{\mb C^{t+w}} \paren{\mb B}-\mathbf{x}\right\|_{2}^2 \leq 3\left(1-\eta \sigma_{p}^{2}\right)^{2t+2w}\|\mathbf{x}\|_{2}^2 +3\sum_{i=1}^{n}\left(\left(1-\eta \sigma_{j}^{2}\right)^{t+w}-1\right)^{2} \paren{\mb w_i^\T \mb n}^2 + 3 \epss^2 \norm{\mb y}_2^2. 
\end{align} 
Thus, 
\begin{align} 
    & \sum_{w=0}^{W-1} \|G_{\mb C^{t+w}} \paren{\mb B} - \mb x\|_2^2 \nonumber \\ 
    \le\; & 3\norm{\mb x}_2^2 \sum_{w =0}^{W-1} \left(1-\eta \sigma_{p}^{2}\right)^{2t+2w} + 3 \sum_{w =0}^{W-1} \sum_{i=1}^{n}\left(\left(1-\eta \sigma_{i}^{2}\right)^{t+w}-1\right)^{2} \paren{\mb w_i^\T \mb n}^2 + 3W\epss^2 \norm{\mb y}_2^2 \\
    \le\; & 3\norm{\mb x}_2^2 \frac{ \left(1-\eta \sigma_{p}^{2}\right)^{2t}(1-(1-\eta \sigma_p^2)^{2W})}{1-(1-\eta \sigma_p^2)^2} + 3W \sum_{i=1}^n \paren{\paren{1 - \eta \sigma_i^2}^{t + W-1} - 1}^2 \paren{\mb w_i^\T \mb n}^2 + 3W\epss^2 \norm{\mb y}_2^2 \\
    \le\; & 3\norm{\mb x}_2^2 \frac{ \left(1-\eta \sigma_{p}^{2}\right)^{2t}}{1-(1-\eta \sigma_p^2)^2} + 3W \sum_{i=1}^n \paren{\paren{1 - \eta \sigma_i^2}^{t + W-1} - 1}^2 \paren{\mb w_i^\T \mb n}^2 + 3W\epss^2 \norm{\mb y}_2^2, 
\end{align} 
completing the proof. 
\end{proof} 

\subsection{ES-EMV algorithm} \label{sec:alg2}
The exponential moving variance version of our method is summarized in \cref{alg:framework_emavg}. 
\begin{algorithm}[!htbp]
    % \small 
\caption{DIP with ES--EMV}
\label{alg:framework_emavg} 
\begin{algorithmic}[1]
\Require random seed $\mb z$, randomly-initialized $G_{\mb \theta}$, forgetting factor $\alpha \in (0, 1)$, patience number $P$, iteration counter $k = 0$, $\mathrm{EMA}^0 = 0$, $\mathrm{EMV}^0 = 0$, $\mathrm{EMV}_{\min} = \infty$
\Ensure reconstruction $\mb x^{*}$
\While{not stopped}
\State update $\mb \theta$ via \cref{eq:dip} to obtain $\mb \theta^{k+1}$ and $\mb x^{k+1}$
\State $\mathrm{EMA}^{k+1} = \paren{1-\alpha} \mathrm{EMA}^{k} + \alpha \mb x^{k+1}$
\State $\mathrm{EMV}^{k+1} = \paren{1-\alpha} \mathrm{EMV}^{k} + \alpha(1-\alpha)\|\mb x^{k+1} -\mathrm{EMA}^{k}\|_2^2$\hspace{-0.5em}
\If{$\mathrm{EMV^{k+1}} < \mathrm{EMV}_{\min}$}
\State $\mathrm{EMV}_{\min} \leftarrow \mathrm{EMV^{k+1}}$, $\mb x^{*} \leftarrow \mb x^{k+1}$
\EndIf
\If{$\mathrm{EMV}_{\min}$ stagnates for $P$ iterations}
\State stop and return $\mb x^\ast$ 
\EndIf
\State $k = k+1$
\EndWhile
\end{algorithmic}
\end{algorithm}

\subsection{More details on major DIP variants}
\label{sec:details_variants}

\paragraph{\textcolor{umn_maroon}{Deep Decoder}}~\citep{heckel2018deep} differs from DIP mainly in terms of network architecture: It is typically a \emph{under-parameterized} network consisting mainly of $1 \times 1$ convolutions, upsampling, ReLU and channel-wise normalization layers, while DIP uses an \emph{over-parameterized}, U-net like convolutional network.

\paragraph{\textcolor{umn_maroon}{GP-DIP}}~\citep{ChengEtAl2019Bayesian} uses the original DIP~\citep{ulyanov2018deep} network and formulation, but replaces stochastic gradient descent (SGD) by stochastic gradient Langevin dynamics (SGLD) in the gradient update step. i.e., for the generic gradient step for optimizing \cref{eq:dip} reads: 
\begin{align}
\mb \theta^+ = \mb \theta - t \nabla_{\mb  \theta} \brac{\ell \paren{\mb y, f\paren{G_{\mb \theta}\paren{\mb z}}} + \lambda R\paren{G_{\mb \theta}\paren{\mb z}}} + \eta
\end{align}
where $\eta$ is zero-mean Gaussian with an isotropic variance level $t$. 

\paragraph{\textcolor{umn_maroon}{DIP-TV}}~\citep{cascarano2021combining} uses the original DIP~\citep{ulyanov2018deep} network, with a Total Variation (TV) regularizer
added. Then, the proposed objective is solved with the Alternating Direction Method of Multipliers (ADMM) framework.

\paragraph{\textcolor{umn_maroon}{SIREN}}~\citep{sitzmann2020implicit} treats the object directly as a continuous function on $\RR^2$ or $\RR^3$ (or higher-dimensional spaces depending on the application) and hence parameterizes it as a multi-layer perceptron (MLP):  1) the input to SIREN is the 2D/3D coordinate of each pixel instead of random values, and 2) the network uses a sinusoidal activation function instead of the commonly used ReLU. When substituting the DIP network with SIREN and solve \cref{eq:dip} problems, similar overfitting issue is still observed. 

\subsection{More details on major ES methods} \label{sec:details_ES_methods}
Here, we provide more details on the main competing methods.

\paragraph{\textcolor{umn_maroon}{Spectral Bias (SB)}}~\citet{shi2021measuring} operates on {deep decoder} models and proposes two modifications to change the spectral bias: (1) controlling the operator norm of the weight $\mb w$ for each convolutional layer by normalization 
\begin{align} 
    \mb w' = \frac{\mb w}{\max\paren{1, \norm{\mb w}_{\mathrm{op}}/\lambda}}, 
\end{align} 
ensuring that $\norm{\mb w'}_{\mathrm{op}} \le \lambda$, which in turn controls the Fourier spectrum of the underlying function represented by the layer; (2) performing Gaussian upsampling instead of the typical bilinear upsampling to suppress the smoothness effect of the latter. These two modifications with appropriate parameter setting ($\lambda$, and $\sigma$ in Gaussian filtering) can improve the learning of the high-frequency components by {deep decoder}, and allow the blurriness-over-sharpness stopping criterion.  
\begin{align} \label{eq:sb_metric}
    \Delta r\paren{\mb x^t} = \frac{1}{W} \abs{\sum_{w=1}^W r\paren{\mb x^{t- w}}  - \sum_{w=1}^W r\paren{\mb x^{t- W- w}}}, 
\end{align} 
where $r\paren{\mb x'} = B\paren{\mb x'}/S\paren{\mb x'}$, and $B\paren{\cdot}$ and $S\paren{\cdot}$ are the blurriness and sharpness metrics in \citet{CreteEtAl2007blur} and \citet{BahramiKot2014Fast}, respectively. In other words, the criterion in \cref{eq:sb_metric} measures the change in the average blurriness-over-sharpness ratios in consecutive windows of size $W$, and small changes indicate good ES points. But, as mentioned, this criterion only works for modified DD models and not for other DIP variants, as acknowledged by the authors in \citet{shi2021measuring} and confirmed in our experiment (see \cref{sec:helper_methods}). 

\paragraph{\textcolor{umn_maroon}{DF-STE}}~\citet{jo2021rethinking} targets Gaussian denoising with known noise levels (i.e. $\mb y = \mb x + \mb n$, where $n$ is the i.i.d. Gaussian noise) and considers the objective. 
\begin{align} 
    \min_{\mb \theta} \frac{1}{n^2} \norm{\mb y - G_{\mb \theta}\paren{\mb y}}_F^2 + \frac{\sigma^2}{n^2} \trace \mb J_{G_{\mb \theta}}\paren{\mb y}, 
\end{align} 
where $\trace \mb J_{G_{\mb \theta}}\paren{\mb y}$ is the trace of the network Jacobian with respect to the input, that is, the divergence term in~\citet{jo2021rethinking}. The divergence term is a proxy for controlling the capacity of the network. The paper then proposes a heuristic zero-crossing stopping criterion that stops the iteration when the loss starts to cross zero into negative values. Although the idea works reasonably well on Gaussian denoising with low and known noise level (the variance level $\sigma^2$ is explicitly needed in the regularization parameter ahead of the divergence term), it starts to break down when the noise level increases even if the right noise level is provided; see \cref{sec:exp_competing_methods}. Also, although the paper has extended the formulation to handle Poisson noise, it is unclear how to generalize the idea for handling other types of noise, as well as how to move beyond simple additive denoising problems. 

\paragraph{\textcolor{umn_maroon}{SV-ES}}~\citet{Li2021} proposes training an autoencoder online using the reconstruction sequence $\set{\mb x^t}_{t \ge 1}$:  
\begin{align} 
    \min_{\mb w, \mb v} \sum_{t \ge 1} \ell_{\mathrm{AE}} \paren{\mb x^t, D_{\mb w} \circ E_{\mb v}\paren{\mb x^t}}. 
\end{align} 
Any new $\mb x^t$ passes through the current autoencoder and the reconstruction error $\ell_{\mathrm{AE}}$ is recorded. They observe that the error curve typically follows a U-shaped shape and that the valley of the curve is approximately aligned with the peak of the PNSR curve. Therefore, they design an ES method by detecting the valley of the error curve. This method works reasonably well for different IPs and different DIP variants. A major drawback is efficiency: the overhead caused by the online training of the autoencoder is on an order of magnitude higher than the cost of the DIP update itself, as shown in \cref{tab:wall-clock time}. 

\paragraph{\textcolor{umn_maroon}{DOP}}~\citet{you2020robust} considers only additive sparse noise (e.g., salt and pepper noise) and proposes modeling the clean image and noise explicitly in the objective: 
\begin{align} 
    \min_{\mb \theta, \mb g, \mb h} \; \norm{\mb y - G_{\mb \theta}\paren{\mb z} - \paren{\mb g \circ \mb g - \mb h \circ \mb h}}_F^2,  
\end{align} 
where the overparameterized term $\mb g \circ \mb g - \mb h \circ \mb h$ ($\circ$ denotes the Hadamard product) is meant to capture sparse noise, where a similar idea has been shown to be effective for sparse recovery in~\citet{VaskeviciusEtAl2019Implicit}. Different properly tuned learning rates for the clean image and sparse noise terms are necessary for success. The downside includes the prolonged running time, as it pushes the peak reconstruction to the very last iteration, and the difficulty to extend the idea to other types of noise. 

\subsection{Additional experimental details \& results}  \label{sec:ap_app_detail_result}

\subsubsection{External codes}   \label{sec:external_codes}
 {\small
\begin{itemize} 
    \item DIP: \url{https://github.com/DmitryUlyanov/deep-image-prior}
    \item {Deep decoder}: \url{https://github.com/reinhardh/supplement_deep_decoder}
    \item DIP-TV: \url{https://github.com/sedaboni/ADMM-DIPTV}
    \item GP-DIP: \url{https://people.cs.umass.edu/~zezhoucheng/gp-dip/}
    \item DF-STE: \url{https://github.com/gistvision/dip-denosing}
    \item SV-ES: \url{https://github.com/sun-umn/Self-Validation}
    \item DOP: \url{https://github.com/ChongYou/robust-image-recovery}
    \item SB: \url{https://github.com/shizenglin/Measure-and-Control-Spectral-Bias}
    \item CBSD68: \url{https://github.com/clausmichele/CBSD68-dataset}
\end{itemize} 
 } 

\subsubsection{Experiment Settings} \label{sec:exp_setting}

Our default setup for all experiments is as follows. Our DIP model is the original from~\cite{ulyanov2018deep}; the optimizer is ADAM with a learning rate $0.01$. For all other models, we use their default architectures, optimizers, and hyperparameters. For ES-WMV, the default window size $W= 100$, and the patience number $P = 1000$. We use both PSNR and SSIM to access the reconstruction quality and report PSNR and SSIM gaps (the difference between our detected and peak numbers) as an indicator of our detection performance. \textbf{For most experiments, we repeat the experiments $3$ times to report the mean and standard deviation}; when not, we explain why. 

\paragraph{Noise generation}  \label{sec:app_detail_result}

Following the noise generation rules of \cite{hendrycks2019robustness}\footnote{\url{https://github.com/hendrycks/robustness}}, we simulate four types of noise and three intensity levels for each type of noise. The detailed information is as follows. \textbf{Gaussian noise:} $0$ mean additive Gaussian noise with variance $0.12$, $0.18$ and $0.26$ for low, medium and high noise levels, respectively; \textbf{Impulse noise:} also known as salt-and-pepper noise, replacing each pixel with probability $p \in [0, 1]$ in a white or black pixel with half chance each. Low, medium and high noise levels correspond to $p = 0.3, 0.5, 0.7$, respectively; \textbf{Speckle noise:} for each pixel $x \in [0, 1]$, the noisy pixel is $x\paren{1+\epss}$, where $\epss$ is zero-mean Gaussian with a variance level $0.20$, $0.35$, $0.45$ for low, medium, and high noise levels, respectively; \textbf{Shot noise:} also known as Poisson noise. For each pixel, $x \in [0, 1]$, the noisy pixel is Poisson distributed with the rate $\lambda x$, where $\lambda$ is $25, 12, 5$ for low, medium, and high noise levels, respectively.

\begin{figure*}[!htbp]
    \centering 
    % \vspace{-1em}
    \includegraphics[width=1\linewidth]{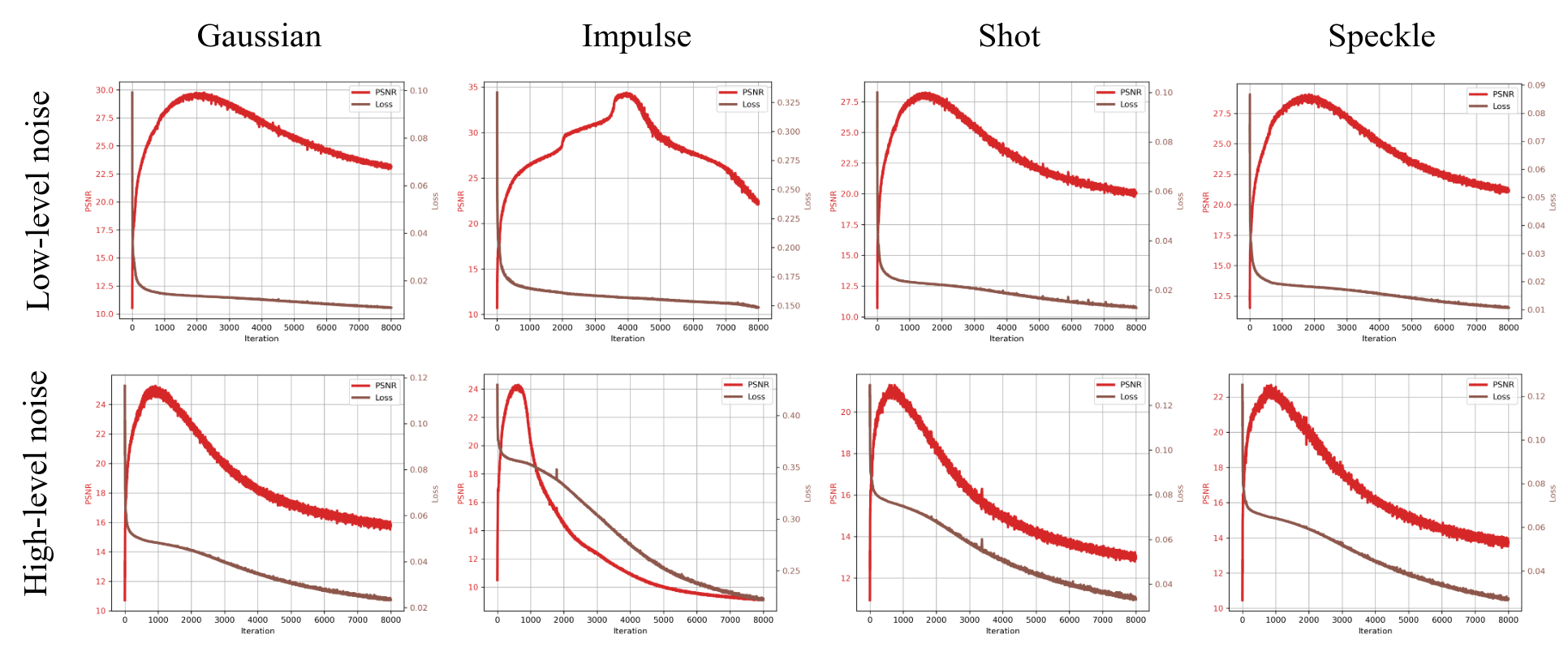}
    % \vspace{-2em}
    \caption{Our ES-WMV method on DIP for denoising ``F16" with different noise types and levels (top: low-level noise; bottom: high-level noise). Red curves are PSNR curves, and brown curves are loss curves. }
    \label{fig:loss_curve}
\end{figure*} 
\subsubsection{Denoising examples}
\label{sec:denoising_eg}
We explore the possibility of using the fitting loss for ES here, but we are unable to find correlations between the trend of the loss and that of the PSNR curve, shown in \cref{fig:loss_curve}

% \begin{figure*}[!htbp]
%     \centering 
%     % \vspace{-1em}
%     \includegraphics[width=0.98\linewidth]{figures/noref_2.png}
%     \caption{Visual comparisons of NR-IQMs and ES-WMV. From top to bottom: shot noise (low), shot noise (high), speckle noise (low), speckle noise (high).}
%     \label{fig:noref_2}
% \end{figure*} 

\begin{figure*}[!htbp]
    \centering 
    % \vspace{-1em}
    \includegraphics[width=0.9\linewidth]{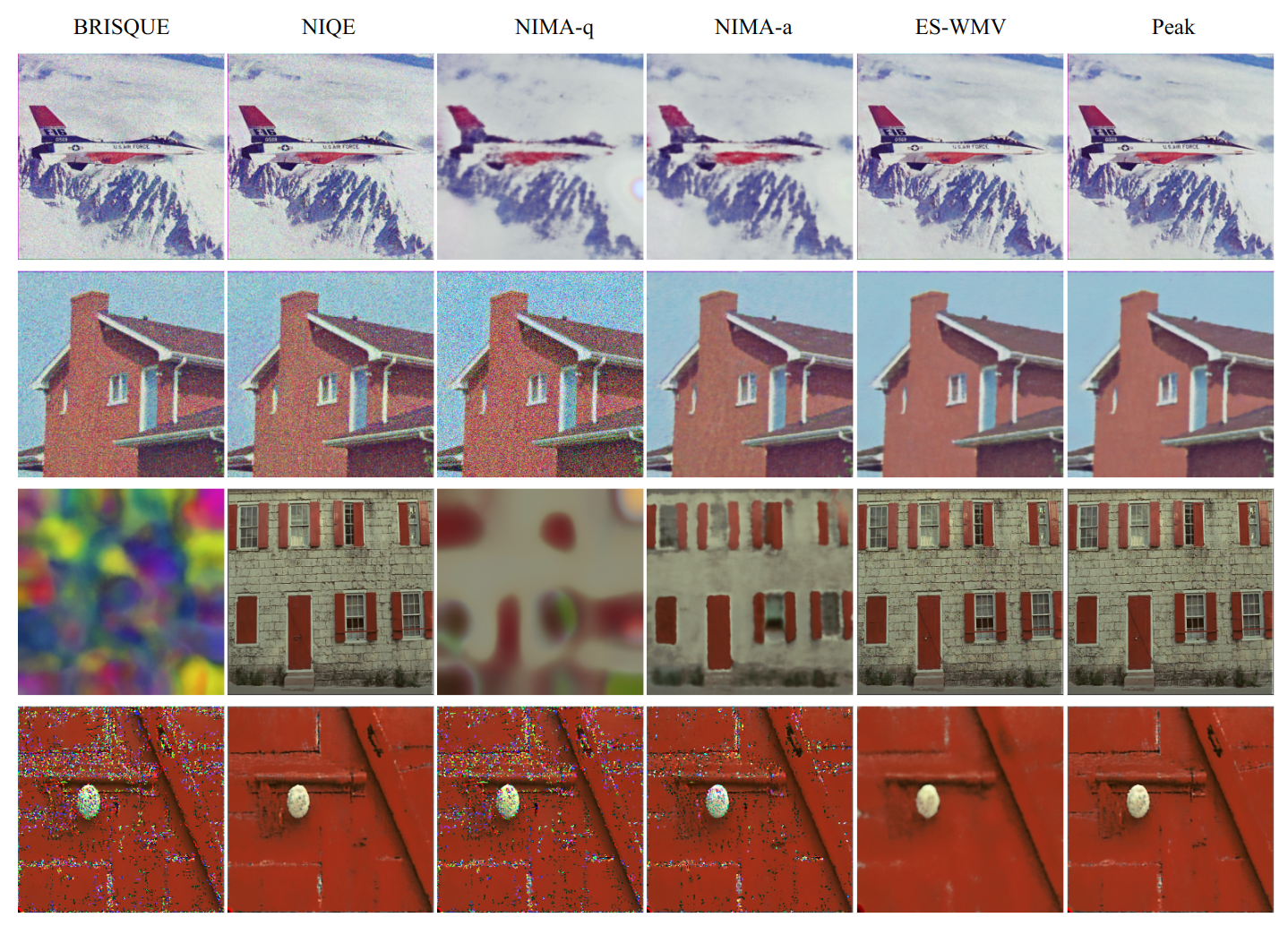}
    \caption{Visual comparisons of NR-IQMs and ES-WMV. From top to bottom: Gaussian noise (low), Gaussian noise (high), impulse noise (low), impulse noise (high).}
    \label{fig:noref_1}
\end{figure*} 
% \begin{figure*}[!htbp]
%     \centering 
%     % \vspace{-1em}
%     \includegraphics[width=0.98\linewidth]{figures/noref_1.png}
%     \caption{Visual comparisons of NR-IQMs and ES-WMV. From top to bottom: Gaussian noise (low), Gaussian noise (high), impulse noise (low), impulse noise (high).}
%     \label{fig:noref_1}
% \end{figure*} 

\begin{figure*}
    \centering
    % \vspace{-1em}
    \includegraphics[width=0.75\linewidth]{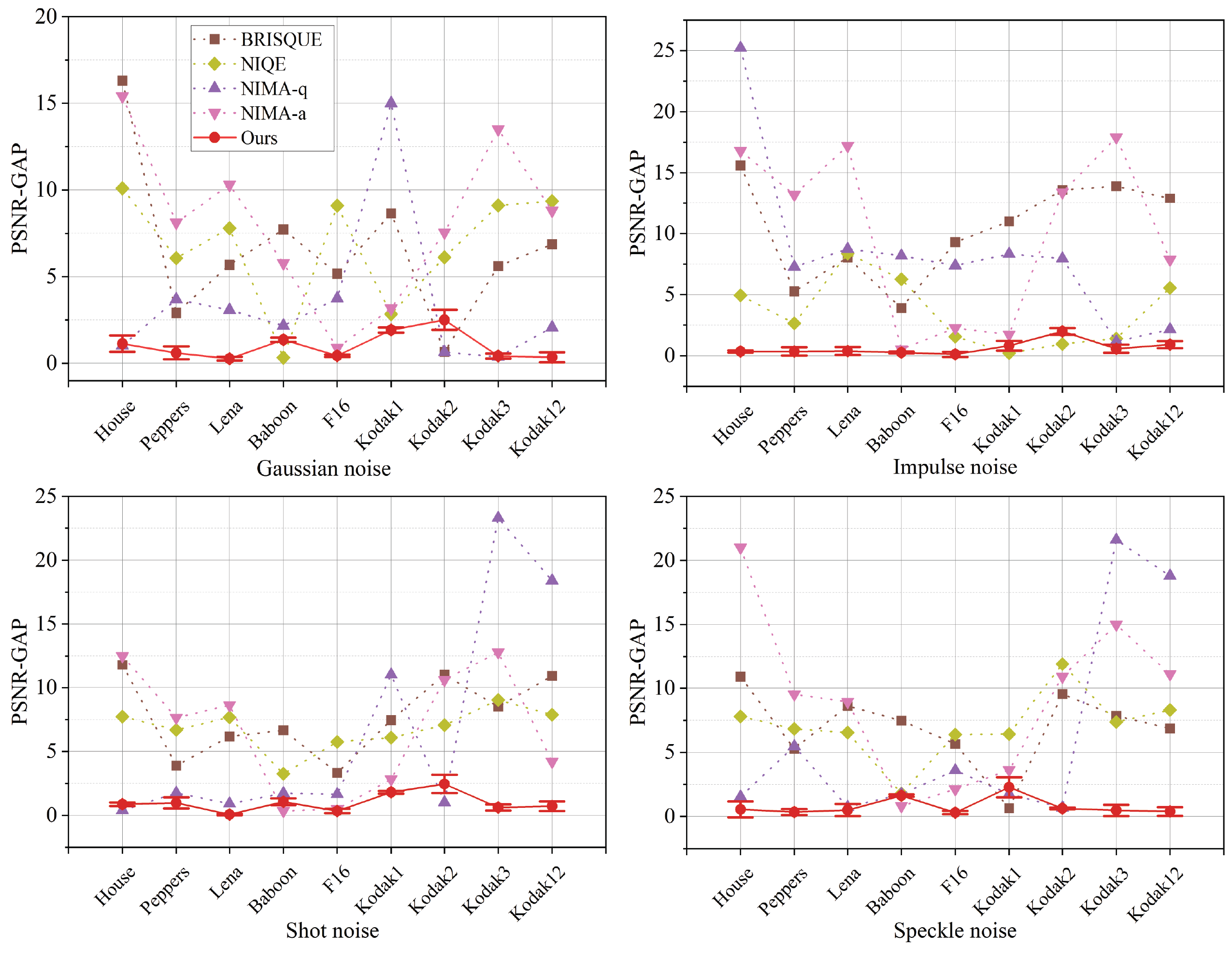}
    % \vspace{-1em}
    \caption{\textbf{High-level noise} detection performance in terms of PSNR gaps. For NIMA, we report both technical quality assessment (NIMA-q) and aesthetic assessment (NIMA-a). Smaller PSNR gaps are better.}
    \label{fig:baseline_h_psnr}
\end{figure*}
\begin{figure*}
    \centering
     % \vspace{-1em}
    \includegraphics[width=0.75\linewidth]{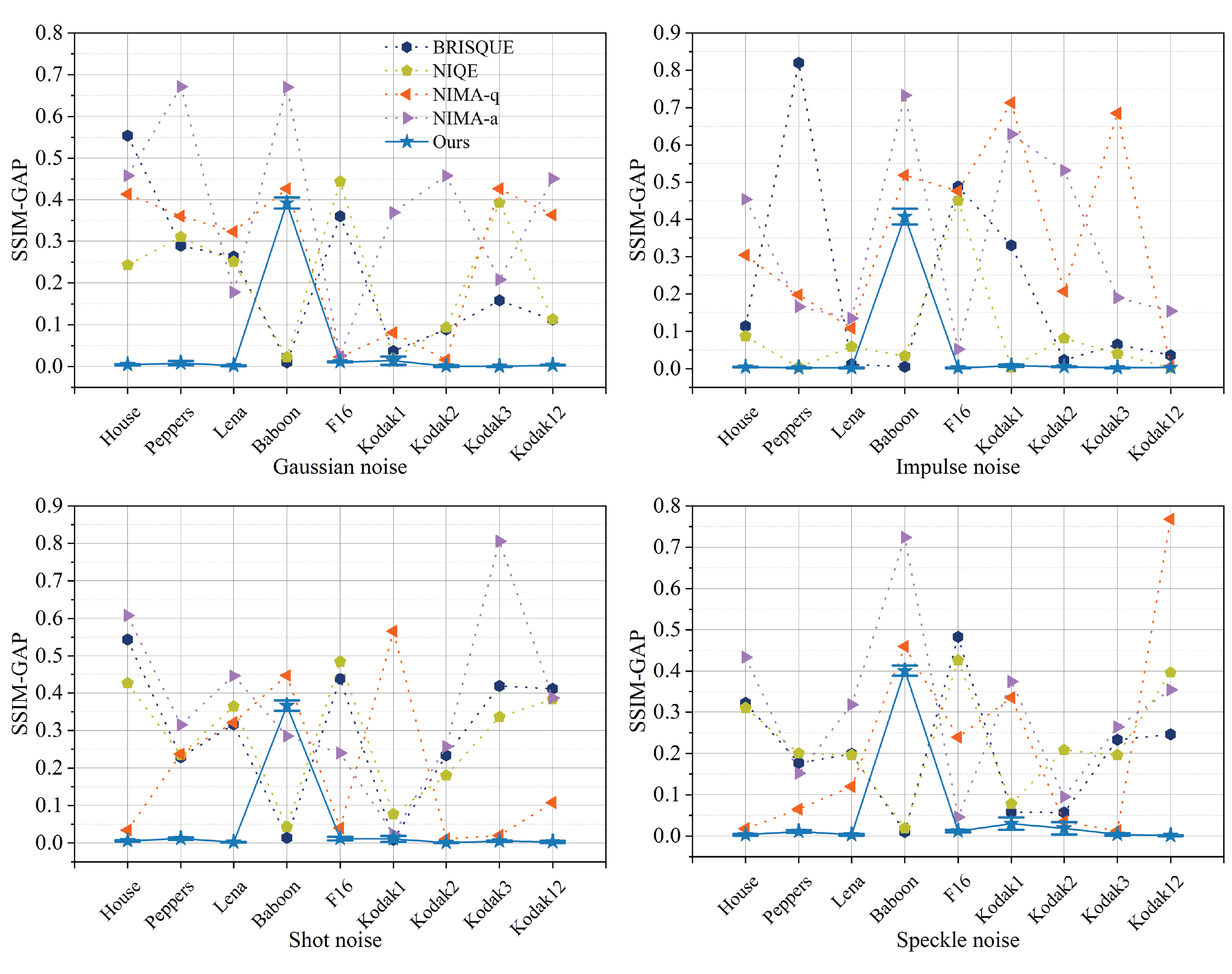}
     % \vspace{-1em}
    \caption{\textbf{Low-level noise} detection performance in terms of SSIM gaps. For NIMA, we report both technical quality assessment (NIMA-q) and aesthetic assessment (NIMA-a). Smaller SSIM gaps are better.}
    \label{fig:baseline_l_ssim}
\end{figure*}
\begin{figure*}
    \centering
     % \vspace{-1em}
    \includegraphics[width=0.75\linewidth]{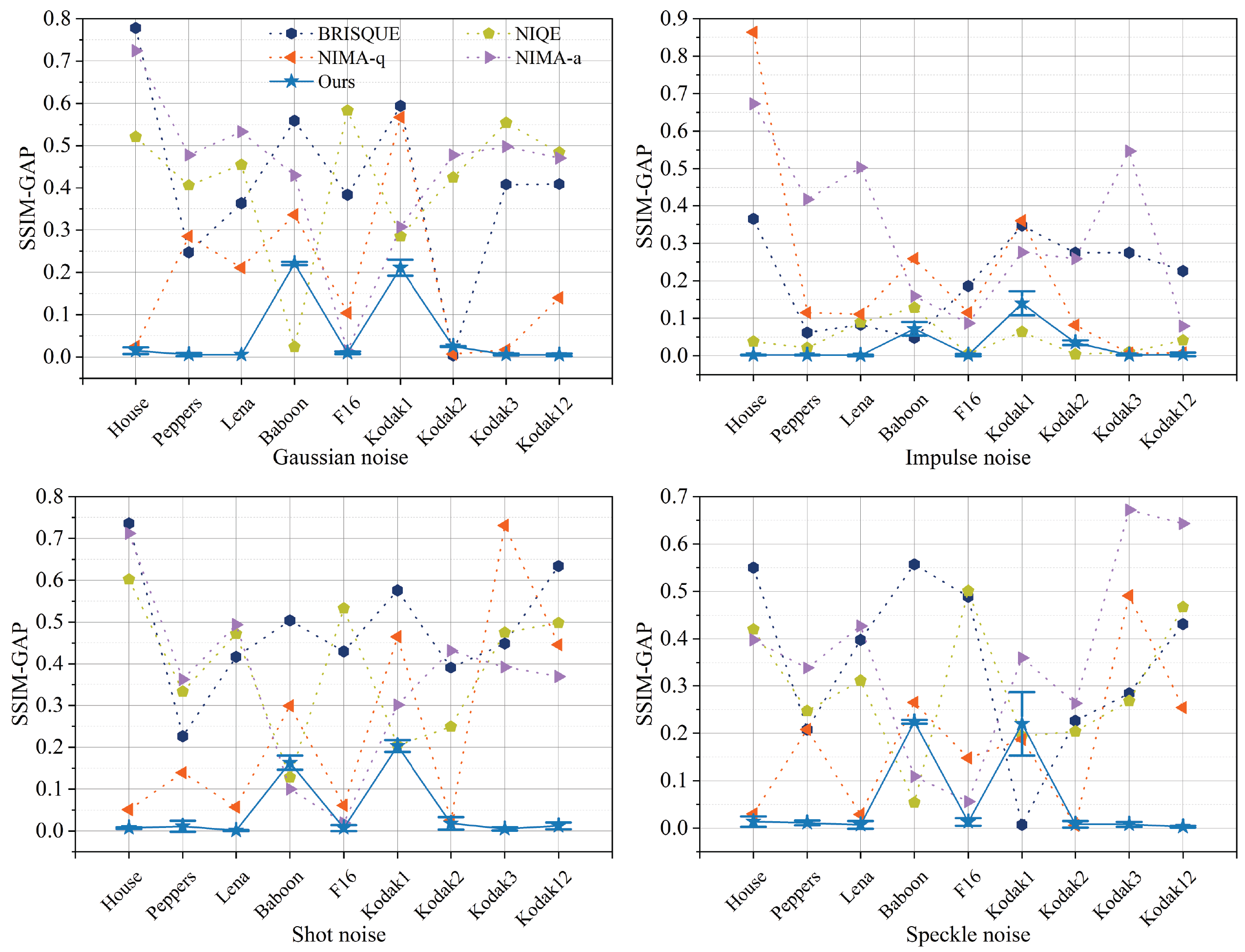}
     % \vspace{-1em}
    \caption{\textbf{High-level noise} detection performance in terms of SSIM gaps. For NIMA, we report both technical quality assessment (NIMA-q) and aesthetic assessment (NIMA-a). Smaller SSIM gaps are better.}
    \label{fig:baseline_h_ssim}
\end{figure*}

\subsubsection{Comparison with baseline methods}
\label{sec:baselines_ap}
To further compare with baseline methods, we report the PSNR gaps in high-level noise cases and the SSIM gaps in low- and high-level noise cases in \cref{fig:baseline_h_psnr},\cref{fig:baseline_l_ssim} and \cref{fig:baseline_h_ssim}, respectively, which show a trend similar to the results of PSNR gaps. The detection gaps of our method are very marginal ($< 0.02$) for most types and levels of noise (except Baboon and Kodak1 for certain types / levels of noise), while the baseline methods can easily exceed $0.1$ for most cases. In addition, we provide some visual detection results in \cref{fig:noref_1,fig:noref_2}. Our ES-WMV significantly outperforms the four baseline methods visually.

\subsubsection{Comparison with competing methods}
\label{sec:competing_ap}
Comparison between ES-WMV with DF-STE for Gaussian and shot noise on the 9 image dataset in terms of SSIM is reported in \cref{fig:df_ste_ssim}. Furthermore, we also test our ES-WMV and DF-STE on CBSD68 in \cref{tab:rethink_cbsd}. Our ES-WMV wins in high-level noise cases but lags behind DF-STE in the low-level cases. The gaps between our ES-WMV and DF-STE for all noise levels mostly come from the peak performance between the original DIP and DF-STE---modifications in DF-STE have affected peak performance, positively for low-level cases and negatively for high-level cases, not much from our ES method, as evident from the uniformly small detection gaps reported in \cref{tab:rethink_cbsd}. Moreover, DF-STE can only handle Gaussian and Poisson noise for denoising, and the exact noise level is a required hyperparameter for their method to work. 

Then we compare our ES-WMV and SV-ES in \cref{fig:sv_es_psnr}. The DIP results with ES-WMV versus DOP in impulse noise are shown in \cref{tab:dop}. For SB, part of the qualitative detection results on the 9 images\footnote{\url{http://www.cs.tut.fi/\~foi/GCF-BM3D/index.html\#ref_results}} are reported in \cref{fig:SB_compare}.

For reference, we compare DIP with the recent one-shot methods based on   diffusion models for solving linear IPs---DDNM+ for image denoising, as shown in \cref{tab:denoising_dm}. Like for \cref{tab:sr_2}, we observe that (1) \textbf{Our ES-WMV is again able to detect near-peak performance for most images}: the average PSNR gap is $\leq 1.02$ and the average SSIM gap is $\leq 0.08$; (2) DDNM+ is sensitive to the noise type and level: from \cref{tab:denoising_dm}, DDNM+ outperforms DIP and DIP+ES-WMV when there is Gaussian noise, but this is when the noise level set for pretraining DDNM+ matches the true noise level, i.e., $\sigma_y=0.18$, \textbf{which is unrealistic in practice as the noise level is not known beforehand}. When the noise level is not set correctly, e.g., as $\sigma_y=0$ in the {DDNM+} ($\sigma_y=.00$) row of \cref{tab:denoising_dm}, the performance of DDNM+ is much worse than that of DIP and DIP+ES-WMV. Also, for impluse noise denoising, DIP is also a clear winner that leads {DDNM+} by a large margin; 
and (3) in \cref{sec:diffusion}, we show that DDNM+ may also suffer from the overfitting issue and our ES-WMV can help DDNM+ to stop around the performance peak as well.
\begin{figure*}
    \centering
     % \vspace{-1em}
    \includegraphics[width=0.75\linewidth]{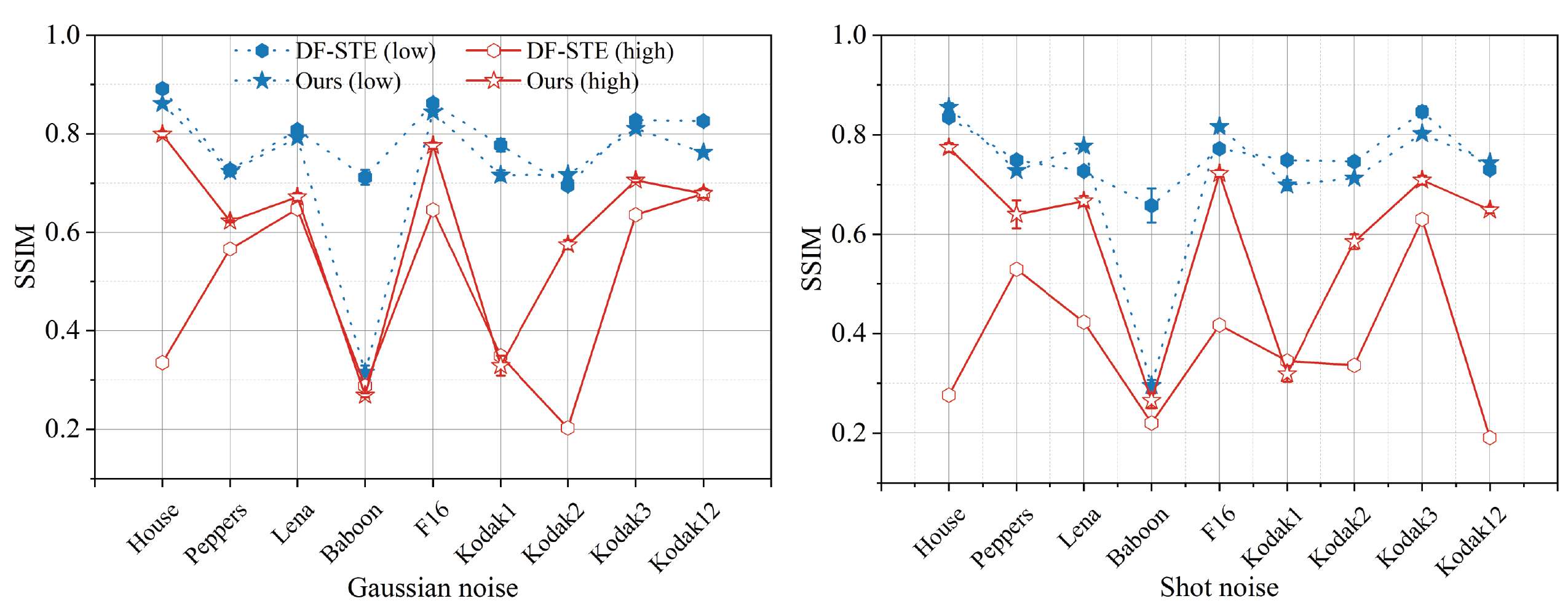}
    \caption{Comparison of DF-STE and ES-WMV for Gaussian and shot noise in terms of SSIM.}
     % \vspace{-1em}
    \label{fig:df_ste_ssim}
\end{figure*}
\begin{table}[!htpb]
    \centering
    \caption{Comparison between ES-WMV and DF-STE for image denoising on the CBSD68 dataset with varying noise level $\sigma$: mean and \scriptsize{(std)}. \normalsize{PSNR gaps below $1.0$ are colored as \textcolor{red}{red}}.}
    \label{tab:rethink_cbsd}
    \setlength{\tabcolsep}{4mm}{
    \begin{tabular}{c c c c}
    %\toprule
    \hline
    %\multirow{2}{*}{\scriptsize{Learning Rate}}
    &
    \multicolumn{1}{c}{\scriptsize{$\sigma=15$}} &
    \multicolumn{1}{c}{\scriptsize{$\sigma=25$}} &
    \multicolumn{1}{c}{\scriptsize{$\sigma=50$}}
    \\
    \hline
    \scriptsize{ES-WMV}
    & \scriptsize{28.7}\tiny({3.2})
    &{\scriptsize{27.4}}\tiny({2.6})
    & {\scriptsize{24.2}}\tiny({2.3})
 
    \\
    % \hline
    \scriptsize{DIP (Peak)}
    & \scriptsize{29.7}\tiny({3.0})
    &{\scriptsize{28.0}}\tiny({2.4})
    & {\scriptsize{24.9}}\tiny({2.3})
    
    \\
    % \hline
    \scriptsize{PSNR Gap}
    & \textcolor{red}{\scriptsize{1.0}}\tiny({0.7})
    &\textcolor{red}{\scriptsize{0.7}}\tiny({0.5})
    & \textcolor{red}{\scriptsize{0.7}}\tiny({0.5})
    
    \\
    
    % \hline
    \scriptsize{DF-STE}
    &{\scriptsize{31.4}}\tiny({1.8})
    & \scriptsize{28.4}\tiny({2.2})
    & \scriptsize{21.1}\tiny({2.5})
    \\
    \hline
    \end{tabular}
    }
    \end{table}
\begin{figure*}
    \centering
    \includegraphics[width=0.75\linewidth]{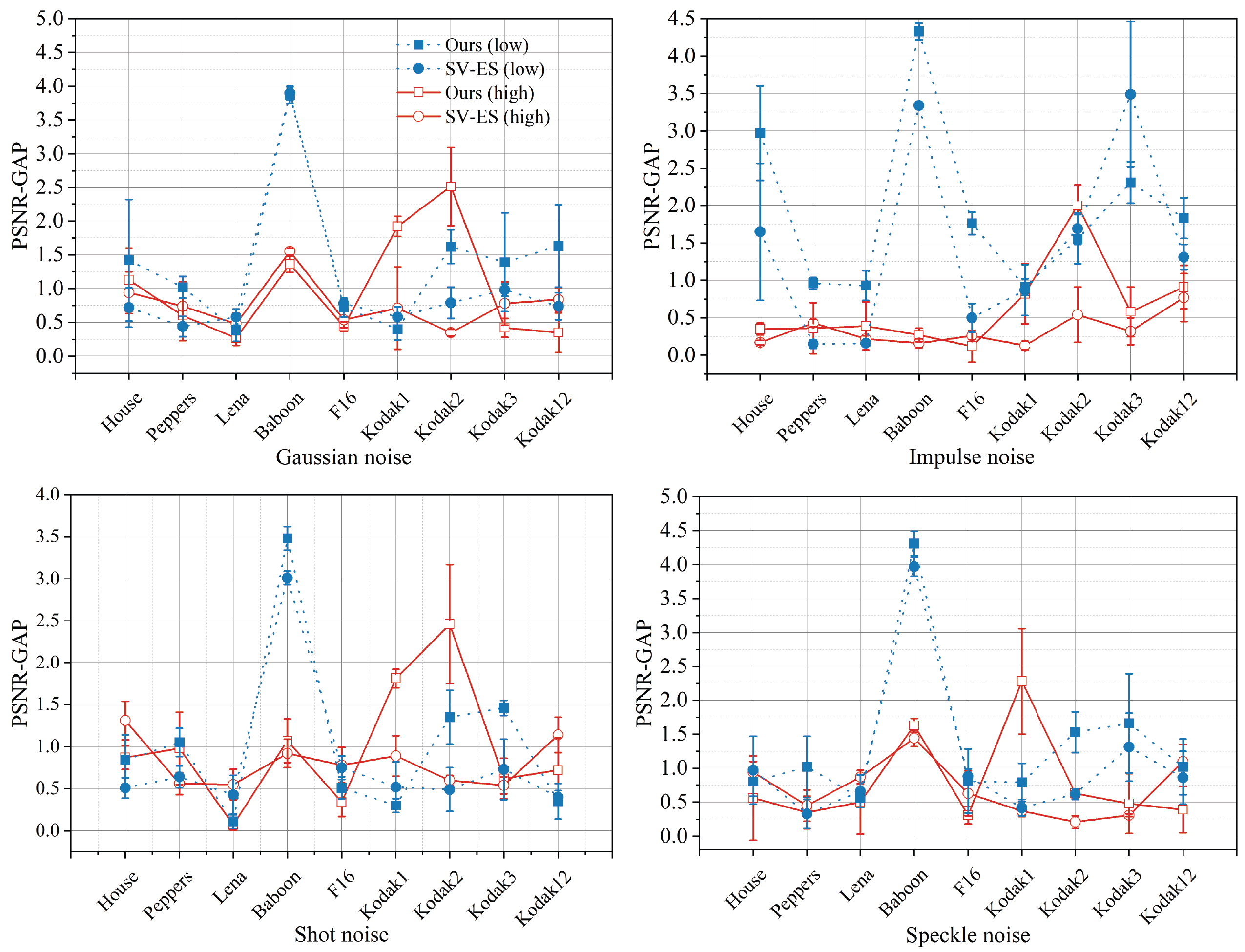}
    \caption{\textbf{Low- and high-level noise} detection performance of SV-ES and ours in terms of PSNR gaps.}
    \label{fig:sv_es_psnr}
\end{figure*}
\begin{table}[!htpb]
    \centering 
    \caption{DIP with ES-WMV vs. DOP on impulse noise: mean and \scriptsize{(std)}.}
    \label{tab:dop}
    \setlength{\tabcolsep}{3mm}{
    \begin{tabular}{c c c c c}
    %\toprule
    \hline
    %\multirow{2}{*}{\scriptsize{Learning Rate}}
    &
    \multicolumn{2}{c}{\scriptsize{Low Level}} &
    \multicolumn{2}{c}{\scriptsize{High Level}}
    \\
    \hline
    &
    \multicolumn{1}{c}{\scriptsize{PSNR}} &
    \multicolumn{1}{c}{\scriptsize{SSIM}}
    &
    \multicolumn{1}{c}{\scriptsize{PSNR}} &
    \multicolumn{1}{c}{\scriptsize{SSIM}}
    \\
    \hline
    \scriptsize{DIP-ES}
    & \scriptsize{31.64} \tiny({5.69})
    & \scriptsize{0.85} \tiny({0.18})
    & \scriptsize{24.74} \tiny({3.23})
    & \scriptsize{0.67} \tiny({0.19})
    \\
    
    % \hline
    \scriptsize{DOP}
    & \scriptsize{32.12} \tiny({4.52})
    & \scriptsize{0.92} \tiny({0.07})
    & \scriptsize{27.34} \tiny({3.78})
    & \scriptsize{0.86} \tiny({0.10})
    \\
    \hline
    \end{tabular}
    }
    \end{table}
\begin{figure*}[!htbp]
    \centering
    \includegraphics[width=0.9\linewidth]{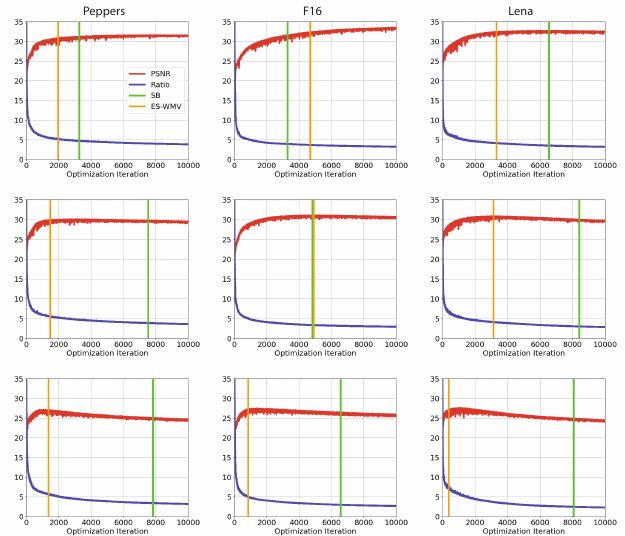}
    \caption{Comparison between ES-WMV and SB for image denoising (top: $\sigma=15$; middle: $\sigma=25$; bottom: $\sigma=50$). The red and blue curves are the PNSR and the ratio metric curves. The orange and green bars indicate the ES points detected by our ES-WMV and SB, respectively. }
    \label{fig:SB_compare}
\end{figure*}

\begin{table}[!htpb]
\begin{center}
\caption{Comparison of ES-WMV for DIP and DDNM+~\cite{wang_zero-shot_2022} for \textbf{denoising images with medium-level Gaussian and impulse noise}: mean and \scriptsize{(std)}. \normalsize{The highest PSNR and SSIM for each task are in \textcolor{red}{red}}. In particular, we set the best hyperparameter for {DDNM+} ($\sigma_y=0.18$), \textbf{which is unfair for the DIP + ES-WMV combination as we fix its hyperparameter setting}.}
\label{tab:denoising_dm}
\setlength{\tabcolsep}{1mm}{
\begin{tabular}{ccccc}
\hline
              & \multicolumn{2}{c}{\footnotesize{PSNR}}                                           & \multicolumn{2}{c}{\footnotesize{SSIM}}                                           \\ \hline
              & \multicolumn{1}{c}{\footnotesize{Gaussian}} &\footnotesize{Impulse} & \multicolumn{1}{c}{\footnotesize{Gaussian}} & \footnotesize{Impulse} \\ \hline
\footnotesize{DIP (peak)}           & \multicolumn{1}{c}{\scriptsize{24.63} \tiny{(2.06)}}            & \textcolor{red}{\scriptsize{37.75}} \tiny{(3.32)}           & \multicolumn{1}{c}{\scriptsize{0.68} \tiny{(0.06)}}             & \textcolor{red}{\scriptsize{0.96}} \tiny{(0.10)}            \\ 
\footnotesize{DIP + ES-WMV} & \multicolumn{1}{c}{\scriptsize{23.61} \tiny{(2.67)}}             & \scriptsize{36.87} \tiny{(4.29)}            & \multicolumn{1}{c}{\scriptsize{0.60} \tiny{(0.13)}}             & \textcolor{red}{\scriptsize{0.96}} \tiny{(0.10)}            \\ 
\footnotesize{{DDNM+} ($\sigma_y=.18$)}  & \multicolumn{1}{c}{\textcolor{red}{\scriptsize{26.93}} \tiny{(2.25)}}            & \scriptsize{22.29} \tiny{(3.00)}           & \multicolumn{1}{c}{\textcolor{red}{\scriptsize{0.78}} \tiny{(0.07)}}             & \scriptsize{0.62} \tiny{(0.12)}            \\ 
\footnotesize{{DDNM+} ($\sigma_y=.00$)}         & \multicolumn{1}{c}{\scriptsize{15.66} \tiny{(0.39)}}            & \scriptsize{15.52} \tiny{(0.43)}           & \multicolumn{1}{c}{\scriptsize{0.25} \tiny{(0.10)}}             & \scriptsize{0.30} \tiny{(0.10)}            \\ \hline
\end{tabular}
}
\end{center}
\end{table}

\subsubsection{ES-WMV as a helper}  
\label{sec:helper_ap}
Performance of ES-WMV on DD, GP-DIP, DIP-TV, and SIREN for Gaussian denoising in terms of SSIM gaps (see \cref{fig:dd_gp_tv_ssim}).

\begin{figure}[!htbp]
    \centering
    % \vspace{-1em}
    \includegraphics[width=0.7\linewidth]{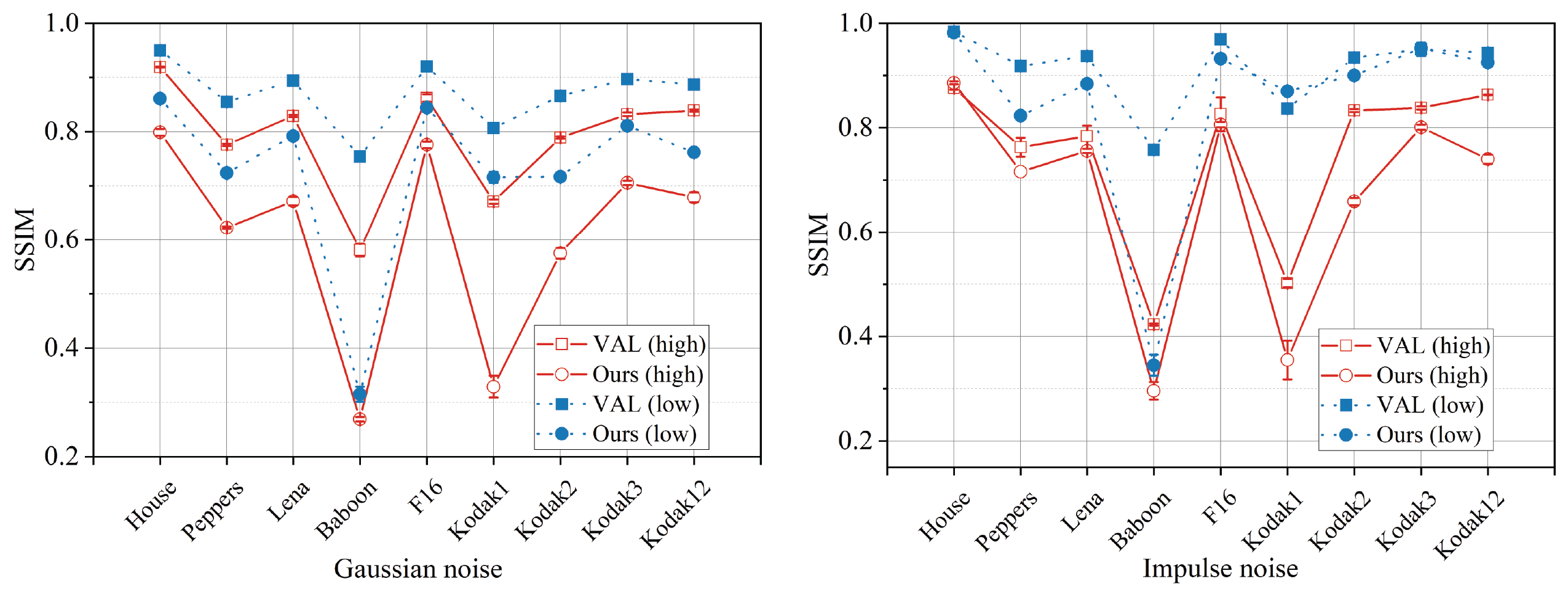}
    % \vspace{-1em}
    \caption{Comparison of VAL and ES-WMV for Gaussian and impulse noise in terms of SSIM.}
    \label{fig:val_ssim}
\end{figure}

\begin{figure}[!htbp]
    \centering
    % \vspace{-1em}
    \includegraphics[width=0.7\linewidth]{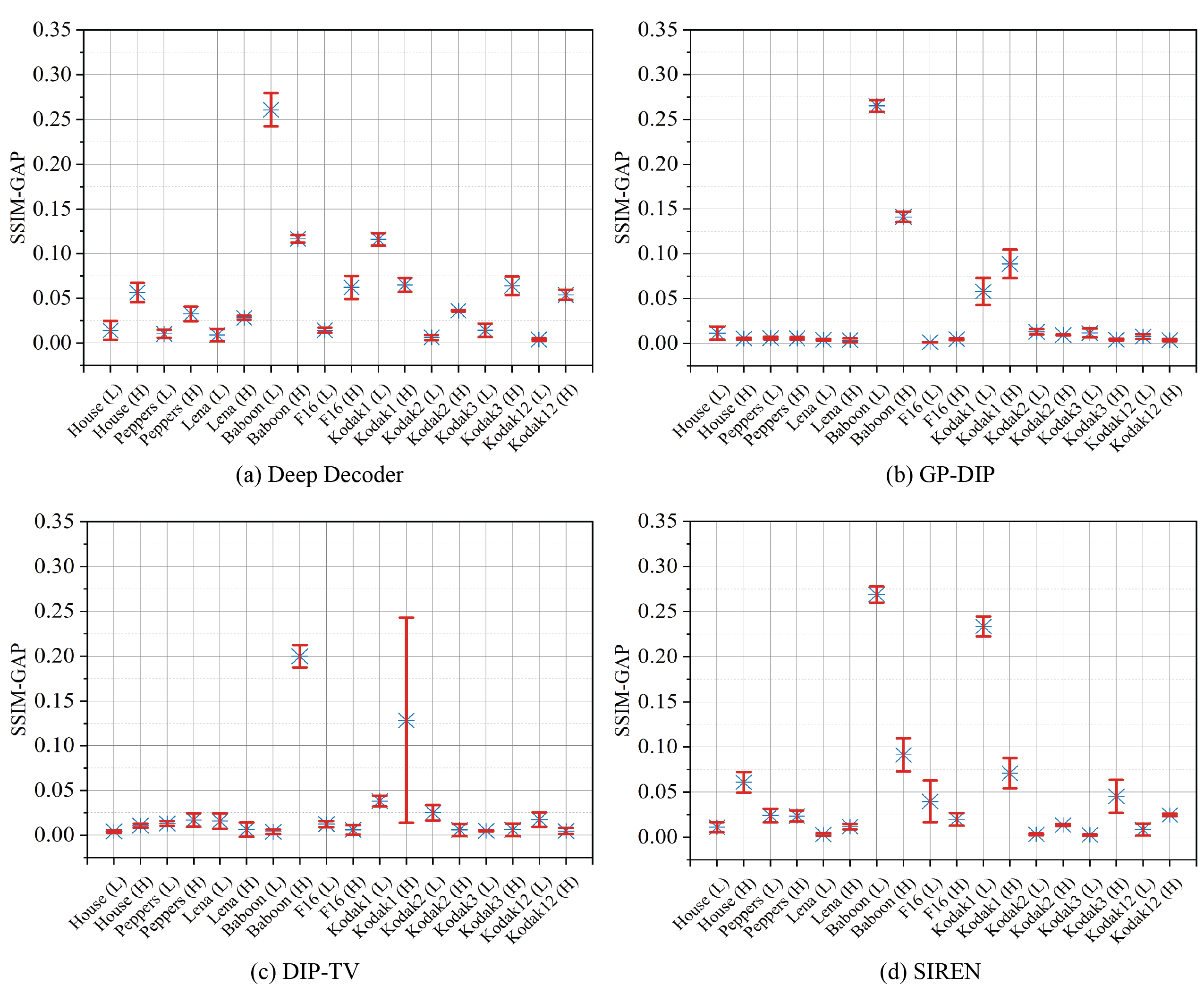}
    % \vspace{-1em}
    \caption{Performance of ES-WMV on DD, GP-DIP, DIP-TV, and SIREN for Gaussian denoising in terms of SSIM gaps. L: low noise level; H: high noise level}
    % \vspace{-1em}
    \label{fig:dd_gp_tv_ssim}
\end{figure}

\subsubsection{Performance on real-world denoising}
\label{sec:real_ap}

\begin{figure}[!htbp]
    \centering
    \includegraphics[width=0.7\linewidth]{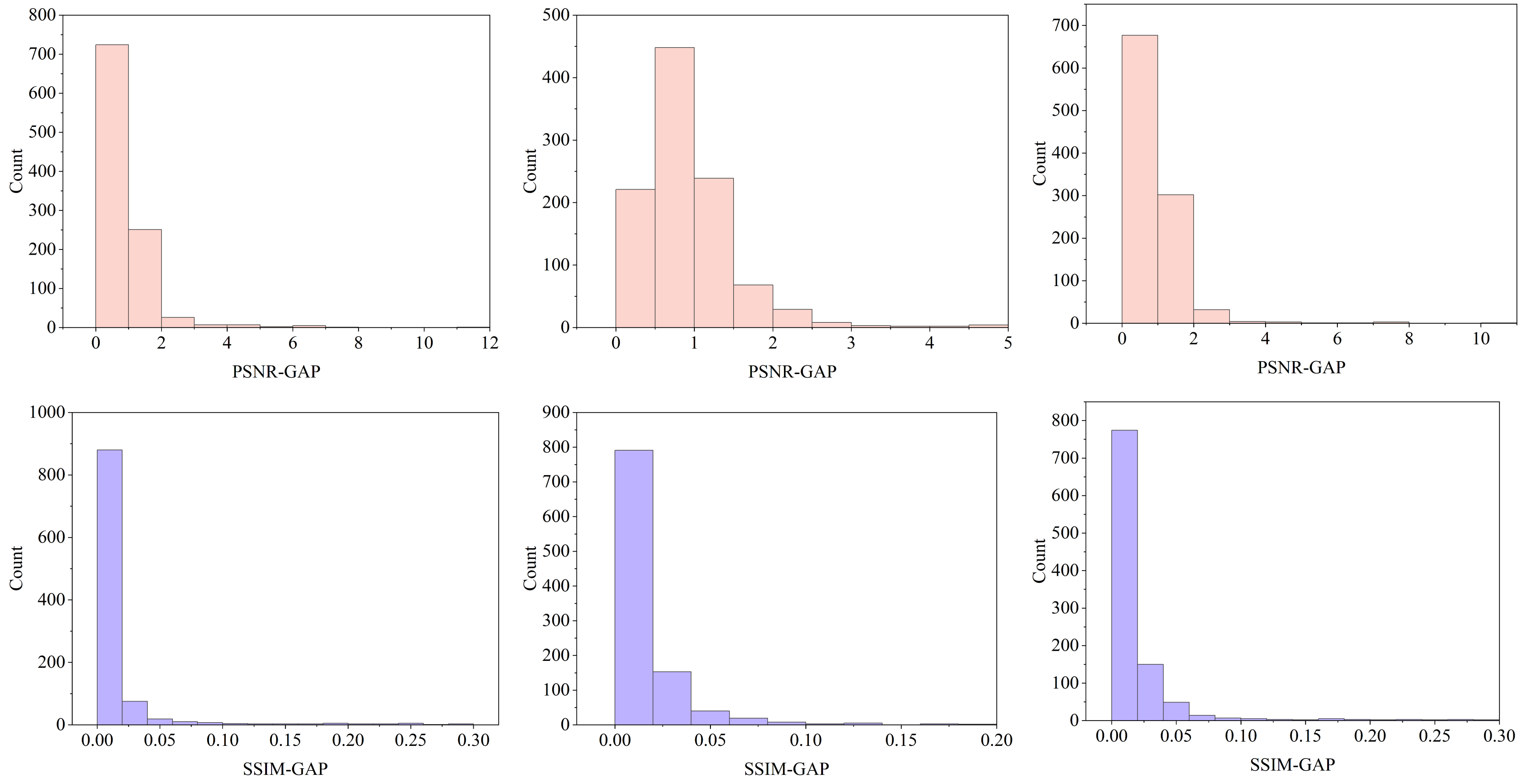}
    \caption{Image denoising of DIP + ES-WMV on the RGB track of the NTIRE 2020 Real Image Denoising Challenge. Top row: histograms of PSNR gaps for DIP (MSE), DIP ($\ell_1$) and DIP (Huber), respectively; bottom row: histograms of SSIM gaps for DIP (MSE), DIP ($\ell_1$) and DIP (Huber), respectively.}
    \label{fig:hist}
\end{figure}

We randomly sample $1024$ images from the RGB track of the NTIRE 2020 Real Image Denoising Challenge~\citep{abdelhamed2020ntire}, and perform DIP-based image denoising. Histograms of PSNR and SSIM gaps are shown in \cref{fig:hist}. For DIP with the three different losses, there are only $4.79\%$, $4.69\%$ and $4.40\%$ images, respectively, whose PSNR gaps are larger than $2dB$.
\begin{table}[!htpb]
    \centering
    \caption{DIP with ES-WMV on real image denoising on the PolyU Dataset: mean and \scriptsize{(std)}. \normalsize{(\textbf{D}: Detected)}}
    \label{tab:real_polyu}
    \setlength{\tabcolsep}{1.5mm}{
    \begin{tabular}{c c c c c}
    %\toprule
    \hline
    %\multirow{2}{*}{\scriptsize{Learning Rate}}
    &
    \multicolumn{1}{c}{\scriptsize{PSNR(\textbf{D})}} &
    \multicolumn{1}{c}{\scriptsize{PSNR Gap}}
    &
    \multicolumn{1}{c}{\scriptsize{SSIM(\textbf{D})}} &
    \multicolumn{1}{c}{\scriptsize{SSIM Gap}}
    \\
    \hline
    \scriptsize{DIP (MSE)}
    & \scriptsize{36.83} \tiny({3.07})
    & \textcolor{red}{\scriptsize{1.26}} \tiny({1.22})
    & \scriptsize{0.98} \tiny({0.02})
    & \textcolor{red}{\scriptsize{0.01}} \tiny({0.01})
    \\
    
    % \hline
    \scriptsize{DIP ($\ell_1$)}
    & \scriptsize{36.20} \tiny({2.81})
    & \textcolor{red}{\scriptsize{1.64}} \tiny({1.58})
    & \scriptsize{0.97} \tiny({0.02})
    & \textcolor{red}{\scriptsize{0.01}} \tiny({0.01})
    \\
    
    % \hline
    \scriptsize{DIP (Huber)}
    & \scriptsize{36.76} \tiny({2.96})
    & \textcolor{red}{\scriptsize{1.28}} \tiny({1.09})
    & \scriptsize{0.98} \tiny({0.02})
    & \textcolor{red}{\scriptsize{0.01}} \tiny({0.01})
    \\
    \hline
    \end{tabular}
    }
    \end{table}
As stated from the beginning, ES-WMV is designed with real-world IPs, targeting unknown noise types and levels. Given the encouraging performance above, we test it on a common real-world denoising dataset---PolyU Dataset~\cite{xu_real-world_2018}, which contains $100$ cropped regions of $512 \times 512$ from $40$ scenes. The results are reported in \cref{tab:real_polyu}. We do not repeat the experiments here; the means and standard deviations are obtained over the $100$ images of the PolyU dataset. On average, our detection gaps are $\le 1.64$ in PSNR and $\le 0.01$ in SSIM for this dataset across various losses. The absolute PNSR and SSIM detected are surprisingly high.

% \subsubsection{Results for MRI reconstruction} \label{sec:mri_visual}

% The detection performance of ES-WMV for MRI reconstruction is shown in \cref{fig:mri_curve} in terms of SSIM.
% \begin{figure}[!htbp] 
%     \centering
%     \includegraphics[width=0.5\linewidth]{figures/MRI_Curve.png}
%     \caption{Detection on MRI reconstruction}
%     \label{fig:mri_curve}
% \end{figure}

\subsubsection{Image Inpainting} 
% \label{sec:inpainting_table}
% \subsection{Image Inpainting}
\label{sec:inpainting}

\begin{figure*}[!htbp]
    \centering 
    % \vspace{-1em}
    \includegraphics[width=0.9\linewidth]{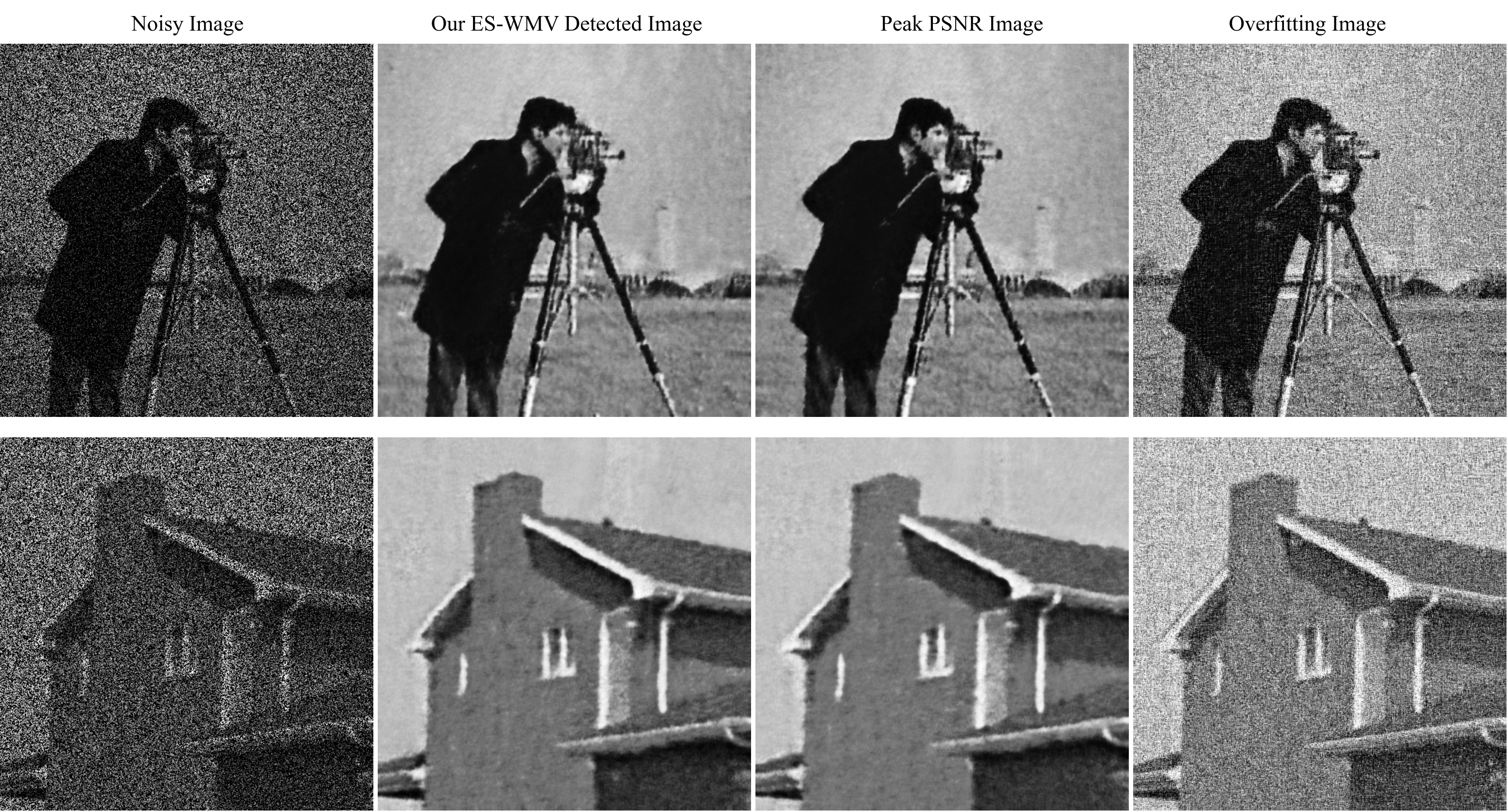}
    \caption{Visual detection performance of ES-WMV on image inpainting.}
    \label{fig:ip}
\end{figure*} 

In this task, a clean image $\mb x_0 \in [0, 1]^{H \times W}$ is contaminated by additive Gaussian noise $\epss$, and then only partially observed to yield the observation $\mb y = (\mb x_0 + \mb \epss) \odot \mb m$, where $\mb m \in \{0, 1\}^{H \times W}$ is a binary mask and $\odot$ denotes the Hadamard product. Given $\mb y$ and $\mb m$, the goal is to reconstruct $\mb x_0$. We consider the formulation reparametrized by DIP, where $G_{\mb \theta}$ is a trainable DNN parametrized by $\mb \theta$ and $\mb z$ is a frozen random seed: 
\begin{equation}
    \ell(\mb \theta) = \norm{ (G_{\mb \theta}\paren{\mb z}-\mb y) \odot \mb m }_F^2. 
\end{equation}
Mask $\mb m$ is generated according to an i.i.d. Bernoulli model with a rate of $50\%$, i.e., half of pixels not observed in expectation. The \textbf{noise $\epss$ is set to the medium level}, i.e., additive Gaussian with $0$ mean and $0.18$ variance. We test our ES-WMV for DIP on the inpainting dataset used in the original DIP paper~\cite{ulyanov2018deep}.  PSNR gaps are $\leq 1.00$ and SSIM gaps are $\leq 0.05$ for most cases (see \cref{tab:image_inpainting}). We also visualize two examples in \cref{fig:ip}.
% In this task, a clean image $\mb x_0 \in [0, 1]^{H \times W}$ is contaminated by additive Gaussian noise $\epss$, and then only partially observed to yield the observation $\mb y = (\mb x_0 + \mb \epss) \odot \mb m$, where $\mb m \in \{0, 1\}^{H \times W}$ is a binary mask and $\odot$ denotes the Hadamard product. Given $\mb y$ and $\mb m$, the goal is to reconstruct $\mb x_0$. We consider the formulation reparametrized by DIP, where $G_{\mb \theta}$ is a trainable DNN parametrized by $\mb \theta$ and $\mb z$ is a frozen random seed: 
% \begin{equation}
%     \ell(\mb \theta) = \norm{ (G_{\mb \theta}\paren{\mb z}-\mb y) \odot \mb m }_F^2. 
% \end{equation}
% Mask $\mb m$ is generated according to an i.i.d. Bernoulli model with a rate of $50\%$, i.e., half of pixels not observed in expectation. The \textbf{noise $\epss$ is set to the medium level}, i.e., additive Gaussian with $0$ mean and $0.18$ variance. We test our ES-WMV for DIP on the inpainting dataset used in the original DIP paper~\cite{ulyanov2018deep}.  PSNR gaps are $\leq 1.00$ and SSIM gaps are $\leq 0.05$ for most cases (see \cref{tab:image_inpainting}). We also visualize two examples in \cref{fig:ip}.
% \begin{figure*}[!htbp]
%     \centering 
%     % \vspace{-1em}
%     \includegraphics[width=0.98\linewidth]{figures/IP.png}
%     \caption{Visual detection performance of ES-WMV on image inpainting.}
%     \label{fig:ip}
% \end{figure*} 

\begin{table}[!htpb]
\begin{center}
\caption{Detection performance of DIP with ES-WMV for image inpainting: mean and \scriptsize{(std)}. \normalsize{PSNR gaps below $1.00$ are colored as \textcolor{red}{red}; SSIM gaps below $0.05$ are colored as {blue}. (\textbf{D}: Detected)}}
\label{tab:image_inpainting}
\setlength{\tabcolsep}{4mm}{
\begin{tabular}{l c c c c}
% \hline

% \multirow{2}{*}{\textit{}} & 
% \multicolumn{3}{c}{{\scriptsize{PSNR $\uparrow$}}} &
% \multicolumn{3}{c}{{\scriptsize{SSIM $\uparrow$}}}
\\
% \cmidrule(l){2-4}
% \cmidrule(l){5-7}
\hline
&
\scriptsize{PSNR(\textbf{D})} &
\scriptsize{PSNR Gap} &
\scriptsize{SSIM(\textbf{D})} &
\scriptsize{SSIM Gap}
 \\

\hline
\scriptsize{Barbara}
&{{\scriptsize{21.59}} \tiny{(0.03)}}
&{\textcolor{red}{\scriptsize{0.20}} \tiny{(0.03)}}
&{\scriptsize{0.67} \tiny{(0.00)}}
&{{\scriptsize{0.00}}} \tiny{(0.00)}
\\

% \hline
\scriptsize{Boat}
&{{\scriptsize{21.91}} \tiny{(0.10)}}
&{\scriptsize{1.16} \tiny{(0.18)}}
&{\scriptsize{0.68} \tiny{(0.00)}}
&{{\scriptsize{0.03}}} \tiny{(0.01)}
\\

% \hline
\scriptsize{House}
&{{\scriptsize{27.95}} \tiny{(0.33)}}
&{\textcolor{red}{\scriptsize{0.48}} \tiny{(0.10)}}
&{\scriptsize{0.89} \tiny{(0.01)}}
&{{\scriptsize{0.01}}} \tiny{(0.00)}
\\

% \hline
\scriptsize{Lena}
&{{\scriptsize{24.71}} \tiny{(0.30)}}
&{\textcolor{red}{\scriptsize{0.37}} \tiny{(0.18)}}
&{\scriptsize{0.80} \tiny{(0.00)}}
&{{\scriptsize{0.01}}} \tiny{(0.00)}
\\

% % \hline
\scriptsize{Peppers}
&{{\scriptsize{25.86}} \tiny{(0.22)}}
&{\textcolor{red}{\scriptsize{0.23}} \tiny{(0.05)}}
&{\scriptsize{0.84} \tiny{(0.01)}}
&{{\scriptsize{0.02}}} \tiny{(0.00)}
\\

% \hline
\scriptsize{C.man}
&{{\scriptsize{25.26}} \tiny{(0.09)}}
&{\textcolor{red}{\scriptsize{0.23}} \tiny{(0.14)}}
&{\scriptsize{0.82} \tiny{(0.00)}}
&{{\scriptsize{0.01}}} \tiny{(0.00)}
\\

% \hline
\scriptsize{Couple}
&{{\scriptsize{21.40}} \tiny{(0.44)}}
&{\scriptsize{1.21} \tiny{(0.53)}}
&{\scriptsize{0.63} \tiny{(0.01)}}
&{{\scriptsize{0.04}}} \tiny{(0.02)}
\\

% \hline
\scriptsize{Finger}
&{{\scriptsize{20.87}} \tiny{(0.04)}}
&{\textcolor{red}{\scriptsize{0.24}} \tiny{(0.17)}}
&{\scriptsize{0.77} \tiny{(0.00)}}
&{{\scriptsize{0.01}}} \tiny{(0.01)}
\\

% \hline
\scriptsize{Hill}
&{{\scriptsize{23.54}} \tiny{(0.08)}}
&{\textcolor{red}{\scriptsize{0.25}} \tiny{(0.11)}}
&{\scriptsize{0.70} \tiny{(0.00)}}
&{{\scriptsize{0.00}}} \tiny{(0.00)}
\\

% \hline
\scriptsize{Man}
&{{\scriptsize{22.92}} \tiny{(0.25)}}
&{\textcolor{red}{\scriptsize{0.46}} \tiny{(0.11)}}
&{\scriptsize{0.70} \tiny{(0.01)}}
&{{\scriptsize{0.01}}} \tiny{(0.00)}
\\

% \hline
\scriptsize{Montage}
&{{\scriptsize{26.16}} \tiny{(0.33)}}
&{\textcolor{red}{\scriptsize{0.38}} \tiny{(0.26)}}
&{\scriptsize{0.86} \tiny{(0.01)}}
&{{{\scriptsize{0.03}}} \tiny{(0.01)}}
\\

% \hline
% \scriptsize{\textbf{Average}}
% &{\textbf{{\scriptsize{23.83} \tiny{(2.21)}}}}
% &{\textbf{\scriptsize{0.47} \tiny{(0.35)}}}
% &{\textbf{\scriptsize{0.76} \tiny{(0.08)}}}
% &{\textbf{\scriptsize{0.02} \tiny{(0.01)}}}
% \\

\hline
\end{tabular}
}
\end{center}
\end{table}

\subsubsection{Image Super-resolution} 
\label{sec:appendix_sr}

Visual comparisons for $2 \times$ image super-resolution task with additional low-level Gaussian noise and impulse noise are shown in \cref{fig:sr_gau,fig:sr_im}, respectively.

\begin{figure*}[!htbp]
    \centering 
    % \vspace{-3em}
    \includegraphics[width=0.9\linewidth]{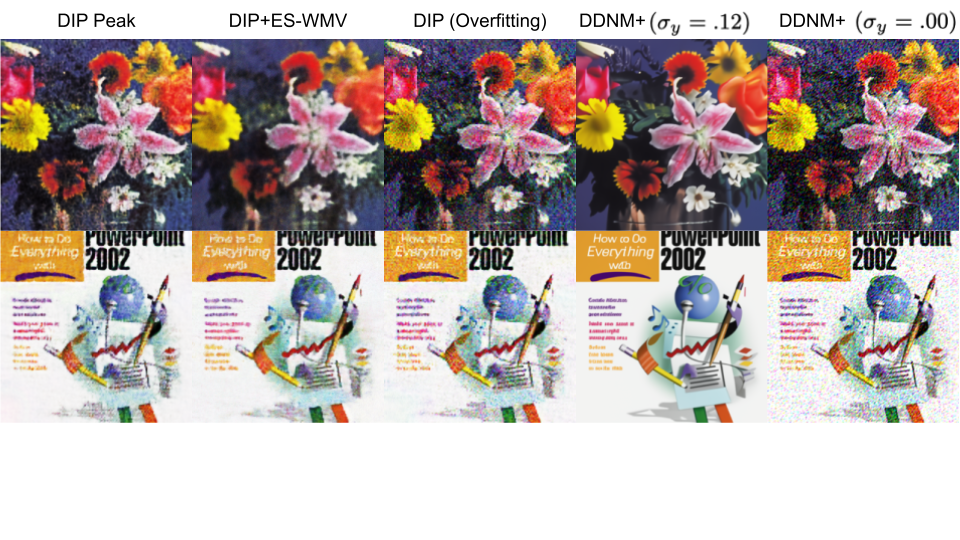}
    % \vspace{-4em}
    \caption{Visual comparisons for $2 \times$ image super-resolution task with additional low-level Gaussian noise.}
    \label{fig:sr_gau}
\end{figure*} 

\begin{figure*}[!htbp]
    \centering 
    % % \vspace{-3em}
    \includegraphics[width=0.9\linewidth]{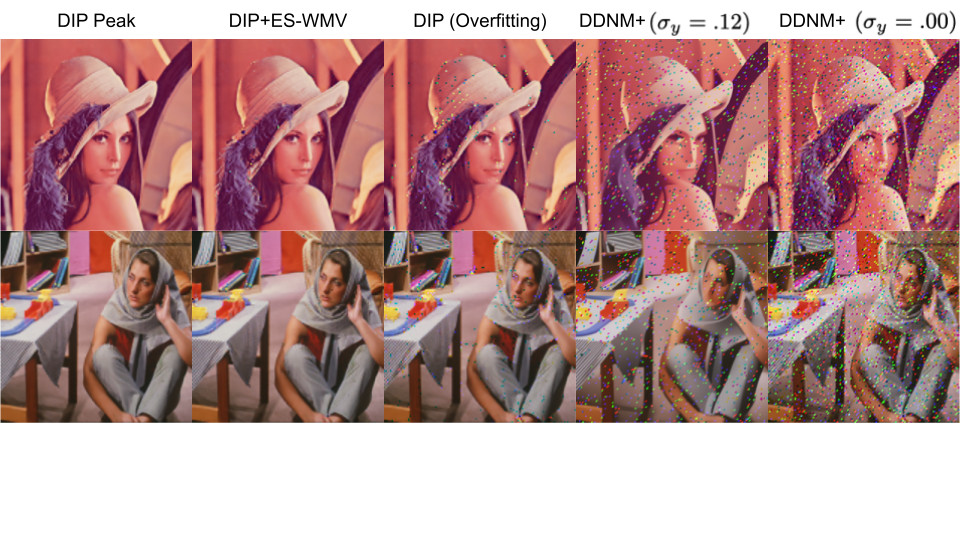}
    % \vspace{-4em}
    \caption{Visual comparisons for $2 \times$ image super-resolution task with additional low-level Impulse noise.}
    \label{fig:sr_im}
\end{figure*} 

\subsubsection{{RAW images demosaicing and denoising}}\label{sec:raw_img_demosaicing}

RAW images demosaicing and denoising are two essential procedures for modern digital cameras to produce high-quality full-color images (~\cite{jdd_doubledip}). Given a noisy RAW image $\wh{\mb x}^{1ch} = \mb x^{1ch} + \mb n$, where $H$ and $W$ are the height and width of the image and $\mb n$ is the noise, the goal is to obtain a high quality full-color image $\mb x^{3ch} \in \R^{H \times W \times 3}$ from it. To achieve this goal, it is obvious that we need to fill in the missing pixels (\emph{demosaicing}) and remove the noisy components $\mb n$ (\emph{denoising}). In this section, we formulate this problem as an image-inpainting problem as~\cite{jdd_doubledip} and adopt DIP to reconstruct the desired full-color image. In addition, we plug our early stopping method into DIP and explore the effectiveness of our method on this low-level vision task. We conduct experiments on the Kodak dataset\footnote{\url{https://r0k.us/graphics/kodak/}} and prepare it following the pipeline in~\cite{jdd_doubledip}. We experiment with Poisson noise ($\lambda=25$; a detailed description of the noise intensity could be found in~\cite{jdd_doubledip}), which is a very common noise under low light conditions. We report the experimental results in \cref{tab:demosaicing}. It is evident that our method could effectively detect the near-peak points and produce reliable early stopping signals for DIP.

\begin{wraptable}{r}{0.45\linewidth}
    \centering
    % \vspace{-3em}
    \caption{RAW images demosaicing and denoising on the Kodak dataset for \textbf{25 cases}: mean and \scriptsize{(std)}. \footnotesize{(\textbf{D}: Detected)}}
    \label{tab:demosaicing}
    \setlength{\tabcolsep}{1mm}{
    \begin{tabular}{c c c c}
    %\toprule
    \hline
    
    \multicolumn{1}{c}{\footnotesize{PSNR(\textbf{D})}} &
    \multicolumn{1}{c}{\footnotesize{PSNR Gap}}
    &
    \multicolumn{1}{c}{\footnotesize{SSIM(\textbf{D})}} &
    \multicolumn{1}{c}{\footnotesize{SSIM Gap}}
    \\
    \hline
    \scriptsize{24.22} \tiny({2.49})
    & \textcolor{red}{\scriptsize{0.92}} \tiny({0.87})
    & \scriptsize{0.58} \tiny({0.14})
    & \textcolor{red}{\scriptsize{0.06}} \tiny({0.08})
    \\
    \hline
    \end{tabular}
    }
    % % \vspace{-2em}
    \end{wraptable}

\subsubsection{MRI reconstruction}
\label{sec:sup_mri}

\begin{figure}
\centering
    \includegraphics[width=0.5\linewidth]{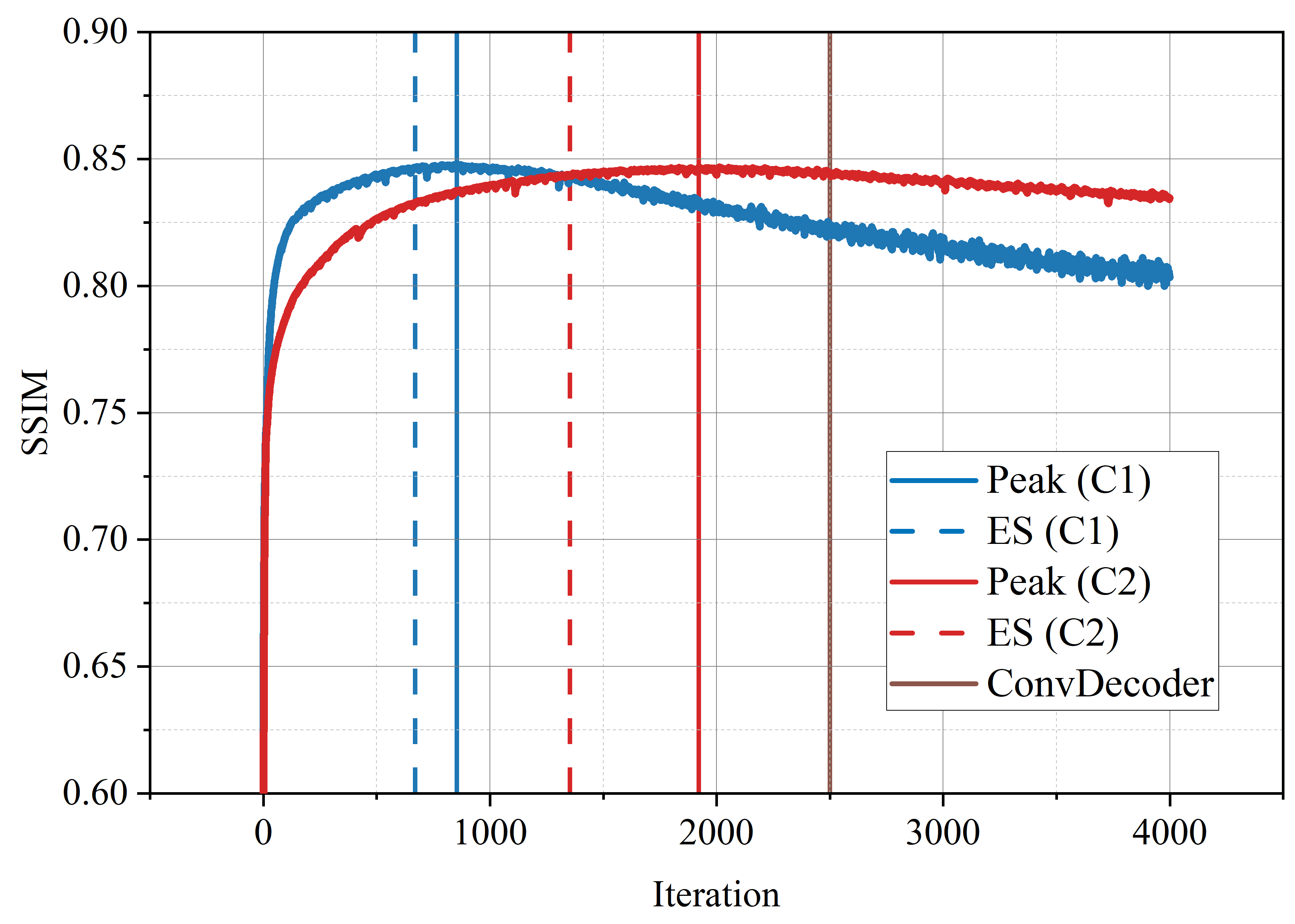}
    \caption{Detection on MRI reconstruction}
    \label{fig:mri_curve}
\end{figure}
We visualize the performance on two random cases (C1: $1001339$ and C2: $1000190$ sampled from \cite{darestani2021accelerated}, part of the fastMRI datatset~\citep{ZbontarEtAl2018fastMRI}) in \cref{fig:mri_curve} (quality measured in SSIM, consistent with~\cite{darestani2021accelerated}).

\subsubsection{Blind image deblurring (BID)}
\label{sec:sup_bid}

In this section, we systematically test our ES-WMV and VAL on the entire standard Levin dataset for both low-level and high-level cases. We set the maximum number of iterations to $10,000$ to ensure sufficient optimization. The detected images of our ES-WMV are substantially better than those of VAL, as shown in \cref{tab:level_bid}.
\begin{table}[!htpb]
    \centering
    \caption{BID detection comparison between ES-WMV and VAL on the Levin dataset for both low-level and high-level noise: mean and \scriptsize{(std)}.\footnotesize{Higher PSNR is in \textcolor{red}{red} and higher SSIM is in {blue}. (\textbf{D}: Detected)}}
    \label{tab:level_bid}
    \setlength{\tabcolsep}{3mm}{
    \begin{tabular}{c c c c c}
    %\toprule
    \hline
    %\multirow{2}{*}{\scriptsize{Learning Rate}}
    &
    \multicolumn{2}{c}{\scriptsize{Low Level}} &
    \multicolumn{2}{c}{\scriptsize{High Level}}
    \\
    \hline
    &
    \multicolumn{1}{c}{\scriptsize{PSNR(\textbf{D})}} &
    \multicolumn{1}{c}{\scriptsize{SSIM(\textbf{D})}} &
    \multicolumn{1}{c}{\scriptsize{PSNR(\textbf{D})}}
    &
    \multicolumn{1}{c}{\scriptsize{SSIM(\textbf{D})}}
    \\
    \hline
    \scriptsize{WMV}
    & \textcolor{red}{\scriptsize{28.54}}\tiny({0.61})
    & {\scriptsize{0.83}}\tiny({0.04})
    & \textcolor{red}{\scriptsize{26.41}}\tiny({0.67})
    & {\scriptsize{0.76}}\tiny({0.04})
    \\
    
    % \hline
    \scriptsize{VAL}
    & {\scriptsize{18.87}}\tiny({1.44})
    & \scriptsize{0.50}\tiny({0.09})
    & \scriptsize{16.69}\tiny({1.39})
    & \scriptsize{0.44}\tiny({0.10})
    \\
    \hline
    \end{tabular}
    }
    \end{table}
    
\subsubsection{ES-WMV vs. ES-EMV}
\label{sec:es_wmv_emv}
We now consider our memory-efficient version (ES-EMV) as described in \cref{alg:framework_emavg}, and compare it with ES-WMV, as shown in \cref{fig:dip_ema}. Besides the memory benefit, ES-EMV runs around 100 times faster than ES-WMV, as reported in \cref{tab:wall-clock time} and does seem to provide a consistent improvement on the detected PSNRs for image denoising tasks on NTIRE 2020 Real Image Denoising Challenge~\citep{abdelhamed2020ntire}, PolyU dataset~\cite{xu_real-world_2018} 
\begin{figure}
    \centering
    % \vspace{-2em}
    \includegraphics[width=0.5\linewidth]{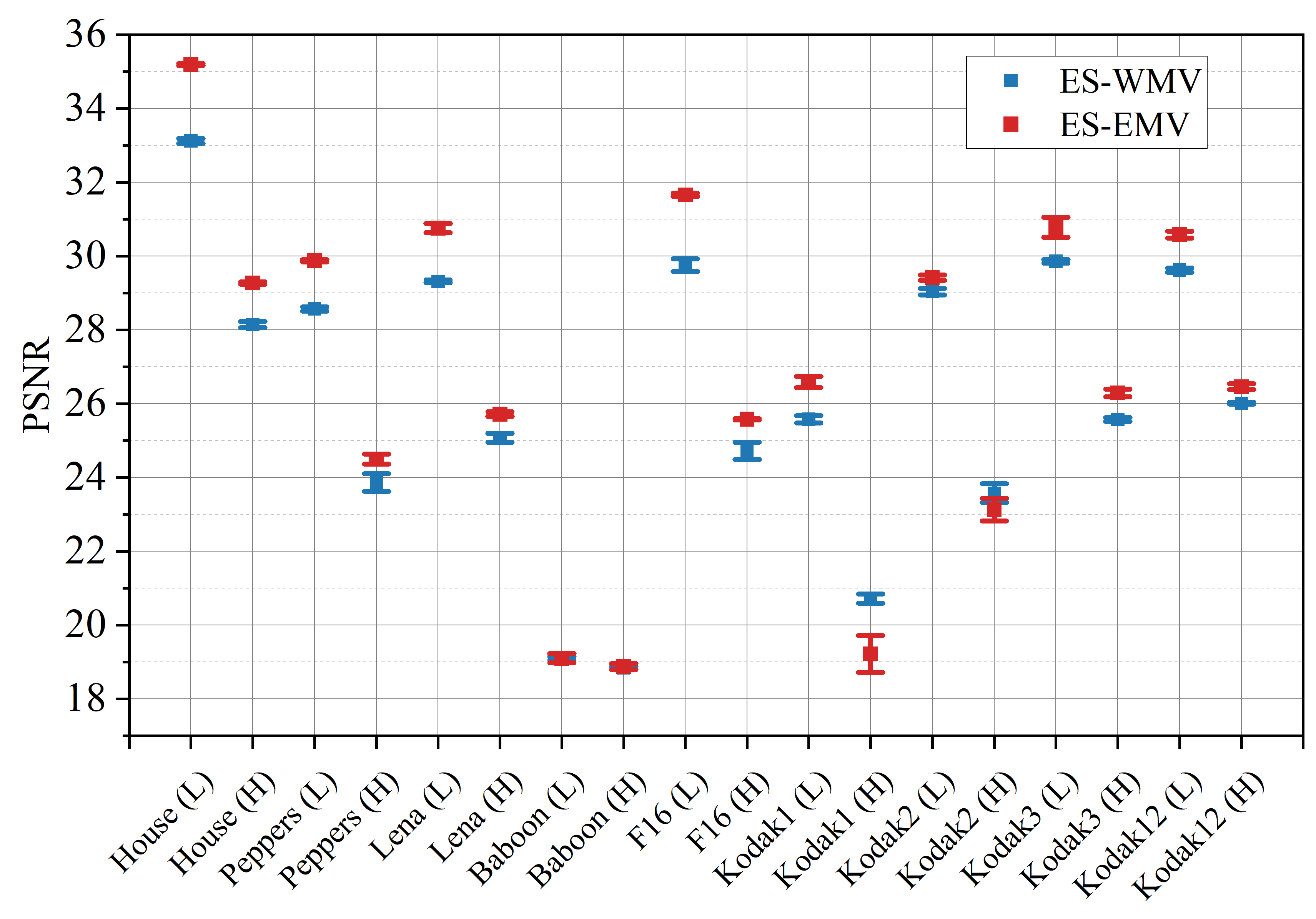}
    \caption{Detected PSNR comparison between DIP with ES-WMV and DIP with ES-EMV on the classic $9$-image dataset~\citep{Dabov2007}.}
    \label{fig:dip_ema}
\end{figure}
and the classic $9$-image dataset~\citep{Dabov2007} (see \cref{tab:wmv_emv_1024,tab:wmv_emv_poly,fig:dip_ema}), due to the strong smoothing effect (we set $\alpha = 0.1$). In this paper, we prefer to keep it simple and leave systematic evaluations of these variants for future work.

\begin{table}[!htpb]
    \centering
    \caption{Detection performance comparison between DIP with ES-WMV and DIP with ES-EMV for real image denoising on $1024$ images from the RGB track of NTIRE 2020 Real Image Denoising Challenge~\citep{abdelhamed2020ntire}: mean and \scriptsize{(std)}. \normalsize{Higher PSNR and SSIM are in \textcolor{red}{red}. (\textbf{D}: Detected)}}
    \label{tab:wmv_emv_1024}
    \setlength{\tabcolsep}{1.5mm}{
    \begin{tabular}{c c c c c}
    %\toprule
    \hline
    %\multirow{2}{*}{\scriptsize{Learning Rate}}
    &
    \multicolumn{1}{c}{\scriptsize{PSNR(\textbf{D})-WMV}} &
    \multicolumn{1}{c}{\scriptsize{PSNR(\textbf{D})-EMV}}
    &
    \multicolumn{1}{c}{\scriptsize{SSIM(\textbf{D})-WMV}} &
    \multicolumn{1}{c}{\scriptsize{SSIM(\textbf{D})-EMV}}
    \\
    \hline
    \scriptsize{DIP (MSE)}
    & \scriptsize{34.04} \tiny({3.68})
    & \textcolor{red}{\scriptsize{34.96}} \tiny({3.80})
    & \scriptsize{0.92} \tiny({0.07})
    & \textcolor{red}{\scriptsize{0.93}} \tiny({0.07})
    \\
    
    % \hline
    \scriptsize{DIP ($\ell_1$)}
    & \scriptsize{33.92} \tiny({4.34})
    & \textcolor{red}{\scriptsize{34.83}} \tiny({4.35})
    & \scriptsize{0.93} \tiny({0.05})
    & \textcolor{red}{\scriptsize{0.94}} \tiny({0.05})
    \\
    
    % \hline
    \scriptsize{DIP (Huber)}
    & \scriptsize{33.72} \tiny({3.86})
    & \textcolor{red}{\scriptsize{34.72}} \tiny({4.04})
    & \scriptsize{0.92} \tiny({0.06})
    & \textcolor{red}{\scriptsize{0.93}} \tiny({0.06})
    \\
    \hline
    \end{tabular}
    }
    \end{table}

\begin{table}[!htpb]
    \centering
    \caption{Detection performance comparison between DIP with ES-WMV and DIP with ES-EMV for real image denoising on the PolyU dataset~\cite{xu_real-world_2018}: mean and \scriptsize{(std)}. \normalsize{Higher PSNR and SSIM are in \textcolor{red}{red}. (\textbf{D}: Detected)}}
    \label{tab:wmv_emv_poly}
    \setlength{\tabcolsep}{1.5mm}{
    \begin{tabular}{c c c c c}
    %\toprule
    \hline
    %\multirow{2}{*}{\scriptsize{Learning Rate}}
    &
    \multicolumn{1}{c}{\scriptsize{PSNR(\textbf{D})-WMV}} &
    \multicolumn{1}{c}{\scriptsize{PSNR(\textbf{D})-EMV}}
    &
    \multicolumn{1}{c}{\scriptsize{SSIM(\textbf{D})-WMV}} &
    \multicolumn{1}{c}{\scriptsize{SSIM(\textbf{D})-EMV}}
    \\
    \hline
    \scriptsize{DIP (MSE)}
    & \scriptsize{36.83} \tiny({3.07})
    & \textcolor{red}{\scriptsize{37.32}} \tiny({3.82})
    & \scriptsize{0.98} \tiny({0.02})
    & \textcolor{red}{\scriptsize{0.98}} \tiny({0.03})
    \\
    
    % \hline
    \scriptsize{DIP ($\ell_1$)}
    & \scriptsize{36.20} \tiny({2.81})
    & \textcolor{red}{\scriptsize{36.43}} \tiny({3.22})
    & \scriptsize{0.97} \tiny({0.02})
    & {\scriptsize{0.97}} \tiny({0.02})
    \\
    
    % \hline
    \scriptsize{DIP (Huber)}
    & \scriptsize{36.76} \tiny({2.96})
    & \textcolor{red}{\scriptsize{37.21}} \tiny({3.19})
    & \scriptsize{0.98} \tiny({0.02})
    & {\scriptsize{0.98}} \tiny({0.02})
    \\
    \hline
    \end{tabular}
    }
    \end{table}

% \subsubsection{{RAW images demosaicing and denoising}}\label{sec:raw_img_demosaicing}

% {RAW images demosaicing and denoising are two essential procedures for a modern digital camera to produce high-quality full-color images (~\cite{jdd_doubledip}). Given a noisy RAW image $\wh{\mb x}^{1ch} = \mb x^{1ch} + \mb n$, where $H$ and $W$ are the height and width of the image and $\mb n$ is the noise, the goal is to obtain a high-quality full-color image $\mb x^{3ch} \in \R^{H \times W \times 3}$ from it. To achieve this goal, it is obvious that we need to fill in the missing pixels (\emph{demosaicing}) and remove the noisy components $\mb n$ (\emph{denoising}). In this section, we formulate this problem as an image-inpainting problem as~\cite{jdd_doubledip} and adopt DIP to reconstruct the desired full-color image. We in additon plug our early-stopping method to DIP and explore the effectiveness of our method on this low-level vision task. We experiment with the Kodak dataset\footnote{\url{https://r0k.us/graphics/kodak/}} and prepare it following the pipeline in~\cite{jdd_doubledip}. For simplicity, however, we here only experiment with Poisson noise ($\lambda=25$; a detailed description of the noise intensity could be found in~\cite{jdd_doubledip}), which is a very common noise in low light conditions. We report the experimental results in \textcolor{red}{Fig.XXX}. It is evident that our method could effectively detect the around-peak points and produce reliable early stopping signals for DIP.}

\subsubsection{Ablation study}
\label{sec:ablation_ap}

\begin{figure}
    \centering
    % \vspace{-1em}
    \includegraphics[width=0.7\linewidth]{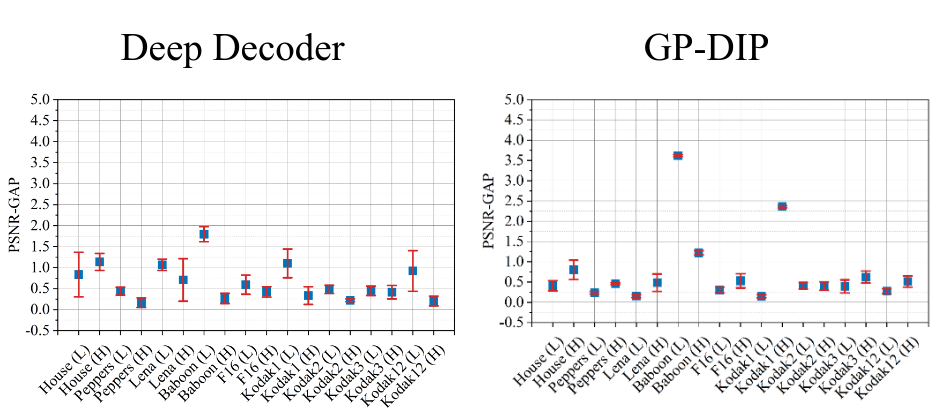}
    % % \vspace{-1em}
    \caption{Performance of ES-WMV on {deep decoder} and GP-DIP with smaller learning rates for Gaussian denoising in terms of PSNR gaps. L: low noise level; H: high noise level.}
    % % \vspace{-1em}
    \label{fig:dd_sgld_new}
\end{figure}

We also notice that smaller learning rates can smooth out the VAR curves and mitigate the multi-valley phenomenon in \cref{fig:dd_curve}. Therefore, we apply our ES-WMV to {deep decoder} and GP-DIP with smaller learning rates (both $0.001$), as shown in \cref{fig:dd_sgld_new}. Compared to the results of {deep decoder} and GP-DIP with the default learning rates in \cref{fig:dd_gp_tv}, most of the PSNR gaps decrease.

\subsection{Potential of ES-WMV for effective ES in zero-shot super-resolution with diffusion models}  \label{sec:diffusion} 
\begin{figure}[!htbp]
    \centering
    \includegraphics[width=\linewidth]{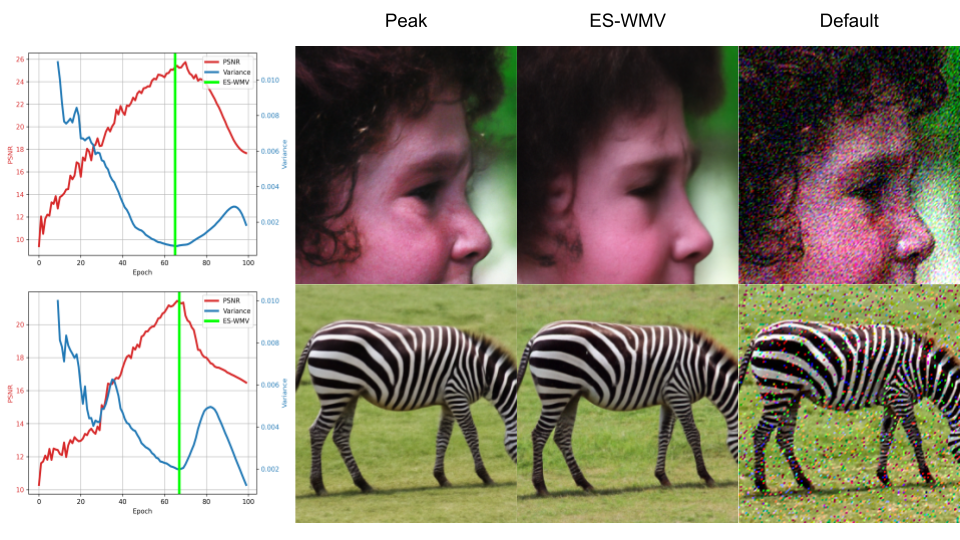}
    \caption{\textbf{Left}: PSNR/VAR curves of DDNM+ for $2 \times$ noisy image super-resolution (top: with Gaussian noise at $\sigma_y = 0.12$; bottom: with impulse noise). \textbf{Right}: visualization of the super-resolved results by DIP (PSNR peak), DIP with ES-WMV, and the default DDNM+. 
    The DDNM+ paper~\cite{wang_zero-shot_2022} fixes the iteration number as $1000$ and stops the reverse diffusion process at the very last iteration by default.}
    \label{fig:diffusion}
\end{figure}
Recently, zero-shot methods based on diffusion models have been proposed to solve linear image restoration tasks, e.g., DDNM+~\cite{wang_zero-shot_2022}\footnote{\url{https://github.com/wyhuai/DDNM/tree/main/hq_demo}}. However, these methods usually rely on pre-trained models from large external training datasets, while DIP does not need any training data or pre-trained models. In \cref{fig:diffusion}, we show that DDNM+ can also have overfitting issues similar to those in DIP methods, especially when the observation $\mb y$ is noisy but the noise type and/or level are not correctly specified to the diffusion models---very likely to happen in practice, as knowing the exact measurement noise type/level is often unrealistic. When DDNM+ is trained assuming no noise, i.e., $\sigma_y = 0.00$, but the downsampled image is contaminated by Gaussian noise at a level $\sigma_y = 0.12$, or by impluse noise, there is substantial overfitting to the noise, as is evident from both the PSNR plots (left of \cref{fig:diffusion}) and direct visualization of the super-resolved images (right of \cref{fig:diffusion}). Moreover, we observe that our ES-MWV method can also help DDNM+ detect near-peak performance! We stress that the experiment here is exploratory and preliminary and that tackling the overfitting issue in DDNM+ style methods for solving inverse problems is out of the scope of this paper. We leave a complete study for future work.

\subsection{{Limitations and analysis of failure cases}}
\label{sec:sup_failure}
\begin{figure}[!htbp]
    \centering
    % % \vspace{-4em}
    \includegraphics[width=0.7\linewidth]{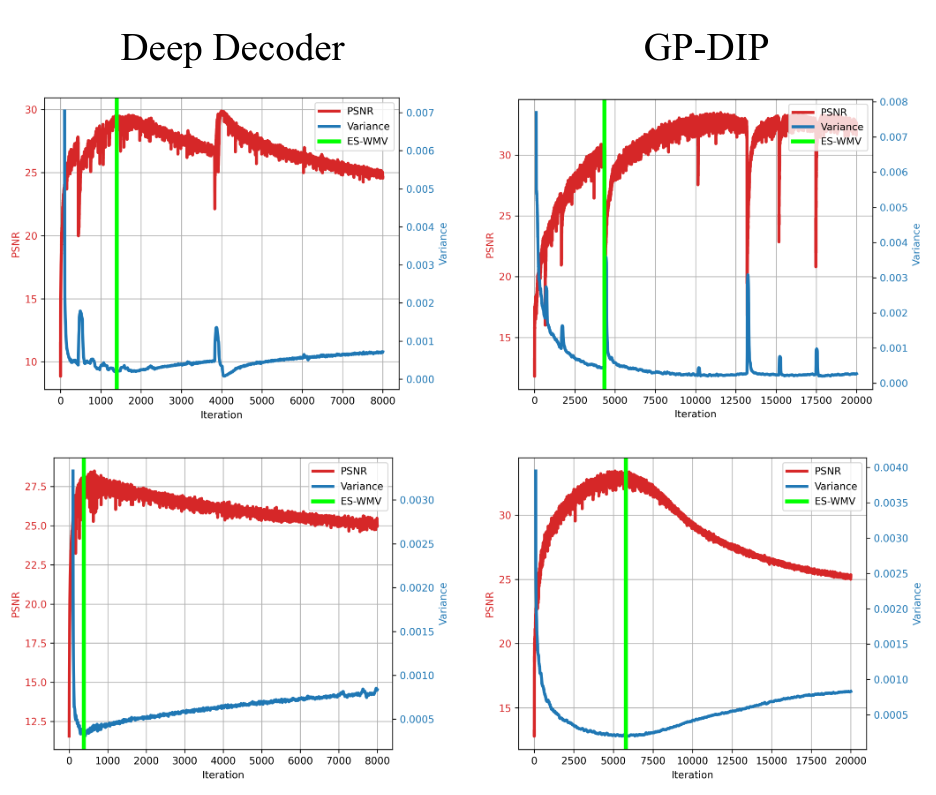}
    % % \vspace{-2em}
    \caption{Top left: DD with the default learning rate for ``Lena(L)''; top right: GP-DIP with the default learning rate for ``House(L)''; bottom left: DD with learning rate = $0.001$ for ``Lena(L)''; bottom right: GP-DIP with learning rate = $0.001$ for ``House(L)''.}
    % \vspace{-1em}
    \label{fig:dd_curve}
\end{figure}
\begin{figure}[!htbp]
    \centering
    \includegraphics[width=0.7\linewidth]{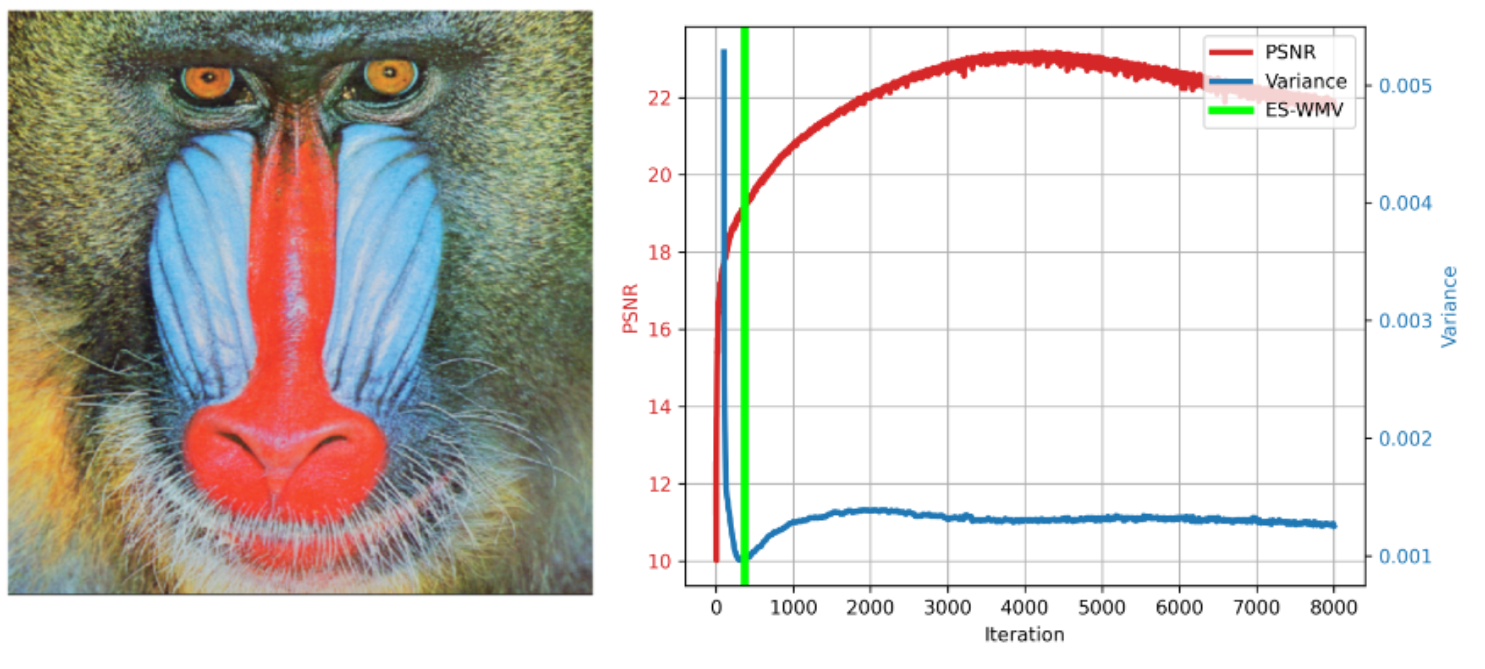}
    \caption{Our ES-WMV method on DIP for denoising ``Baboon" with low-level Gaussian noise. Left: clean ``Baboon"; right: the denoising process.}
    \label{fig:baboon}
\end{figure}

{For limitations, our theoretical justification is only partial, sharing the same difficulty of analyzing DNNs in general; our ES method struggles with images with substantial high-frequency components; our detection is sometimes off the peak in terms of iteration numbers when helping certain DIP variants, e.g. DIP-TV with low-level Gaussian noise (\cref{fig:helper_example}), but the detected PSNR gap is still small.}

Our ES method needs three things to succeed: (1) the U-shape of the VAR curve, (2) the VAR valley aligning with the PSNR peak, and (3) the successful numerical detection of the VAR valley. In this section, we discuss two major failure modes of our ES method: (I) the VAR valley aligns well with the PSNR peak, but the U-shape assumption is violated. A dominant pattern is the presence of multiple valleys, see, e.g., the top row of \cref{fig:dd_curve} that shows such examples with DIP variants, DD and GP-DIP (we do not observe the multi-valley phenomenon on DIP itself in \cref{fig:denoising_example}). Since our numerical valley detection method aims to locate the first major valley, it may not locate the deepest valley among the multiple valleys. Fortunately, for these cases, we observe that using smaller learning rates can help to smooth out their curves and mitigate the multi-valley phenomenon, leading to much smaller detection gaps (see the bottom row of \cref{fig:dd_curve});
(II) the VAR valley does not align well with the PSNR peak, which often happens on images with significant high-frequency components, e.g., \cref{fig:baboon}. We suspect that this is because the initial VAR decrease tends to correlate with the early learning of low-frequency components in DIP. When there are substantial high-frequency components in an image, the PSNR curve takes more time to pick up the high-frequency components, after the VAR curve already reaches the first major valley; hence, the misalignment between the VAR valley and the PSNR peak occurs. In addition, our ES-WMV can fail for images with substantial high-frequency components, e.g. \cref{fig:baboon}.